\documentclass[17pt]{article}
\usepackage{dsfont}
\usepackage{environ}
\usepackage[utf8x]{inputenc}

\usepackage{mdframed}
\usepackage{algorithm,algorithm,multicol}
\usepackage{pgfplots}

\definecolor{Gred}{RGB}{219, 50, 54}
\definecolor{Ggreen}{RGB}{60, 186, 84}
\definecolor{Gblue}{RGB}{72, 133, 237}
\definecolor{Gyellow}{RGB}{247, 178, 16}
\definecolor{ToCgreen}{RGB}{0, 128, 0}
\definecolor{myGold}{RGB}{231,141,20}
\definecolor{myBlue}{rgb}{0.19,0.41,.65}
\definecolor{myPurple}{RGB}{175,0,124}
\definecolor{niceRed}{RGB}{153,0,0}
\definecolor{niceRed}{RGB}{190,38,38}
\definecolor{blueGrotto}{HTML}{059DC0}
\definecolor{royalBlue}{HTML}{057DCD}
\definecolor{navyBlueP}{HTML}{0B579C}
\definecolor{limeGreen}{HTML}{81B622}
\definecolor{nicePink}{RGB}{247,83,148}

\usepackage{cmap}
\usepackage[T1]{fontenc}
\usepackage{bm}
\usepackage{amsmath}
\usepackage{amsfonts}
\usepackage{amssymb}
\usepackage{amsbsy}
\usepackage{amsthm}
\usepackage{bbm}

\usepackage{graphicx, ucs}

\usepackage{subcaption}
\usepackage{rotating}
\usepackage{float}
\usepackage{tikz}

\usepackage{algorithm, caption}
\usepackage[noend]{algpseudocode}
\usepackage{listings}

\usepackage{enumitem}

\usepackage[pagebackref]{hyperref}
\hypersetup{
	colorlinks = true,
	urlcolor = {myPurple},
	linkcolor = {royalBlue},
	citecolor = {nicePink}
}
\usepackage[nameinlink]{cleveref}

\usepackage{multirow}
\usepackage{array}

\usepackage{chngcntr}

\counterwithin*{equation}{section}
\usepackage{chngcntr}
\usepackage{soul}
\usepackage{nicefrac}

\usepackage{fancyhdr}
\fancyheadoffset{0pt}
\def\compactify{\itemsep=0pt \topsep=0pt \partopsep=0pt \parsep=0pt}
\let\latexusecounter=\usecounter

\definecolor{myC}{rgb}{0, 255, 255}
\definecolor{myY}{rgb}{204, 204, 0}
\definecolor{myM}{rgb}{255, 0, 255}
\definecolor{secinhead}{RGB}{249,196,95}
\definecolor{lgray}{gray}{0.8}

\newtheorem{theorem}{Theorem} 
\newtheorem*{theorem*}{Theorem} 
\newtheorem*{proposition*}{Proposition} 
\newtheorem{lemma}{Lemma}
\newtheorem{claim}{Claim}
\newtheorem{proposition}{Proposition}
\newtheorem{fact}{Fact}

\newtheorem{openquestion}{Open Question}

\newtheorem{definition}{Definition}
\newtheorem{remark}{Remark}

\newtheorem{observation}{Observation}
\newtheorem{example}{Example}

\renewcommand{\Pr}{\mathop{\bf Pr\/}}
\newcommand{\E}{\mathop{\bf E\/}}

\newcommand{\er}{\mathrm{er}}

\newcommand{\sgn}{\textnormal{sgn}}

\newcommand{\reals}{\mathbb R}

\newcommand{\nats}{\mathbb N}

\newcommand{\eps}{\epsilon}

\newcommand{\calA}{\mathcal{A}}

\newcommand{\calF}{\mathcal{F}}
\newcommand{\calG}{\mathcal{G}}
\newcommand{\calH}{\mathcal{H}}

\newcommand{\calM}{\mathcal{M}}

\newcommand{\calT}{\mathcal{T}}
\newcommand{\calU}{\mathcal{U}}

\newcommand{\calW}{\mathcal{W}}
\newcommand{\calX}{\mathcal{X}}
\newcommand{\calY}{\mathcal{Y}}

\usepackage{wrapfig}

\def\l{\ell}
\def\<{\langle}
\def\>{\rangle}

\DeclareMathOperator*{\argmax}{argmax}

\makeatletter
\newcommand{\printfnsymbol}[1]{%
  \textsuperscript{\@fnsymbol{#1}}%
}
\makeatother

\def\wt{\widetilde}
\def\wh{\widehat}

\def\vec{\bm}

\makeatletter
\renewenvironment{abstract}{%
	\if@twocolumn
	\section*{\abstractname}%
	\else %% <- here I've removed \small
	\begin{center}%
		{\bfseries \large\abstractname\vspace{\z@}}%  %% <- here I've added \Large
	\end{center}%
	\quotation
	\fi}
{\if@twocolumn\else\endquotation\fi}
\makeatother

\begin{document}
	
	\title{Multiclass Learnability Beyond the PAC Framework:\\
 Universal Rates and Partial Concept Classes}
	\author{
	    \textbf{Alkis Kalavasis\thanks{Equal contribution.}} \\
		 NTUA \\
		\texttt{kalavasisalkis@mail.ntua.gr}
        \and
		\textbf{Grigoris Velegkas\printfnsymbol{1}} \\
		 Yale University\\
		\texttt{grigoris.velegkas@yale.edu}
        \and
		\textbf{Amin Karbasi} \\
		 Yale University\\
		\texttt{amin.karbasi@yale.edu}
	}
	\maketitle
	\thispagestyle{empty}

	\begin{abstract}
	\small
In this paper we study the problem of multiclass classification with a bounded number of different labels $k$, in the realizable setting. We extend the traditional PAC model to \textbf{a)} distribution-dependent learning rates, and \textbf{b)} learning rates under data-dependent assumptions. 
    First, we consider the \emph{universal learning} setting (Bousquet, Hanneke, Moran, van Handel and Yehudayoff, STOC '21), 
    for which we provide a complete characterization of the achievable learning rates
    that holds for every fixed distribution.
    In particular, we show the following trichotomy: for any concept class, the optimal learning rate is either exponential, linear or arbitrarily slow. Additionally, we provide complexity measures of the underlying hypothesis class that characterize when these rates occur. 
    Second, we consider the problem of multiclass classification with \emph{structured} data (such as data lying on a low dimensional manifold or satisfying margin conditions),
       a setting which is captured by \emph{partial concept classes} (Alon, Hanneke, Holzman and Moran, FOCS '21). Partial
       concepts are functions that can be \emph{undefined} in certain parts
       of the input space. We extend the traditional PAC learnability of total concept classes to partial concept classes in the multiclass setting and investigate differences between partial and total concepts.
       
	\end{abstract}

	\newpage
	
\section{Introduction}\label{section:introduction}
%Multiple real-life machine learning problems, such as image recognition, require classification into many classes. 
Classifying data into multiple different classes is a fundamental problem in machine learning
    that has many real-life applications, 
    %the most prominent of which is, arguably, 
    such as image recognition, web advertisement and text categorization.
Due to its importance, multiclass classification has been an attractive field of research 
both for theorists \cite{natarajan1988two, characterization, rubinstein2009shifting, natarajan1989learning,multiclass-erm,daniely2014optimal,understanding-ml} and for practitioners \cite{shalev2004learning,collins2004parameter,aly}. Essentially, it boils down to 
learning a classifier $h$ from a domain $\calX$ to a label
space $\calY$, where $|\calY| \geq 2,$ and the error is measured by the probability that
$h(x)$ is incorrect.
In this work, we focus on the setting where the number of labels
is finite and we identify $\calY$ with $ [k] := \{0,1,...,k\}$ for some
 constant $k \in \nats$.
 
\paragraph{Multiclass PAC Learning} 
The PAC model~\cite{valiant1984theory} constitutes the gold-standard learning framework. 
%\grigorisnote{Roughly speaking, in this framework there is a \emph{hypothesis class} $\calH \subseteq \{0,1\}^{\calY}$ that is given to the learner and some unknown distribution $\calD$ over $\calX \times \calY$. The learner is presented with $n$ labeled points $S = \{(x_1,y_1),\ldots,(x_n,y_n)\}$ that are drawn i.i.d. from $\calD$, which we call the \emph{training set}. The learner needs to produce a classification rule $\wh{h}_n$ whose expected error with respect to $\calD$ is small. In the \emph{realizable} setting, there is some $h \in \calH$ that is (almost) always correct.}
A seminal result in learning theory \cite{vapnik2015uniform,blumer} characterizes PAC learnability of
binary classes ($k=1$) through the Vapnik-Chervonenkis (VC) dimension and provides a 
clear algorithmic landscape with the \emph{empirical risk minimization} (ERM) principle yielding (almost) optimal statistical learning algorithms.
The picture established for the binary setting extends to the case of multiple labels when the number of classes
$|\calY|$ is bounded. The works of \cite{natarajan1988two,natarajan1989learning} and \cite{characterization} 
identified natural extensions of the VC dimension, such as the \emph{Natarajan dimension}, whose finiteness characterizes multiclass learnability in this setting. 
%These extensions (called distinguishers) contain a rich collection of combinatorial measures such as the Natarajan and Graph dimensions and the Pseudo-dimension and are all equivalent (up to $\log(k)$ factors) when the number of labels is bounded.  
%the Natarajan dimension and the Graph dimension.
%and capture PAC learnability in this setting. 
Moreover,
the ERM principle still holds and achieves the desired learning rate by essentially reducing learnability to optimization.
%and the ERM sample complexity \cite{daniely2014optimal} is, essentially, characterized by the Graph dimension (even without bounded $| \calY |$\footnote{maybe erase the 2nd part of the sentence since in our setting the two dims are equivalent.}) \cite{multiclass-erm}. 
The fundamental result of the multiclass PAC learning (in the realizable setting for the expected prediction error)
\footnote{We mention that \Cref{definition:multiclass pac learning} essentially works in the probabilistic prediction model of \cite{haussler1994predicting}. In this model, one does not require the hypotheses used for prediction to be represented in any specified way since the prediction for the test point $x$ is performed on the fly.
For instance, a valid classification rule
in this model is to memorize all the training points and
run an algorithm for every unlabeled instance $x.$
In Valiant's PAC model, 
we are interested in approximating a target hypothesis $h \in \calH$ given an $h$-labeled sample and we must output a representation of a hypothesis in some hypothesis class $\calF$ \cite{haussler1994predicting}. Also, \Cref{definition:multiclass pac learning} deals with the expected error.} 
can be summarized in the following elegant equation for any $n \in \nats$,
%\footnote{To make the exposition easier to follow we consider the definition of PAC that is concerned with the expected value of the error. Similar results hold for the $(\varepsilon,\delta)$-definition.}
 which we explain right-after:
\begin{equation}\label{definition:multiclass pac learning}
    \inf_{\wh{h}_n} \sup_{P \in \mathrm{RE}(\calH)} \E[\mathrm{er}(\wh{h}_n)]
    = \min \left( \wt{\Theta}_{k}\left(\frac{\mathrm{Ndim}(\calH)}{n} \right), 1 \right ) 
    %=
    %\min \left( \wt{\Theta}_{k}\left(\frac{\mathrm{Gdim}(\calH)}{n} \right), 1 \right)
    %,\forall n \in \nats
    \,,
\end{equation}
where $\mathrm{Ndim}(\calH)$ stands for 
the Natarajan dimension of $\calH$ and $\wt{\Theta}_{k}$ subsumes dependencies on $k$ (the lower bound of order $\mathrm{Ndim}(\calH)/n$ is from \cite{rubinstein2009shifting} and
the upper bound $\mathrm{Ndim}(\calH) \log (k)/n$ comes from an application of the one-inclusion 
hypergraph algorithm \cite{moran-dinur}). In words, assume that $n$ 
is the number of training samples and let $\calH \subseteq [k]^\calX$ be a set of multiclass classifiers mapping the elements of the domain $\calX$ to $[k]$ that the learner has access to. The learner observes $n$ labeled examples $(x,y) \in \calX \times [k]$ generated i.i.d. from some unknown distribution $P$ with the constraint that $P$ is \emph{realizable} with respect to $\calH$, i.e., there is some $h \in \calH$ that has (almost surely) zero classification error. The learner then
outputs a guess hypothesis $\wh{h}_n : \calX \to [k]$ (potentially, without an explicit description). The fundamental theorem of PAC learning, as shown above, controls the expected error $\E[ \mathrm{er}(\wh{h}_n) ]$, where $\mathrm{er}(\wh{h}_n) := \Pr_{(x,y) \sim P}[\wh{h}_n(x) \neq y]$, in a \emph{minimax} sense, i.e., it controls the performance of the best algorithm $\wh{h}_n$ (inf) against the worst-case realizable distribution $P \in \mathrm{RE}(\calH)$ (sup) and states that 
the following dichotomy occurs: if the Natarajan  %(or equivalently the Graph)
dimension $\mathrm{Ndim}(\calH)$ %(respectively $\mathrm{Gdim}(\calH)$)
is finite, the error rate decreases as, roughly, $1/n$, so $\calH$ is PAC learnable at a linear rate; otherwise, the class $\calH$ is not PAC learnable.
Additionally, this theory provides a clean algorithmic landscape: the ERM principle, which means outputting some classifier in $\calH$ that best fits the training set, (roughly) achieves the rates of \Cref{definition:multiclass pac learning}.

\paragraph{Towards Novel Learning Theories} While the PAC model provides a solid and attractive theoretical framework, 
%one may argue that 
it fails to (fundamentally) capture the real-world behavior of various applied ML problems. In this work, we focus on the following two points of criticism for the standard PAC model.  
The first natural point concerns the supremum over all realizable distributions in \Cref{definition:multiclass pac learning}. 
\begin{observation}
\label{obs1}
The PAC model is distribution-independent and captures the worst-case learning rate.
Is it possible to design a learning theory that provides distribution-dependent learning rates?
\end{observation}

Another critique is that one cannot express natural data-dependent assumptions through the PAC framework. For instance, high-dimensional data may lie in a low-dimensional manifold. To be more specific, consider the task of classifying images of vehicles. The representation of such images corresponds to a low-dimensional subset of the space of all possible images, most of which do not correspond to vehicles. 
A prominent way to capture such assumptions is via \emph{partial concepts} \cite{alon2022theory}: these are functions which can be \emph{undefined} on a subset of $\calX$, a departure from the traditional model.  
\begin{observation}
\label{obs2}
The PAC model only considers total concept classes, i.e., $\calH \subseteq [k]^\calX,$
%where $h \in \calH$ is a total function from $\calX$ to $[k]$ 
which cannot express data-dependent constraints. Is it possible to design a learning theory for partial concepts $h : \calX \to \{0,1,...,k,\star\}$, where $h(x) = \star$ means that $h$ is undefined at $x$? 
\end{observation}
The aim of this paper is to develop (i) a distribution-dependent learning theory for multiclass classification and (ii) a learning theory for partial multiclass concept classes in the distribution-independent setting\footnote{While we believe that one could design a unified learning theory addressing the two questions at once, we prefer to provide two separate theories, since they are both interesting in their own.}. We comment that we focus on the realizable setting that already poses important challenges and requires novel ideas and we believe that our results can be extended to the agnostic case, which is left for future work. We remark that such theories for binary classification were recently developed~\cite{universal,alon2022theory}. However, in various practical applications, such explanations may not suffice,
%they do not suffice to cover practical applications,
since it is rarely the case that there are only two classes. As it is already evident from the PAC setting, moving from binary classification to multiclass classification is not trivial~\cite{multiclass-erm}.
We now discuss \Cref{obs1} and \ref{obs2}; we underline that our goal is not to replace, but to build upon and complement, the traditional PAC model, which constitutes the bedrock of learning theory.

\paragraph{Distribution-Dependent Learning Rates}
In many modern machine learning applications the generalization error $\E[\mathrm{er}(\wh{h}_n)]$ drops \emph{exponentially fast}
as a function of the sample size $n$~\cite{cohn1990can,cohn1992tight,schuurmans,viering2021shape}. However, the
traditional PAC learning theory predicts merely $\wt{O}(d/n)$ rates in the realizable setting,
where $d$ is the complexity measure of the underlying concept
class that the algorithm is trying to learn. 
A possible explanation for this discrepancy between the theoretical guarantees and the empirical performance of the learning 
algorithms is the worst-case nature of the PAC guarantees. Notice that in \Cref{definition:multiclass pac learning}, for any fixed learning algorithm, one considers its performance against the worst distribution for it. In particular, this means that as the sample size $n$ increases and new classifiers $\wh{h}_n$ are produced, the distribution that is used as a benchmark can differ. However, in many practical applications, one considers some \emph{fixed} distribution $P$ and measures the performance of the classifier
as $n \rightarrow \infty$ without changing $P$. Hence, there is an important need to study mathematical models that capture this behavior of learning algorithms and not just the minimax one. One such approach
that was recently proposed by~\cite{universal} is to study 
\emph{universal learning rates}, which means that 
the learning rates guarantees hold for $\emph{every}$ fixed (realizable) distribution $P$, but there is not a \emph{uniform} bound over all of the distributions. To be more precise,
$\calH$ is learnable at rate $R$ (where $\lim_{n \to \infty} R(n) = 0) $ in the universal setting if
\begin{equation}
\label{definition:universal setting guarantee}
    \exists \wh{h}_n: \forall P \in \mathrm{RE}(\calH), ~\exists C = C(P), c = c(P) > 0 \text{ so that } \E[\er(\wh{h}_n)] \leq C \cdot R(c \cdot n), \forall n \in \nats.
\end{equation}
Note that the above equation is the same as in the PAC model with the exception of a change between the existential quantifiers: in PAC, the focus is on the case where $\exists C,c > 0 : \forall P \in \mathrm{RE}(\calH)$ the guarantee holds (which expresses \emph{uniformity}), while in the universal setting distribution-dependent constants are allowed.
This subtle change in the definition can make the error-rate
landscape vastly different. As an example, consider the case where
for $\wh{h}_n$, we have that $\er(\wh{h}_n) \leq C(P)e^{-c(P)n}$, for
every distribution $P$, where $C(P),c(P)$ are some
distribution-dependent constants. When we take the pointwise
supremum over all of these infinitely many distributions, it can be the
case that
the resulting function drops as $C'/n$, where $C'$ is a distribution-independent constant \cite{universal}.

\paragraph{Partial Concept Classes} The motivation behind \Cref{obs2} is that, in various practical learning tasks, the data satisfy some special properties that make the learning process simpler. For instance, it is a common principle to use classification with margin where the points in the dataset have a safe gap from the decision boundary. Such properties induce \emph{data-dependent} assumptions that the traditional PAC learning theory framework provably fails to express. 
%For instance, for the task of linear classification with margin, the Perceptron algorithm uses the class of all linear classifiers in $\reals^d$ which is not PAC learnable when the ambient dimension $d$ is unbounded. 
In fact, existing data-dependent analyses diverge from the standard PAC model \cite{shawe1998structural, herbrich2001algorithmic} and provide problem-specific approaches. Thus, there is a need for a formal framework that allows us to express such data-dependent restrictions and study these problems in a \emph{unified and principled way}. Recently, \cite{alon2022theory} proposed an elegant extension of the binary PAC model to handle such tasks via the framework of \emph{partial concept classes}. As an intuitive example, a halfspace with margin is a partial function that is undefined inside the forbidden margin and is a well-defined halfspace outside the margin boundaries.%\grigorisnote{would it be useful to describe the multiclass variant of large-margin halfspaces?}
%We introduce the framework formally and extend their setting to the case of multiple labels in \Cref{section:partial}. We manage to characterize multiclass learnability in the setting of partial concept classes and delve into the algorithmic landscape of these tasks, which is significantly different from the standard ERM principle.

\subsection{The Traditional Multiclass Learning Problem}\label{section:universal rates formal}
Let $[k] = \{0,1,...,k\}$ for some fixed positive integer $k \in \nats$. We consider a domain $\calX$ and a concept class $\calH \subseteq [k]^\calX$. A classifier is a universally measurable\footnote{We discuss measurability formally in \Cref{appendix:measurability}} function $h : \calX \to [k]$. The error rate of a classifier $h$ with respect to a probability distribution $P$ on $\calX \times [k]$ is equal to $\mathrm{er}(h) = \mathrm{er}_P(h) = \Pr_{(x,y) \sim P}[h(x) \neq y]$. We focus on the setting where $P$ is realizable, i.e., $\inf_{h \in \calH} \mathrm{er}_P(h) = 0$. Formally, a (deterministic\footnote{We focus for simplicity on deterministic learners. Our results extend to randomized algorithms.}) learning algorithm is a sequence of universally measurable functions $H_n : (\calX \times [k])^n \times \calX \to [k]$, which take as input a sequence of $n$ independent pairs $(x_i,y_i) \sim P$ (training set) and output data-dependent classifiers $\wh{h}_n: \calX \rightarrow [k]$, where $\wh{h}_n(x) = H_n((x_1,y_1),\ldots,(x_n, y_n),x)$. The goal is to come up with algorithms whose $\E[\mathrm{er}(\wh{h}_n)]$ admits a fast decay as a function of $n$, where the expectation is over the training set.

\subsection{Universal Multiclass Learning: Our Results}
The aim of our first theory is to fully characterize the admissible universal rates of learning, i.e., $\E[\mathrm{er}(\wh{h}_n)],$ in the multiclass classification setting with a bounded number of labels. The following definition formalizes this notion of achievable rate in the (realizable) universal learning model~\cite{universal}.
\begin{definition}[\cite{universal}, Definition 1.4]
Let $\calH \subseteq [k]^\calX$ and let $R : \nats \to [0,1],$ with $R(n) \to 0$, be a rate function.
We say that
 $\calH$ is \textbf{learnable at rate $R$} if there exists a learning algorithm $\wh{h}_n$ such that for every realizable distribution $P$ on $\calX \times [k]$ with respect to $\calH$, there exist distribution-dependent $C,c > 0$ for which $\E[\mathrm{er}(\wh{h}_n)] \leq C R(cn),$ for all $n \in \nats$.
    Also, $\calH$ is \textbf{not learnable at rate faster than $R$} if for any learning algorithm $\wh{h}_n$, there exists a realizable distribution $P$ on $\calX \times [k]$ with respect to $\calH$ and distribution-dependent $C,c > 0$ for which $\E[\mathrm{er}(\wh{h}_n)] \geq C R(cn)$ for infinitely many $n \in \nats$.
    $\calH$ is \textbf{learnable with optimal rate $R$} if it is learnable at rate $R$ and is not learnable faster than $R$.
    Finally, $\calH$ requires \textbf{arbitrarily slow rates} if, for every $R(n) \to 0$, $\calH$ is not learnable at rate faster than $R$.
\end{definition}

In the universal multiclass setting, we show that the following fundamental trichotomy occurs (in comparison with the dichotomy witnessed in the uniform PAC model). 
This result is a theoretical justification of the exponential error rates observed in practice.
\begin{theorem}
\label{theorem:tri-main}
Fix a constant $k \in \nats$.
Consider a hypothesis class $\calH \subseteq [k]^\calX$ with $|\calH| > k+2$. Then, exactly one of the following holds for the learning rate of $\calH$ in the realizable case:
\begin{itemize}
    \item $\calH$ is learnable at an optimal rate $e^{-n}$.
    \item $\calH$ is learnable at an optimal rate $1/n$.
    \item $\calH$ requires arbitrarily slow rates.
\end{itemize}
\end{theorem}
We mention that $|\calH| > k+2$ comes without loss of generality.\footnote{The constraint $|\calH| > k+2$ rules out some degenerate scenarios, we kindly refer to~\Cref{appendix:cardinality of H}.}
In contrast to the standard PAC model, any concept class is learnable in the universal rates setting \cite{hanneke-kontorovich}. Intuitively, the analogue of non-learnability in the uniform setting is the case of arbitrarily slow rates.
Our second result is the specification of some combinatorial complexity measures of $\calH$ that characterize the optimal learning rate of this class. Let us first provide some informal definitions of these measures. We begin with the notion of \emph{multiclass Littlestone trees}, which extends the binary Littlestone trees from~\cite{universal}.

\begin{definition}
[Informal (see \Cref{definition:multi-little-tree})]
\label{definition:littletree-informal}
A \textbf{multiclass Littlestone tree} 
%\grigorisnote{should we use k-L tree? }  
for $\calH \subseteq [k]^{\calX}$ is a complete binary tree of depth $d \leq \infty$ whose internal
nodes are labeled by $\calX$, and whose two edges connecting a node to its children are labeled by two different elements in $[k]$,
such that every path of length at most $d$ emanating from the root is consistent with a concept $h \in \calH$.
We say that $\calH$ has an \textbf{infinite multiclass Littlestone tree} if there is a multiclass Littlestone tree for $\calH$ of depth $d = \infty$.
\end{definition}

For some intuition we refer the reader to \Cref{fig:lit}. The above complexity measure appears in the definition of the multiclass Littlestone dimension~\cite{multiclass-erm}. In fact, a class $\calH \subseteq [k]^\calX$ 
has multiclass Littlestone dimension $d$ if it has a multiclass Littlestone tree of depth $d$ but not of depth $d+1$. We underline that having an infinite multiclass Littlestone tree is \emph{not} the same as having an unbounded multiclass Littlestone dimension. A class 
$\calH$ has unbounded Littlestone dimension if for
every $d \in \nats$ there is \emph{some} tree of depth $d$.
However, this does not mean that there is a \emph{single}
infinite tree. This is a fundamental conceptual gap between the uniform and the universal settings.

% It is possible for a class to have finite multiclass Littlestone trees of arbitrarily large depth. This property implies that the multiclass Littlestone dimension is infinite (there is no \emph{uniform} bound) but it does not imply that there exists any \emph{single} tree of \emph{infinite} depth. This is a fundamental conceptual gap between the uniform and the universal settings. 

%\alkissnote{We should add a picture and a table that gives 3 binary(?) infinite classes: threshold in R (inf tree and dim) , threshold in N (inf dim but not tree), and $1\{ x = n\}$ for $n \in N$ (fin dim). }
%\grigorisnote{What if we say that $x = (x_1,\ldots, x_{\log k})$ and do thresholds etc. for every coordinate? Will we get the same table as in the binary setting?}

\begin{wrapfigure}{r}{0.55\textwidth}
%\begin{figure}

\centering
\begin{tikzpicture}
\draw (0,0) to (1.5, 1);
\draw[line width=0.5mm,color=red] (0,0) to (1.5,-1);
\draw (1.5,1) to (3, 1.5);
\draw (1.5,1) to (3, 0.5);
\draw[line width=0.5mm,color=red] (1.5,-1) to (3,-0.5);
\draw (1.5,-1) to (3,-1.5);
\draw[densely dashed] (3,1.5) to (4.2,1.75);
\draw[densely dashed] (3,1.5) to (4.2,1.25);
\draw[densely dashed] (3,0.5) to (4.2,.75);
\draw[densely dashed] (3,0.5) to (4.2,.25);
\draw[densely dashed] (3,-1.5) to (4.2,-1.75);
\draw[densely dashed] (3,-1.5) to (4.2,-1.25);
\draw[line width=0.5mm,color=red,densely dashed] (3,-0.5) to (4.2,-.75);
\draw[densely dashed] (3,-0.5) to (4.2,-.25);

\draw (0.75,0.7) node {$\scriptstyle \l^{(0)}_\emptyset$};
\draw (0.75,-0.815) node[color=red] {$\scriptstyle \l^{(1)}_\emptyset$};

\draw (2.25,1.45) node {$\scriptstyle \ell_0^{(0)}$};
\draw (2.25,0.5) node {$\scriptstyle \ell_0^{(1)}$};
\draw (2.25,-0.555) node[color=red] {$\scriptstyle \ell_1^{(0)}$};
\draw (2.25,-1.5) node {$\scriptstyle \ell_1^{(1)}$};

\draw (3.8,1.8) node {$\scriptstyle$};
\draw (3.8,1.2) node {$\scriptstyle$};
\draw (3.8,.8) node {$\scriptstyle$};
\draw (3.8,.2) node {$\scriptstyle$};

\draw (3.8,-1.8) node {$\scriptstyle $};
\draw (3.8,-1.2) node {$\scriptstyle$};
\draw (3.8,-.885) node[color=red] {$\scriptstyle \ell_{10}^{(1)}$};
\draw (3.8,-.1) node {$\scriptstyle \ell_{10}^{(0)}$};

\draw (0,0) node[color=red,circle,radius=.15,fill=white] {$\bm{x_\varnothing}$};
\draw (1.5,1) node[circle,radius=.15,fill=white] {$x_0$};
\draw (1.5,-1) node[color=red,circle,radius=.15,fill=white] {$\bm{x_1}$};
\draw (3,1.5) node[circle,radius=.15,fill=white] {$x_{00}$};
\draw (3,0.5) node[circle,radius=.15,fill=white] {$x_{01}$};
\draw (3,-0.5) node[color=red,circle,radius=.15,fill=white] {$\bm{x_{10}}$};
\draw (3,-1.5) node[circle,radius=.15,fill=white] {$x_{11}$};

\draw[color=red,<-] (4.3,-.75) to[out=0,in=180] (5.3,0);
\draw[color=red] (5.3,.66) node[right] {$\exists\,h\in\mathcal{H} : $};
\draw[color=red] (5.3,.02) node[right] {$h(x_\varnothing)=\ell_\emptyset^{(1)}$};
\draw[color=red] (5.3,-.62) node[right] {$h(x_1)=\ell_1^{(0)}$};
\draw[color=red] (5.3,-1.26) node[right] {$h(x_{10})=\ell_{10}^{(1)}$};

\end{tikzpicture}
\caption{A multiclass Littlestone tree of depth $3$. Each node $x_u$ has two children $\l_u^{(0)} \neq \l_u^{(1)}$ where $\l_u^{(i)} \in [k]$ for any $i \in \{0,1\}$ and $u \in \{0,1\}^*,$ i.e., the set of binary strings of arbitrary length. Every branch is 
consistent with some concept $h \in \mathcal{H}$. The figure is adapted from~\cite{universal}.
\label{fig:lit}
}
%\end{figure}

\end{wrapfigure}
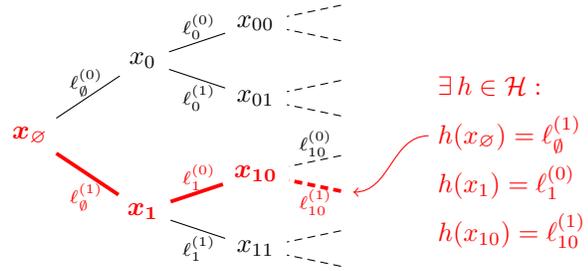

The next definition is novel and is motivated by the fundamental notion of the Natarajan dimension from the multiclass PAC setting (see \Cref{definition:natarajan dimension}). We first need some terminology: a tuple $(x_1,...,x_t,s^{(0)}_1,...,s^{(0)}_t,s^{(1)}_1,...,s^{(1)}_t) \in \calX^t \times [k]^t \times [k]^t$ with $s_i^{(0)} \neq s_i^{(1)}$, for any $i \in [t],$ is \textbf{$N$-consistent} with the edge $(y_1,...,y_t) \in \{0,1\}^t$ and the concept $h \in \calH$ if $h(x_i) = s^{(y_i)}_i$ for any $i \in [t]$. Recall that if the tuple is $N$-consistent with any binary pattern $y \in \{0,1\}^t$, we say that $(x_1,...,x_t)$ is \textbf{$N$-shattered}. More generally, a path is $N$-consistent with a concept $h \in \calH$ if each node of the path is $N$-consistent with the edge connecting the node with its child across the path and $h$. Since the next definition might be hard to parse, we refer the reader to \Cref{fig:nl-tree} for some intuition. 

\begin{definition}
[Informal (see \Cref{definition:NL tree})]
A \textbf{Natarajan-Littlestone (NL) tree} for $\calH \subseteq [k]^\calX$ is a complete tree of depth $d \leq \infty$ so that every level $1 \leq t \leq d$ has branching factor $2^t$ and nodes that are labeled by $\calX^t \times [k]^t \times [k]^t$ (so that for all $i \in [t]$ the two labels in $[k] \times [k]$ are different) and whose $2^t$ edges connecting a node to its children are labeled by the elements of $\{0,1\}^t$. It must hold that every path of length at most $d$ emanating from the root is $N$-consistent with a concept $h \in \calH$. We say that $\calH$ has an \textbf{infinite Natarajan-Littlestone tree} if there is an NL tree for $\calH$ of depth $d = \infty$.
\end{definition}

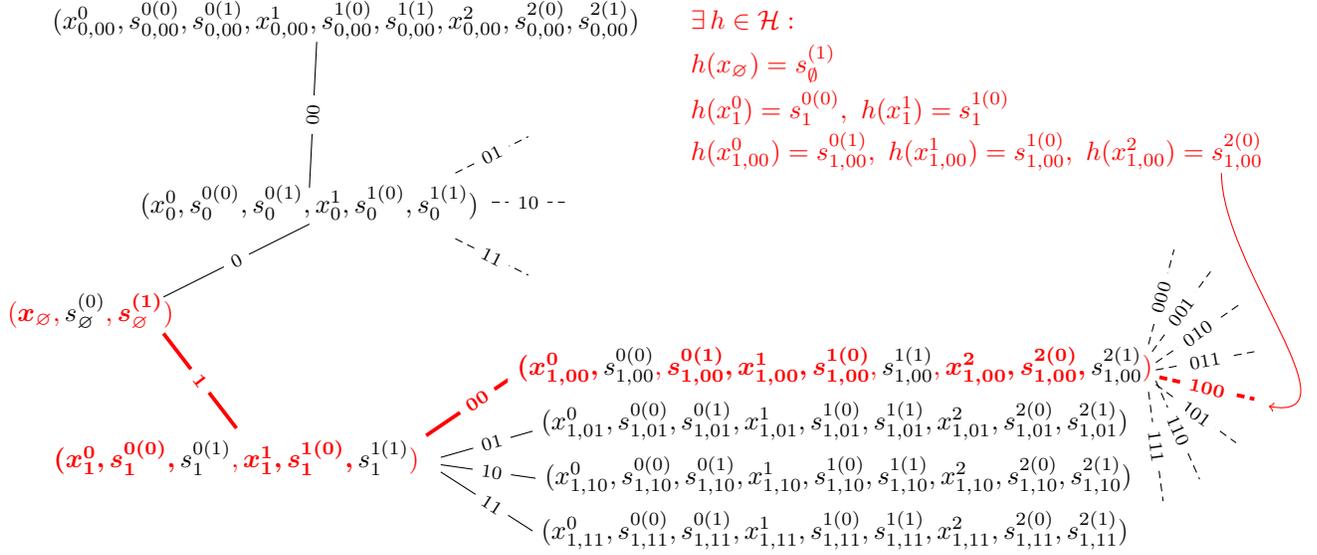
\begin{figure}[ht!]
\centering

\begin{tikzpicture}[scale=0.97]

\draw (-6,0) node[color=red,circle,radius=.15,fill=white] {$(\bm{x_\varnothing}, \color{black}s_{\varnothing}^{(0)}\color{red}, \bm{s_{\varnothing}^{(1)}})$};

\draw (-3,1.5) node[fill=white] {$(x_0^0, s_0^{0 (0)}, s_0^{0(1)},x_0^1,s_0^{1(0)}, s_0^{1(1)})$};
\draw (-4,-2) node[color=red,fill=white] {$\bm{(x_1^0, s_1^{0 (0)},} \color{black}s_1^{0(1)}\color{red},\bm{x_1^1,s_1^{1(0)},} \color{black}s_1^{1(1)}\color{red})$};

\draw (-5,.2) to node[sloped,anchor=center,fill=white] {$\scriptstyle 0$} (-3, 1.2) ;
\draw[line width=0.5mm,color=red] (-5,-.3) to node[sloped,anchor=center,fill=white] {$\scriptstyle\bm{1}$} (-4,-1.6);

\draw (-1.2,-2.2) to node[sloped,anchor=center,fill=white] {$\scriptstyle 11$} (0.2,-3.1);
\draw (-1.2,-2.1) to node[sloped,anchor=center,fill=white] {$\scriptstyle 10$} (0.2,-2.3);
\draw (-1.2,-2) to node[sloped,anchor=center,fill=white] {$\scriptstyle 01$} (0.2,-1.6);
\draw[line width=0.5mm,color=red] (-1.4,-1.7) to node[sloped,anchor=center,fill=white] {$\scriptstyle\bm{00}$} (0,-.75);

\draw (4.2,-1.5) node[fill=white] { $(x_{1,01}^0,s_{1,01}^{0(0)},s_{1,01}^{0(1)},x_{1,01}^1,s_{1,01}^{1(0)},s_{1,01}^{1(1)},x_{1,01}^2,s_{1,01}^{2(0)},s_{1,01}^{2(1)})$};
\draw (4.25,-2.25) node[fill=white] {$(x_{1,10}^0,s_{1,10}^{0(0)},s_{1,10}^{0(1)},x_{1,10}^1,s_{1,10}^{1(0)},s_{1,10}^{1(1)},x_{1,10}^2,s_{1,10}^{2(0)},s_{1,10}^{2(1)})$};
\draw (4.2,-3) node[fill=white] {$(x_{1,11}^0,s_{1,11}^{0(0)},s_{1,11}^{0(1)},x_{1,11}^1,s_{1,11}^{1(0)},s_{1,11}^{1(1)},x_{1,11}^2,s_{1,11}^{2(0)},s_{1,11}^{2(1)})$};
\draw (4.2,-0.75) node[color=red,fill=white] {$\bm{(x_{1,00}^0,}\color{black}s_{1,00}^{0(0)}\color{red},\bm{s_{1,00}^{0(1)},x_{1,00}^1,s_{1,00}^{1(0)}},\color{black}s_{1,00}^{1(1)}\color{red},\bm{x_{1,00}^2,s_{1,00}^{2(0)},}\color{black}s_{1,00}^{2(1)}\color{red})$};

\draw (-2.5,4) node[fill=white] {$(x_{0,00}^0,s_{0,00}^{0(0)},s_{0,00}^{0(1)},x_{0,00}^1,s_{0,00}^{1(0)},s_{0,00}^{1(1)},x_{0,00}^2,s_{0,00}^{2(0)},s_{0,00}^{2(1)})$};
% \draw (4.25,2.25) node[fill=white] {$(x_{0,01}^0,x_{0,01}^1,x_{0,01}^2)$};
% \draw (5,1.5) node[fill=white] {$(x_{0,10}^0,x_{0,10}^1,x_{0,10}^2)$};
% \draw (5.75,0.75) node[fill=white] {$(x_{0,11}^0,x_{0,11}^1,x_{0,11}^2)$};

\draw (-3,1.7) to node[sloped,anchor=center,fill=white] {$\scriptstyle 00$} (-2.9, 3.7);
\draw[densely dashed] (-1, 1.9) to node[sloped,anchor=center,fill=white] {$\scriptstyle 01$} (0, 2.4);
\draw[densely dashed] (-.5,1.5) to node[sloped,anchor=center,fill=white] {$\scriptstyle 10$} (0.5, 1.5);
\draw[densely dashed] (-1,1) to node[sloped,anchor=center,fill=white] {$\scriptstyle 11$} (0, 0.5);

\draw[densely dashed] (8.5,-.45) to node[sloped,anchor=center,fill=white] 
{$\scriptstyle 000$} (8.85,.85);
\draw[densely dashed] (8.55,-.55) to node[sloped,anchor=center,fill=white] 
{$\scriptstyle 001$} (9.35,.55);
\draw[densely dashed] (8.6,-.7) to node[sloped,anchor=center,fill=white] 
{$\scriptstyle 010$} (9.75,.1);
\draw[densely dashed] (8.6,-.8) to node[sloped,anchor=center,fill=white] 
{$\scriptstyle 011$}(9.95,-0.55);
\draw[line width=0.5mm,color=red,densely dashed] (8.65,-.9) to 
node[sloped,anchor=center,fill=white] {$\scriptstyle\bm{100}$} (9.95,-1.2);
\draw[densely dashed] (8.6,-1) to node[sloped,anchor=center,fill=white] 
{$\scriptstyle 101$} (9.7,-1.8);
\draw[densely dashed] (8.6,-.95) to node[sloped,anchor=center,fill=white] 
{$\scriptstyle 110$} (9.2,-2.3);
\draw[densely dashed] (8.5,-1.1) to 
node[sloped,anchor=center,fill=white] 
{$\scriptstyle 111$} (8.7,-2.6);

\draw[color=red,<-] (10.15,-1.3) to[out=-10,in=270] (9.5,1.9);
\draw[color=red] (2.1,4) node[right] {$\exists\,h\in\mathcal{H} : $};
\draw[color=red] (2.1,3.4) node[right] {$h(x_\varnothing)= s_\emptyset^{(1)}$};
\draw[color=red] (2.1,2.8) node[right] {$h(x_1^0)=s_1^{0(0)},~h(x_1^1)=s_1^{1(0)}$};
\draw[color=red] (2.1,2.2) node[right] 
{$h(x_{1,00}^0)=s_{1,00}^{0(1)},~h(x_{1,00}^1)=s_{1,00}^{1(0)},~h(x_{1,00}^2)=s_{1,00}^{2(0)}$};

\end{tikzpicture}

\caption{A Natarajan-Littlestone tree of depth $3$. 
Every branch is 
consistent with a concept $h\in\mathcal{H}$. This is 
illustrated here for one of the branches. Due to lack of space, not all nodes and
external edges are drawn. The figure is adapted from~\cite{universal}.
The root of the tree is the point $(x_{\emptyset})$ with two colors $s^{(0)}_{\emptyset} \neq s^{(1)}_{\emptyset}$. In this example, the branch picks the string $'1'$ and hence the node of the second level contains the two points $(x_1^0,x_1^1)$ and the associated colors. We proceed in a similar manner.
\label{fig:nl-tree}}
\end{figure}

An NL tree looks like a multiclass Littlestone tree whose branching factor increases exponentially with the depth of the tree and where each node at depth $t$ in the NL tree  
contains $t$ points $x_1,...,x_t$ of $\calX$ and two colorings $s^{(0)},s^{(1)}$ so that $s^{(0)}(x_i) \neq s^{(1)}(x_i)$ for all $i \in [t]$. Crucially, this structure encapsulates the notion of $N$-shattering 
%and of the Natarajan dimension 
in the combinatorial structure of a Littlestone tree. Intuitively, along each path in the NL tree, we encounter $N$-shattered sets of size increasing with the depth. Using these two definitions, we can state our second result which is a complete characterization
of the optimal rates achievable for any given concept class $\calH \subseteq [k]^\calX$.
\begin{theorem}
\label{theorem:trichotomy}
Fix a constant positive integer $k$.
Consider a hypothesis class $\calH \subseteq [k]^\calX$ with $|\calH| > k+2$. Then, one of the following holds for any $n \in \nats$ in the realizable case:
\begin{itemize}
    \item If $\calH$ does not have an infinite multiclass Littlestone tree, then it is learnable at an optimal rate $e^{-n}$.
    \item If $\calH$ has an infinite multiclass Littlestone tree but does not have an infinite Natarajan-Littlestone tree, then it is learnable at an optimal rate $1/n$.
    \item If $\calH$ has an infinite Natarajan-Littlestone tree, then it requires arbitrarily slow rates.
\end{itemize}
\end{theorem}
It is clear that the above result implies \Cref{theorem:tri-main}.
We remark that not only the achievable rates are different compared to the uniform setting, but also the algorithms we use
to get these rates differ vastly from ERM.
We sketch the main techniques in \Cref{section:tri-sketch}. For the formal proof, see \Cref{section:tri-proof}.
We briefly summarize our main
technical contributions in this setting: using the pre-described complexity 
measures of $\calH$, we introduce novel Gale-Stewart games (see \Cref{section:notation and preliminaries} for a definition) which lead to new learning algorithms
that achieve the optimal learning rates. Also, we extend the lower
bounds from~\cite{universal} that hold for binary classification
to the multiclass setting using our combinatorial measures. %Finally, we introduce a third combinatorial structure and use it to reduce multiclass classification to binary classification when we are aiming for linear rates. 
Our constructions raise interesting questions about how
the equivalence between various combinatorial dimensions, like the Graph
dimension and the Natarajan dimension, which is established in the PAC 
setting translates to the universal setting. For further details, we refer to \Cref{section:overview} and \Cref{section:gl-trees}.

 \subsection{Partial Multiclass Learning: Our Results}
\label{section:partial}
As we mentioned earlier, traditional PAC learning cannot capture data-dependent assumptions. Inspired by~\cite{alon2022theory}, we slightly modify the basic multiclass learning problem in a quite simple manner: instead of dealing with concept classes $\calH \subseteq \{0,1,...,k\}^\calX$ where each concept $h \in \calH$ is a \textbf{total function} $h : \calX \to \{0,1,...,k\}$, we study \textbf{partial concept classes} $\calH \subseteq \{0,1,...,k,\star\}^\calX$, where each concept $h$ is now a \textbf{partial function} and $h(x) = \star$ means that the function $h$ is \textbf{undefined} at $x$. We define the support of $h$ as the set $\mathrm{supp}(h) = \{ x \in \calX : h(x) \neq \star \}.$

To illustrate the power of partial classes, we comment that the fundamental class of $d$-dimensional halfspaces with margin $\gamma > 0$  and $k$ labels can be cast as a partial class $\calH_\gamma = \{ h_W : W \in \reals^{k \times d} \}$, where $h_W(x) = i \in [k]$  if  $(W_i - W_j) \cdot x \geq \gamma$   for all $j \neq i$  and  $h_W(x) = \star$ otherwise~\cite{alon2022theory}. As another example,
we can express the constraint that the data have to be
in a low-dimensional space by considering the partial concept class $\calH = \left\{h:\reals^d \rightarrow \{0,1,...,k,\star\}:
\mathrm{dim}\left(\mathrm{supp}(h) \right) \ll  d \right\}$, where $\mathrm{dim}(S)$ captures the dimension of the set of points in $S$.
%\grigorisnote{To illustrate the power of partial classes we note that given a set of constraints $\calC$, the class $\calH$ that satisfies these constraints is $\calH := \{h: \calX \rightarrow \{\star, 0, 1\}: (h(x_1),\ldots,h(x_m)) \text{ satisfies the constraints } \calC, \forall m \in \nats\}.$}

 We characterize multiclass PAC learnability of partial concepts in the realizable setting. A distribution $P$ on $\calX \times \{0,1,...,k\}$ is \textbf{realizable} by $\calH$ if,
almost surely, for any $n$, a training set $(x_i,y_i)_{i \in [n]} \sim P^n$ is realizable by some partial concept $h \in \calH$, i.e., $\{x_i\}_{i \in [n]} \subseteq \mathrm{supp}(h)$ and $h(x_i) = y_i$ for all $i \leq n$. For a partial concept $h$ and a distribution $P$ on $\calX \times \{0,1,...,k\}$, we let $\mathrm{er}_P(h) = \Pr_{(x,y) \sim P}[h(x) \neq y]$, i.e., whenever $h$ outputs $\star$ it is counted as a mistake. We mention that the standard combinatorial measures such as the VC or the Natarajan dimension \emph{naturally extend to the partial setting}; e.g., a partial class $\calH$ VC shatters a set of points if any \emph{binary} pattern is realized by $\calH$ (we forget about $\star$).
\begin{definition}
[Multiclass Partial PAC Learnability~\cite{alon2022theory}]
\label{definition:partial-pac}
A partial concept class $\calH \subseteq \{0,1,...,k,\star\}^\calX$ is \textbf{PAC learnable} if for every $\eps, \delta \in (0,1)$, there exists a finite $\calM(\eps, \delta) \in \nats$ and a learning algorithm $\mathbb A$ such that, for every distribution $P$ on $\calX \times \{0,1,...,k\}$ realizable with respect to $\calH$, for $S \sim P^{\calM(\eps, \delta)}$, it holds that
$
\Pr_{S} \left[  \mathrm{er}_P(\mathbb A(S) ) \leq \eps \right] \geq 1-\delta\,.
$
The sample complexity of $\mathbb A$ is the value $\calM(\eps, \delta)$ and the optimal sample complexity is the minimum possible value of $\calM(\eps, \delta)$ for any given $\eps, \delta$.
\end{definition}

We provide a combinatorial characterization of multiclass PAC learnability in the partial setting with a bounded number of labels. We additionally give bounds for $\calM(\eps, \delta)$ in our more general \Cref{theorem:partial-pac-learning-sc}.
We mention that this result is essentially implied by combining the results of \cite{moran-dinur,alon2022theory} and we state (and prove) it for completeness.
\begin{theorem}
\label{theorem:partial-pac-learning}
Fix a positive constant $k \in \nats$.
For any partial concept class $\calH \subseteq \{0,1,\ldots,k,\star\}^{\calX}$, it holds that $\calH$ is PAC learnable if and only if
$\mathrm{Ndim}(\calH) < \infty$.
\end{theorem}
Here, $\mathrm{Ndim}(\calH)$ is the Natarajan dimension of $\calH$ (see \Cref{definition:natarajan dimension} and \Cref{rem:vc}). At first sight this result may not seem surprising. However, its proof is different from the standard multiclass PAC learning (for constant $k$), which goes through uniform convergence and ERM. In fact, such tools provably fail \cite{alon2022theory}. For a sketch, see below \Cref{theorem:partial-pac-learning-sc}.
%\grigorisnote{TODO: If we show the equivalence between $\Psi$-dimensions in this setting add discussion here.}
We complement the above structural result with some additional insight which sheds light towards the, perhaps unanticipated, complexity of partial concept classes. 
To this end, we discuss the question of \emph{disambiguation} \cite{kantorovich-disambiguation-open,alon2022theory}: Can a partial Natarajan class (i.e., with finite $\mathrm{Ndim}$) be represented by total Natarajan classes? To address this task, the notion of disambiguation is required: roughly, a total class $\overline{\calH}$ disambiguates the partial class $\calH$ if every partial concept $h \in \calH$ can be extended to some total concept $\overline{h} \in \overline{\calH}$, i.e., $\overline{h}$ agrees with $h$ in the support of $h$ and assigns to the undefined points some labels. For a formal definition of disambiguation, see \Cref{definition:disambiguation}. For the case $k=1$,  \cite{alon2022theory} provided an easy-to-learn partial class that cannot be represented by any total class of bounded VC dimension using, surprisingly, some recent results from communication complexity and graph theory (see e.g., \cite{balodis2022unambiguous}). We extend this result to the multiclass setting using Sauer's lemma, which provides a bound on the growth function \cite{understanding-ml}, and tools from the binary impossibility result. 
\begin{theorem}
[Informal, see \Cref{theorem:disambiguation}]
\label{theorem:disambiguation-inf}
Fix $k \in \nats$.
For any $n \in \nats$, there exists a  class $\calH \subseteq \{0,1,...,k,\star\}^{\nats}$ with $\mathrm{Ndim}(\calH) = O_k(1)$ such that any disambiguation $\overline{\calH}$ of $\calH$ has $\mathrm{Ndim}(\overline{\calH}) = \infty$.
\end{theorem}

We denote with $O_k(1)$ a constant that depends on $k.$
Via \Cref{theorem:partial-pac-learning}, the above partial class $\calH$ is PAC learnable; however, \emph{any} disambiguation of $\calH$ causes a blow-up to the Natarajan dimension.
This result showcases the complexity of partial concepts. We briefly outline
the main technical contributions in this regime: we extend the main
results of~\cite{alon2022theory} for binary classification to the 
multiclass setting using appropriate complexity measures. Recall that the combinatorial Sauer-Shelah-Perles (SSP) lemma \cite{sauer,understanding-ml} bounds the size of a (total) class $\calH \subseteq \{0,1\}^n$ by $\sum_{i = 0}^{\mathrm{VC}(\calH)} \binom{n}{i}$. Notably,~\cite{alon2022theory} showed that this lemma does \emph{not} hold true for partial concept classes. To obtain our disambiguation result, we prove that the second variant of the SSP lemma, which uses the growth function of the class \cite{understanding-ml}, \emph{does} hold in the partial regime, which may be of independent interest.

\subsection{Preliminaries and Related Work}
\label{section:notation and preliminaries}

%\subsection{Preliminaries}
\subsubsection{Preliminaries}
In this section, we discuss some preliminaries for this paper. We refer to \Cref{appendix:preliminaries} for further details.
%We set $[k] = \{0,1,...,k\}$ and $\nats = \{1,2,...\}$.
\paragraph{Gale-Stewart Games \& Ordinals}
An important tool we leverage to establish our results in the universal learning setting is the theory of \emph{Gale-Stewart} (GS) \emph{games} \cite{gale1953infinite}.  Every such game consists of two players, a learner $\texttt{P}_L$ and an adversary $\texttt{P}_A$, and is played over an infinite sequence of discrete steps. In each step, the adversary presents some point $x_t \in \calX_t$ to the learner and the learner
picks a response $y_t\in\calY_t$. 
If some predefined condition gets violated at some step $t$, the game terminates and the learner wins. On the other hand, if the condition does not get violated during the infinite sequence of these time-steps, the adversary wins. The main property which characterizes the GS game is that the winning strategy of the learner is \emph{finitely-decidable}, i.e., she knows that she has won the game after playing a finite number of rounds. 
%A more detailed treatment of this topic can be found in \Cref{appendix:gs}. 
\cite{gale1953infinite,kechris2012classical,hodges1993model} proved that either $\texttt{P}_L$ or $\texttt{P}_A$
has a winning strategy, i.e., playing that strategy makes them win regardless of the opponents actions.
Similar to~\cite{universal}, the main reason we use GS games in this work is to obtain functions that are \emph{eventually} correct. We deal with GS games that are finite but there is not a bound on
the number of steps they need to terminate. It turns out that an \emph{ordinal} is the right notion
to use for measuring the
remaining steps needed for the game to terminate. 
%The important property we
%care about is that the ordinals that are associated with
%these games do not admit \emph{infinitely decreasing chains.}
The key motivation behind the use of ordinal numbers is to capture the intermediate state between an infinite tree 
%(whose depth corresponds to the absolute Cantor's infinity $\Omega$) 
and a uniformly bounded tree.
%(whose depth corresponds to some finite ordinal number).

\subsubsection{Related Work} 
Classification with multiple labels is extensively studied and, for the setting with bounded $k$, PAC learnability is well-understood \cite{natarajan1989learning, characterization, understanding-ml,multiclass-erm}. The works of \cite{vapnik2015uniform,blumer} provide a fundamental dichotomy/equivalence between the finiteness of the VC dimension and binary classification (PAC learnability). \cite{natarajan1988two} and \cite{natarajan1989learning} extend the PAC framework to the multiclass setting by providing the notions of the Graph and the Natarajan dimension. When the number of labels $k$ is a finite constant, then these two dimensions both characterize PAC learnability since $\mathrm{Ndim}(\calH) \leq \mathrm{Gdim}(\calH) \leq \mathrm{Ndim}(\calH) \cdot O(\log(k))$ \cite{multiclass-erm}. Afterwards, \cite{characterization} and \cite{haussler-long} provide a combinatorial abstraction which captures as special cases e.g., the Graph and Natarajan dimensions and Pollard's pseudo-dimension. In this general setting, \cite{characterization} identify the notion of \emph{$\Psi$-distinguishers} that characterize PAC learnability when $k$ is bounded.
More to that, uniform convergence still applies when the number of classes is bounded and, so, PAC learning provides the ERM principle for algorithm design. We remark the situation gets much more complicated when the number of labels is not bounded \cite{daniely2014optimal,moran-dinur,rubinstein2009shifting, multiclass-erm}. For instance, \cite{daniely2014optimal} show that the ERM principle does not apply in this case. In a recent breakthrough, \cite{moran-dinur} show that the DS dimension captures learnability and the Natarajan dimension provably fails to achieve this. 

Our work provides two theoretical perspectives complementing and extending the standard multiclass PAC learning.
For the \textbf{universal rates}, the seminal work of \cite{universal} provided a similar trichotomy for the binary setting (we obtain their results by setting $k=1$). The gap between exponential and linear rates was studied by \cite{schuurmans} in some special cases. Also, \cite{antos-lugosi} showed that there exist concept classes for which no improvement on the PAC learning rate is possible in the universal setting.
A natural approach to obtain results for the multiclass setting is via reductions to the binary setting. 
In the exponential rates setting, a first idea would be to consider for the class $\calH \subseteq [k]^\calX$ the binary restrictions $\calH|_i = \{ h_i : h \in \calH \},$ where $h_i$ denotes the $i$-th bit of the output of $h$. In order to obtain the desired result for exponential rates, one has to prove that if $\calH$ does not have an infinite multiclass Littlestone tree, then every $\calH|_i$ does not 
have an infinite binary Littlestone tree. However, it is not clear how to obtain such a result. This is why we design statistical learning algorithms for the multiclass setting from scratch, following the conceptual roadmap introduced in \cite{universal}. We note that the existence of multiple labels requires various novel ideas compared to the binary setting. We introduce novel Gale-Stewart games (see \Cref{algorithms:universal1}, \Cref{algorithms:universal2}) in order to provide winning strategies for the learning player both in the exponential and the linear rates settings. Finally, in terms of reductions for the linear rates setting, we provide a sufficient condition for learnability at a linear rate using a reduction to the binary setting; however, it is again not clear how to use these complexity measures in order to obtain lower bounds (see \Cref{open-question}). In general, the idea of ``universal learning'' has been studied in various concepts such as universal consistency \cite{stone1977consistent,devroye2013probabilistic, hanneke2020learning, hanneke-kontorovich,blanchard2021universal4,hanneke2022universally,blanchard2022universal1,blanchard2022universal2,blanchard2022universal3} and active learning \cite{hanneke2009theoretical,balcan2010true,hanneke2011rates,hanneke2015minimax}. For an extended discussion, we refer to~\cite{universal}.
%A natural approach would be via contradiction: Assuming that for some index, the class $\calH|_i$ has an infinite tree, then one could build an infinite multiclass tree. This proof is not clear to us. On the contrary, we manage to prove the other direction (which cannot be applied to give a lower bound on the learning rate). 
For the \textbf{partial concepts} setting, our work builds on the seminal work \cite{alon2022theory} and uses tools from \cite{moran-dinur}. The work of \cite{alon2022theory} shows that the algorithmic landscape behind partial concept classes is quite elusive. We extend some of their results to learning scenarios where one deals with more than two labels. Our contributions draw ideas from various directions, namely the one-inclusion hypergraph algorithm \cite{haussler1994predicting,rubinstein2009shifting,daniely2014optimal,moran-dinur} which we define formally in \Cref{appendix:one-inclusion}, the Sauer-Shelah-Perles lemma \cite{sauer} and recent breakthroughs in the intersection of graph theory and complexity theory concerning the Alon-Saks-Seymour problem \cite{huang2012counterexample, amano2014some, goos2015lower, shigeta2015ordered, ben2017low, goos2018deterministic,  alon2022theory, balodis2022unambiguous}.

\begin{remark}
[Connection between Universal Rates and Partial Concept Classes]
There is an interesting and intrinsic connection between universal rates and partial concepts: in the universal learning setting, the first step of the approach is to use the data to simulate some Gale-Stewart games and show that, with high probability, most of them will have ``converged’’, i.e., the function that corresponds to the learning strategy of the learner will be correct. In turn, this defines some data-dependent constraints. For instance, assume that $g$ is a successful NL pattern avoidance function, i.e., a function which takes as input any $\ell$ points $x_1,\ldots,x_\ell$ and any two (everywhere different) mappings $s^{(0)}, s^{(1)}$ from points to labels and returns an invalid pattern, i.e., a binary pattern $y$ of length $\ell$ that is not compatible with the definition of the Natarajan dimension (i.e., there is no function $h \in \calH$ such that if $y_i = 1$ then $h(x_i) = s^{(1)}(x_i)$ and if $y_i = 0$ then $h(x_i) = s^{(0)}(x_i)$, for all $i \in [\ell]$). Then, we can define a partial concept class $\calH'$, the set of all functions from $X$ to $\{1,\ldots,k,\star\}$ that satisfy the constraint of this pattern avoidance function, and it has two important properties: its Natarajan dimension is bounded by $\ell$ and a learning algorithm for $\calH'$ also learns $\calH$. Hence, understanding the learnability of partial concept classes is an essential step in coming up with more natural learning strategies in the universal learning setting. 

Moreover, a unifying message of both of these settings is that going beyond the traditional PAC learning framework is essential to understanding the behavior of learning algorithms in practice. Importantly, in both of these settings ERM is not an optimal learning algorithm and the one-inclusion graph predictor is an essential part in deriving results in both theories.
\end{remark}

% \paragraph{Concurrent Work.} In a concurrent work, \cite{bousquet2022fine} provide fine-grained distribution-dependent learning curves for the binary setting. We expect similar results for the multiclass setting, which constitutes an interesting direction for future work. 
% 

\subsection{Summary of Technical Contributions}
Let us start with the technical challenges concerning the universal multiclass classification setting: the first natural idea is to reduce the problem to the binary setting, i.e., use the algorithms from \cite{universal} as a black-box to derive algorithms for the multiclass case. However, this approach introduces various technical challenges when dealing with the induced binary hypothesis classes. We discuss in detail these challenges for the exponential rates case.

First, it was not clear to us how to prove that if the original class $\calH$ does not have an infinite multiclass Littlestone tree, then each of the induced binary classes satisfies this property as well. As a result, we developed new algorithms from scratch. There were several technical challenges that we had to overcome. The straightforward extension of the Gale-Stewart games that appear in \cite{universal}, i.e., where the adversary presents a point $x$ and the learner picks a label in $[k]$ (imitating the online learning game), does not seem to work. Therefore, we propose a more involved Gale-Stewart game where the learner proposes both a point $x$ and two potential labels for it, and the learner has to commit to one of the two. This game allows us to obtain tight upper and lower bounds that depend on the finiteness of the multiclass Littlestone tree. Note that the function that is obtained from the winning strategy of the leaner in this game cannot be directly used to make predictions, since it takes as input two labels. Hence, our predictor, essentially, plays a tournament between all the potential labels using this function we just mentioned. In the next step of the approach, i.e., in the transition from the online setting to the statistical setting, one needs to show how to simulate this Gale-Stewart game using the data that the learner has access to. Since our Gale-Stewart game is more complex than the one in \cite{universal}, simulating it using data becomes more challenging.

In the case of linear rates these challenges are even more technically involved and this is related to the fact that the Natarajan-Littlestone tree has a more complicated structure than the VCL tree (for example, we need to check all the possible mappings from some given points to labels). The Gale-Stewart game that handles the case of linear rates goes as follows: the adversary presents to the learner a tuple of points $x$, and two (different) colorings for these points. Similarly as before, we could not use a simpler game to obtain the result. Subsequently, as in the exponential rates case, simulating this game using data becomes more complicated than in \cite{universal}.
Proving the arbitrarily slow rates lower bound required some extra care than in the VCL tree in order to guarantee that the designed distribution is realizable. 

In the setting of partial concept classes, our main contribution is a proof that an alternative version of the SSP lemma, which bounds the growth function, holds in this setting. In contrast, \cite{alon2022theory} showed that the traditional version of the lemma (the "combinatorial" one) does not hold for partial concept classes. We remark that these two versions are equivalent for total concept classes. This lemma allows us to "tensorize" the disambiguation lower bound that was proved in \cite{alon2022theory} for binary classes and establish it for the multiclass setting.  

\subsection{Future Directions}
\label{open}
We deal with the settings of universal learning and partial concept classes, two fundamental questions \cite{universal,alon2022theory} which are witnessed in real-life applications, nevertheless the classical theory fails to explain. Our results raise various interesting questions for future work (apart from \Cref{open-question}). First, it would be interesting to extend our results to the agnostic setting. Second, for the universal setting, we believe it is an important next step to shed light towards multiclass universal learning with an unbounded number of labels (whose uniform learnability was recently characterized by \cite{moran-dinur}).
We shortly mention that the analysis of the exponential case
still holds even for a countably infinite number of labels.
Moreover, for the partial concepts setting, the work of \cite{alon2022theory} leaves numerous fascinating open questions for the binary setting that can be asked in the multiclass setting too. In general, our work along with its seminal binary counterparts \cite{universal,alon2022theory} shows that the algorithmic landscape occuring in practice is quite diverse and the ERM principle is provably insufficient. It is important to come up with principled algorithmic strategies that bring theory closer to practice.

%\section{Conclusion and Societal Impact}\label{section:conclusion}Our work focuses on multiclass learnability with a bounded number of labels and provides two new theories that go beyond the traditional PAC model. We deal with the settings of universal learning and partial concept classes, two fundamental questions \cite{universal,alon2022theory} which are witnessed in real-life applications, nevertheless the classical theory fails to explain. Our results raise various interesting questions for future work (apart from \Cref{open-question}). First, it would be interesting to extend our results to the agnostic setting. 
%Second, for the universal setting, we believe it is an important next step to shed light towards multiclass learning with unbounded labels (whose uniform learnability was recently characterized by \cite{moran-dinur}).  Moreover, for the partial concepts setting, the work of \cite{alon2022theory} leaves numerous fascinating open questions for the binary setting that can be asked in the multiclass setting too. In general, our work along with its seminal binary counterparts \cite{universal,alon2022theory} shows that the algorithmic landscape occuring in practice is quite diverse and the ERM principle is provably insufficient. It is important to come up with principled algorithmic strategies that bring theory closer to practice. This work does not have any negative social impacts.

\section{Technical Overview \& Proof Sketches}
\label{section:overview}
In this section we briefly discuss the technical details and provide 
proof sketches of our main results.

\subsection{Technical Overview of Universal Multiclass Learning}
\label{section:tri-sketch}
In the universal multiclass setting, we provide three lower bounds and two algorithms in order to get the desired trichotomy of \Cref{theorem:tri-main}.
The first lower bound states that no class $\calH$ is learnable at rate faster than exponential (see \Cref{proposition:exp-lower-bound}). Our first essential contribution is that any $\calH \subseteq [k]^\calX$ is learnable at (optimal) rate $e^{-n}$ if and only it has no infinite multiclass Littlestone tree. For this task, we provide \Cref{algorithms:universal1} that achieves this rate. Our approach uses tools from infinite game theory (Gale-Stewart games) and set theory (ordinals) in order to show that there exists an online learning algorithm that makes only a finite number of mistakes. We denote this key subroutine with $g_t$ in \Cref{algorithms:universal1}. In fact, this subroutine can be seen as an extension of the well-known multiclass Standard Optimal Algorithm (SOA) to ordinal numbers. 
Our final algorithm runs the above subroutine $g_t$ on multiple batches using data-splitting techniques and then takes a majority vote; the intuition behind this step is that the majority vote of various executions of our algorithm will be much better concentrated than a single execution and will achieve the desired exponential rate. The main technical challenge is to construct $g_t$.

\begin{algorithm}[ht!]
  \caption{Exponential Rates Algorithm for Universal Multiclass Learning}
  \setlength{\columnseprule}{1pt}
\def\columnseprulecolor{\color{black}}
  \begin{multicols}{2}
    \begin{algorithmic}
    \State \texttt{\color{blue}Exponential Rates \color{black}}
    \State Let $g_t:\calX \rightarrow Y$ be an \emph{eventually correct} labeling function (\Cref{theorem:strategy}).
    \State Let $(X_1,Y_1,\ldots,X_n,Y_n)$ be the training set.
    \State Estimate $\hat{t}_n$ such that $\Pr[\er(g_{\wh{t}_n})] \leq 3/8$.
    \State Break the training set into $N = n/\wh{t}_n$ batches.
    \State Create $N$ copies of $g$: $g^1,...,g^N$ where the $i$-th copy is trained on the $i$-th batch.
    \State To predict the label of some $x \in \calX$, take the majority vote
    over all $g^i_{\wh{t}_n}$.
    \end{algorithmic}
    \columnbreak
    \begin{algorithmic}
    \State \texttt{\color{blue}Exponential GS Game\color{black}}
    \State For any $t \in \nats:$
    \State \hspace{2.2mm}$\texttt{P}_A$ picks $\kappa_t = (\xi_t, y^{(0)}_t, y^{(1)}_t )\in \calX \times [k] \times [k].$ 
    \State \hspace{2.2mm}$\texttt{P}_A$ reveals $\kappa_t$ to the learner $\texttt{P}_L$.
    \State \hspace{2.2mm}$\texttt{P}_L$ chooses $\eta_t \in  \left\{0, 1\right\}$.

    \State $\texttt{P}_L$ wins the game if for some $t \in \nats$
    \[
    \{ h \in \calH : h(\xi_\l) = y_\l^{(\eta_\l)} ~ \forall \l \in [1..t] \} = \emptyset\,.
    \]
    %$\calH_{ \vec \kappa, \vec \eta, t} = \{ h \in \calH : h(\xi_\l) = y_\l^{(\eta_\l)} ~ \forall \l \in [1..t] \}.$
    \end{algorithmic}
  \end{multicols}
\label{algorithms:universal1}
\end{algorithm}

Our approach to construct the eventually correct function $g_t$ passes through the adversarial online learning setting. As a first step, we introduce the standard multiclass online learning game \cite{multiclass-erm} between an adversary and a learner. In this game, the adversary picks a point $x_t \in \calX$ and the learner guesses its true label $y_t \in [k]$. 
In the standard mistake bound model \cite{littlestone1988learning,multiclass-erm}, the learner's goal is to achieve a \emph{uniformly} bounded number of mistakes (and this is associated with the multiclass Littlestone dimension and the Standard Optimal Algorithm). We extend this model to the case where we can guarantee a \textbf{finite} number of mistakes for each realizable sequence, but without an a priori bound on the number of mistakes, i.e., this number is not \emph{uniformly bounded}. This is the motivation behind \Cref{definition:littletree-informal}. We prove that when $\calH$ does not have an infinite multiclass Littlestone tree, there exists an online learning algorithm for this setting which makes finitely many mistakes. This is exactly the eventually correct function $g_t$ of \Cref{algorithms:universal1}. To be precise, the function $g_t$ corresponds to the \emph{winning strategy} in round $t$ of the learning player in the above game.
\begin{theorem}
[Informal, see \Cref{thm:adversarial}]
\label{theorem:strategy}
For any $\calH \subseteq [k]^\calX$, if $\calH$ does not have an infinite multiclass Littlestone tree, there is a strategy $g_t, t \in \nats,$ for the learner that makes only finitely many mistakes. Otherwise, the adversary has a winning strategy.
\end{theorem}
To prove this result, we invoke the theory of Gale-Stewart games. We introduce a novel two-player game, the \texttt{Exponential GS Game} outlined in \Cref{algorithms:universal1}. The structure of this game looks like the standard multiclass online learning game but has some evident differences; the adversary not only reveals a point $\xi_t$ but also two colors for it. Then, the learner should choose between these two.
The structure of this game (while unaccustomed) is crucial and generalizes the game of \cite{universal}.
The learner wins if the class of consistent hypotheses $\calH_{\kappa_1,\eta_1,...,\kappa_t,\eta_t} = \{ h \in \calH : h(\xi_\l) = y_\l^{(\eta_\l)} ~ \forall \l \in [1..t] \}$ becomes empty after a finite number of rounds. If the game continues indefinitely, the adversary wins.
The intuition behind the definition of the class in \Cref{algorithms:universal1} is that the adversary wins as long as there is always a hypothesis in $\calH$ that is consistent with the examples (this is in parallel with the definition of an infinite path in the multiclass Littlestone tree). Using tools from Gale-Stewart games (\Cref{thm:gale-stewart winning strategies}), we manage to show that the learning player $\texttt{P}_L$ has a winning strategy if and only if $\calH$ does not have an infinite multiclass Littlestone tree. This winning strategy is in fact the \emph{ordinal Standard Optimal Algorithm}. Recall that it is possible to have unbounded multiclass Littlestone dimension while not having an infinite multiclass Littlestone tree. 

In order to quantify this intermediate state (between uniformly bounded and truly infinite), we invoke the theory of ordinal numbers and  introduce the \emph{ordinal multiclass Littlestone dimension}, which quantifies ``how infinite'' the multiclass Littlestone dimension is. Hence, the learner's strategy is to play according to the SOA where the standard Littlstone dimension is replaced by the ordinal one (see \Cref{section:from-online}).
We note that the analysis of the above game constitutes
an important technical contribution in the exponential rates setting.
More to that, we believe that the link between the multiclass SOA and ordinals' theory is an interesting conceptual step.   
%Having obtained this function $g_t$, we have to apply our algorithm to random training data. Crucially, simply applying this online learning algorithm to the data is not sufficient to yield fast rates. We have to apply some additional algorithmic tools in order to obtain a statistical learning algorithm. To this end, we follow the approach of \cite{universal} and apply the online learning algorithm $g_t$ into different batches of the training set and finally aggregate the various obtained classifiers using the majority vote. This technique provably guarantees that the overall classifier will achieve an (optimal) exponential learning rate (see \Cref{section:to-prob}). 
For further details concerning the exponential rates, we refer to \Cref{section:exponential}.
Our next result (\Cref{theorem:inf-littlestone-linear-lower-bound}) is a lower bound indicating a sharp transition in the learning rate: A class $\calH \subseteq [k]^\calX$ that has an infinite multiclass Littlestone tree is learnable no faster than $1/n$. Its proof uses the probabilistic method and shows that for any learning algorithm $\wh{h}_n$, there exists a realizable distribution $P$ over $\calX \times [k]$ such that $\E[\mathrm{er}(\wh{h}_n)] \geq \Omega(1/n)$ for infinitely many $n$, when $\calH$ has an infinite multiclass Littlestone tree.

We can now move to the linear rates setting where the situation is significantly more involved technically. In this setting, we show that any $\calH \subseteq [k]^\calX$ is learnable at rate $1/n$ if and only if it has no infinite NL tree. The structure of an NL tree indicates that the notion of the Natarajan dimension (which characterizes learnability in the uniform setting) is invoked in order to control the complexity/expressivity of our concept class. 
Compared to the exponential rates setting, we shift our goal from hoping for a finite number of mistakes to looking for a control over the \emph{model complexity}. This model complexity is quantified by the notion of an \emph{NL pattern in the data} (see \Cref{def:pattern}). Conceptually, the design of the algorithm for the linear rates follows a similar path as in the exponential case; we first develop an infinite game which makes use of the structure of the NL trees. This \texttt{Linear GS Game} is also original and can be found in \Cref{algorithms:universal2}.
We remark that the precise structure of the game is quite important for our results and various modifications of it seem to fail.

\begin{algorithm}[ht!]
  \caption{Linear Rates Algorithm for Universal Multiclass Learning}
  \setlength{\columnseprule}{1pt}
\def\columnseprulecolor{\color{black}}
  \begin{multicols}{2}
    \begin{algorithmic}
    \State \color{blue}\texttt{Linear Rates}\color{black}
    \State Let $g_t:\calX^t \times [k]^t \times [k]^t \rightarrow \{0,1\}^t$ be an \emph{eventually correct} NL-pattern avoidance function.
    \State Let $(X_1,Y_1,\ldots,X_n,Y_n)$ be the training set.
    \State Estimate $\hat{t}_n$ such that $\Pr[\er(g_{\wh{t}_n})] \leq 3/8$.
    \State Break the training set into $N = n/\wh{t}_n$ batches.
    \State Create $N$ copies of $g$: $g^1,...,g^N$, where the $i$-th copy is trained on the $i$-th batch.
    \State Create $N$ copies of the one-inclusion graph predictor, each copy is equipped with $g^i_{\wh{t}_n}$.
    \State To predict the label of $x$, take the majority vote
    over all the one-inclusion graph predictors.
    \end{algorithmic}
    \columnbreak
        \begin{algorithmic}
    \State \color{blue}\texttt{Linear GS Game}\color{black}
    \State For any $t \in \nats:$
    \State \hspace{5mm} $\texttt{P}_A$ picks a point $\xi_t = \left(\xi_t^{(0)}, ..., \xi_t^{(t-1)}, s_t^{(0)}, s_t^{(1)} \right)$
    \State \hspace{5mm} where (i) $ \xi_t \in \calX^{t} \times [k]^t \times [k]^t$ and
    \State \hspace{5mm} (ii) $s_t^{(0)}, s_t^{(1)}$ s.t. $s_t^{(0)}(\xi_t^{(i)}) \neq s_t^{(1)}(\xi_t^{(i)})~\forall i$. 
    \State \hspace{5mm} $\texttt{P}_A$ reveals $\xi_t$ to the learner $\texttt{P}_L$.
    \State \hspace{5mm} $\texttt{P}_L$ chooses a pattern $\eta_t \in  \left\{0, 1\right\}^t$.
    
    \State $\texttt{P}_L$ wins the game if for some $t \in \nats$
    \[
    %\calH_{\xi_1, \eta_1, ...., \xi_t, \eta_t } =
    \left \{ 
h \in \calH  
~\textnormal{s.t.}
\begin{array}{ll}
      h(\xi_z^{(i)}) = s_z^{(\eta_z^{(i)})}(\xi_z^{(i)}) \\
      \textnormal{for}~0\leq i < z, z \in [1..t] \\
\end{array} 
\right \}=
    \emptyset\,.
    \]
    %$ \calH_{\xi_1, \eta_1, ...., \xi_t, \eta_t } = \left \{ 
%h \in \calH  
%~\textnormal{s.t.}
%\begin{array}{ll}
 %     h(\xi_z^{(i)}) = s_z^{(\eta_z^{(i)})}(\xi_z^{(i)}) \\
 %     \textnormal{for}~0\leq i < z, z \in [1..t] \\
%\end{array} 
%\right \}\,.$
        \end{algorithmic}
  \end{multicols}
\label{algorithms:universal2}
\end{algorithm}

In the \texttt{Linear GS Game}, the adversary picks $t$ points and two colorings for these points which are everywhere different. Then the learner responds with an \emph{NL pattern} $\eta_t \in \{0,1\}^t$ with the goal that there is \emph{no} $h \in \calH$ that is $N$-consistent with the adversary's input. Hence, the learner aims to find \emph{forbidden NL patterns in the data}.
The finiteness of the NL tree implies the existence of a winning strategy for the learner in the game and, hence, an algorithm (which one can construct) for learning to rule out NL patterns. 
Then, we show how to simulate this game using any $\calH$-realizable sequence and utilize the learner's strategy to, eventually, find forbidden NL patterns in the data. 
The simulation of the game is another novel part for the linear rates (see \Cref{fig:pattern-algo}).
Intuitively, there exists some finite number $m$, which depends on the data sequence, such that for any collection of $m+1$ points, there exists some invalid NL pattern. The definition of the NL patterns then indicates that we cannot $N$-shatter any collection of $m+1$ points and hence we can work, in some sense, with a class whose Natarajan dimension is $m$. We can then use the one-inclusion hypergraph algorithm  \cite{haussler1994predicting, rubinstein2009shifting, daniely2014optimal, moran-dinur} (see \Cref{appendix:one-inclusion}) to get a good predictor for the data. Again a single execution of the above strategy is not sufficient, we have to use data-splitting and aggregate our collection of predictors using the majority vote.  \Cref{algorithms:universal2} achieves an optimal rate of $1/n$. Finally, we prove that a class with an infinite NL tree requires arbitrarily slow rates.\\

\begin{example}
[Universal Rates for Linear Multiclass Classification]
Let us define the class of linear classifiers with $k$ labels in $\mathbb{R}^d$ as $L = \{ h_W(x) = \arg\max (W x) | W \in \mathbb{R}^{k \times d} \}$.
We first remark that the Natarajan dimension of $L$ is $\widetilde{\Theta}(kd)$ as noted for instance in \cite{multiclass-erm} and so for fixed $k,d$, this is a Natarajan class.
Let us discuss the complexity of learning this class in the universal setting. It is important to note that our characterization of universal multiclass learning depends only on whether the two proposed trees are infinite or not. In particular, we note that class $L$ has an infinite multiclass Littlestone tree and a finite Natarajan tree (since it can only shatter a finite number of points as a Natarajan class).
We also remark that if we consider the class of linear multiclass classifiers over $\mathbb{N}^d$ (a discrete geometric space), then this class does not even have an infinite multiclass Littlestone tree and hence is learnable at an exponentially fast rate. To see this, one can use the one-versus-one reduction and observe that any classifier $h_W$ corresponds to a collection of $\binom{k}{2}$ binary classifiers which are halfspaces in $\mathbb{N}^d$ where $h_W^{i,j}(x) = \sgn((W_i - W_j) x)$ for any $i < j$. Then, due to realizability, one can use an argument from infinite Ramsey theory (see Example 2.10 from \cite{universal}) and prove that after a finite number of mistakes, they can detect the correct halfspace for any pair $i<j$. Aggregating these predictors, we get a multiclass linear classifier that enjoys an exponentially fast learning rate.
\end{example}

\subsection{Technical Overview of Partial Multiclass Learning}
\label{section:partial-sketch}
In the partial multiclass setting with a bounded number of labels, we first characterize learnability in terms of the Natarajan dimension. For the proof of \Cref{theorem:partial-pac-learning}, it suffices to show the following more fine-grained  \Cref{theorem:partial-pac-learning-sc}.

\begin{theorem}
\label{theorem:partial-pac-learning-sc}
For any partial class $\calH \subseteq \{0,1,...,k,\star\}^\calX$ with $\mathrm{Ndim}(\calH) \leq \infty$, the sample complexity of PAC learning $\calH$ satisfies
$
C_1 \cdot \frac{\mathrm{Ndim}(\calH) + \log(1/\delta)}{\eps} \leq \calM(\eps, \delta) \leq C_2 \cdot \frac{\mathrm{Ndim}(\calH) \log(k) \log(1/\delta)}{\eps} ,
$ for some constants $C_1, C_2$.
In particular, if $\mathrm{Ndim}(\calH) = \infty,$ then $\calH$ is not PAC learnable.
\end{theorem}
For the upper bound, we have to employ the  one-inclusion hypergraph algorithm (see \Cref{appendix:one-inclusion}). Following the methodology of \cite{alon2022theory}, we extend its guarantees (which hold for total concept classes) to the partial setting. This algorithm guarantees an expected error which can be boosted to a high probability result using standard concentration and boosting techniques. To show that the partial concept class $\calH \subseteq \{0,1,...,k,\star\}$ is not learnable if it has infinite Natarajan dimension, we reduce the problem to classification with total concepts and invoke the existing (standard) lower bound. The main take-away from \Cref{theorem:partial-pac-learning-sc} is that the algorithmic landscape of partial concept classes is provably elusive, as already indicated by the seminal work of \cite{alon2022theory}. To this end, we provide a second result that shows when one can apply the well-understood ERM principle (which is valid when the number of labels is bounded) with partial concepts. For details, we refer to \Cref{theorem:equivalence-total-partial}. To conclude, we address the task of disambiguation \cite{kantorovich-disambiguation-open,alon2022theory} of partial concepts  (see \Cref{definition:disambiguation}). Our proof of \Cref{theorem:disambiguation-inf} relies on an interesting observation: the seminal work of \cite{alon2022theory} showed that the combinatorial variant of the SSP lemma \cite{sauer} does not hold in this setting. This lemma has a second variant that uses the growth function \cite{understanding-ml} instead of the size of the class. We show that a natural extension of this variant for partial classes is still correct (see \Cref{lemma:ssp}). 
%Hence, for partial concepts, there is a gap between the two SSP variants. 
Using this tool and techniques from \cite{alon2022theory}, we obtain our impossibility result.\\
% \textbf{Conclusion and Societal Impact.}
% Our work focuses on multiclass learnability with a bounded number of labels, provides novel theories that go beyond the traditional PAC model and opens some interesting future directions (see \Cref{open}). This work does not have any negative social impacts.

\section{Universal Multiclass Learning: The Proof of \Cref{theorem:trichotomy}}
\label{section:tri-proof}
In this section, we prove \Cref{theorem:trichotomy} which then directly gives us \Cref{theorem:tri-main} as well. Our text is organized as follows: 
\begin{itemize}
    \item In \Cref{section:exponential}, we analyze the exponential rates case for multiclass learning. We introduce the notion of the multiclass Littlestone tree and prove \Cref{theorem:exp-rates}.
    \begin{itemize}
        \item In order to achieve this, we first analyze the problem from an adversarial online perspective in \Cref{section:from-online}. The main result in this section is \Cref{thm:adversarial}.
        \item The above result is in the adversarial setting and hence we have to transform this online algorithm into a statistical one, i.e., we have to move from the adversarial setting to the probabilistic setting. This is done is \Cref{section:to-prob}. We provide the analysis of our final algorithm in \Cref{theorem:prob-algo-exp}.
    \end{itemize}
    \item In \Cref{section:linear-exp-gap}, we show that a class with infinite multiclass Littlestone tree cannot be learned at a rate faster than linear.
    \item In \Cref{section:linear}, we introduce the notion of a Natarajan-Littlestone tree and we prove \Cref{theorem:linear-rates-algo}.
    \begin{itemize}
        \item In \Cref{section:nl-game}, we provide the important notion of a Natarajan-Littlestone game which guarantees the existence of an eventually correct algorithm when the class does not have an infinite NL tree.  
        \item In \Cref{section:pattern-avoid}, we introduce the notion of an NL pattern in the data and using the above algorithm as a pattern avoidance function.
        \item The behavior of the pattern avoidance algorithm in the probabilistic setting is given in \Cref{section:pat-avoid-prob}.
        \item The final linear rate algorithm can be found at \Cref{section:linear-rates}.
    \end{itemize}
    \item In \Cref{section:slow}, the final lower bound which is related to arbitrarily slow rates (see \Cref{theorem:slow-rates}).
    \item In \Cref{section:gl-trees}, we provide the notion of a Graph-Littlestone tree and give a sufficient condition for learning with a linear rate. More to that, we propose \Cref{open-question}.
\end{itemize}

The missing proofs are presented in \Cref{appendix:proofs for universal rates}.

\subsection{Exponential Rates for Multiclass Learning}
\label{section:exponential}
For a sequence $\vec y = (y_1,y_2,...)$, we denote $\vec y_{\leq k} = (y_1,...,y_k)$. We may also usually identify elements of $\{0,1\}^d$ with strings or a prefix of a sequence of length $d$. We begin with a formal definition of a crucial combinatorial measure, namely the multiclass Littlestone tree of a class $\calH$.

\begin{definition}
[Multiclass Littlestone Tree]
\label{definition:multi-little-tree}
A multiclass Littlestone tree for $\calH \subseteq \{0,1,...,k\}^{\calX}$ is a complete binary tree of depth $d \leq \infty$ whose internal
nodes are labeled by $\calX$, and whose two edges connecting a node to its children are labeled by two different elements in $[k]$,
such that every path of length at most $d$ emanating from the root is consistent with a concept $h \in \calH$. Typically, a multiclass Littlestone tree is a collection \[
\bigcup_{0 \leq \l < d} \left\{ x_u : u \in \{0,1\}^\l \right\} = \{ x_{\emptyset} \} \cup \{x_{0}, x_1\} \cup \{x_{00}, x_{01}, x_{10}, x_{11} \} \cup ... 
\]
such that for every path $\vec y \in \{0,1\}^d$ and finite $n < d$, there exists $h \in \calH$ so that $h(x_{\vec y_{\leq \l}}) = s_{\vec y_{\leq \l+1}}$ for $0 \leq \l \leq n$, where $s_{\vec y_{\leq \l+1}}$ is the label of the edge connecting the nodes $x_{\vec y_{\leq \l}}$ and $x_{\vec y_{ \leq \l+1}}$.
We say that $\calH$ has an infinite multiclass Littlestone tree if there is a multiclass Littlestone tree for $\calH$ of depth $d = \infty$.
\end{definition}

To give some intuition about this construction we state some of its properties. First, it is crucial to note that a class $\calH$ with a finite multiclass Littlestone tree can have infinite multiclass Littlestone dimension. In fact, a class has finite multiclass Littlestone dimension if the depth of the tree admits a \emph{uniform} upper bound. However, it may be the case that the class admits an unbounded tree, in the sense that for any finite depth, there exists a tree of that depth; nevertheless the class does not have an infinite tree. Second, a class with a finite multiclass Littlestone tree may contain trees with infinite paths (e.g., a class that contains the constant mapping $h = 1$ can shatter the rightmost path at an infinite depth). Finally, a class with finite multiclass Littlestone tree cannot have a tree with an infinite complete binary subtree, since one could use only this subtree and obtain an infinite tree. Essentially, this tree captures ``how infinite'' the multiclass Littlestone dimension of $\calH$ is: even if there is no uniform bound on the multiclass Littlstone dimension of $\calH$, whenever $\calH$ does not have an infinite tree we know that for \emph{any} $S = \bigcup_{0 \leq \l < d} \left\{ x_u : u \in \{0,1\}^\l \right\}$ there is some $n^*(S) < \infty$, which depends on the sequence $S$, so that the tree can be shattered up to level $n^*(S)$.

The goal of this section is to prove the next result.
\begin{theorem}
[Exponential rates]
\label{theorem:exp-rates}
If $\calH \subseteq [k]^{\calX}$ does not have an infinite multiclass Littlestone tree, then for any distribution $P$, $\calH$ is learnable with an exponential optimal rate.
%, i.e. $\E[\er(\hat{h}_n)] \leq C(P)e^{-c(P)n}$. 
\end{theorem}

In order to prove this result, two steps are required: first, it suffices to show that no class of hypotheses can be learned in a rate faster than exponential (lower bound) and second, we have to show that not having an infinite multiclass Littlestone tree is a sufficient for learning at an exponential rate condition (upper bound). Hence, these two directions tightly characterize universal learnability at exponential rates for multiclass classification. 

For the lower bound, the argument is essentially the same as in~\cite{universal}. Let us first give some intuition. If the distribution is supported on a finite number of points, e.g., two, there is an exponentially small probability that all the $n$ samples will contain the same point $(x,y)$, so it will not help the learner distinguish between the functions $h \in \calH$ that label this point correctly. For completeness, we present the argument formally in \Cref{proof:exp-lower-bound}.
\begin{proposition}
[Exponential Rates (Lower Bound)]
\label{proposition:exp-lower-bound}
Fix $\calH \subseteq \{0,1,...,k\}^{\calX}$. For any learning algorithm $\wh{h}_n$, there exists a realizable distribution $P$ over $\calX \times \{0,1,...,k\}$ such that $\E[\mathrm{er}(\wh{h}_n)] \geq \Omega(2^{-n})$ for infinitely many $n$. This means that $\calH$ is not learnable at rate faster than exponential.
\end{proposition}

We continue with the upper bound, i.e., the design of an algorithm that learns $\calH \subseteq [k]^{\calX}$ at an exponentially fast rate when its multiclass Littlestone tree is not infinite. Our approach consists of two main steps. First, we consider the classical online adversarial setting where a learner has to guess the label of a point that is presented to her by an adversary. We prove that if $\calH$ does not have an infinite multiclass Littlestone tree, there is a strategy the learner can employ that makes a \emph{finite} number of mistakes. Then, given such a strategy, we prove that there is an algorithm in the statistical setting which achieves exponential learning rates. The details are outlined in the subsequent sections.
%We show the following.

\subsubsection{Viewing Exponential Rates in an Online Setting}
\label{section:from-online}
In order to design our algorithms, we have to consider the following setting.
We introduce the multiclass online learning game (\Cref{fig:online-game}) that has been studied extensively (see e.g.,~\cite{multiclass-erm}). In this game, there are two players, the adversary that chooses features and reveals them to the second player, the learner whose goal is to guess a label for the given example.
\begin{figure}[ht!]
    \centering
    \begin{mdframed}[style=MyFrame,nobreak=true]
\begin{center}
\begin{enumerate}
    \item The adversary picks a point $x_t \in \calX$. % and two labels $\calY_t$.
    \item The learner guesses a value $\wh{y}_t \in [k]$% \calY_t \subseteq [k]$.
    \item The adversary chooses the value $y_t$ as true label so that 
    $y_t = h(x_t)$ 
    for some $h \in \calH$
    that is consistent with the previous examples $(x_p,y_p)$ for any $p \leq t$. 
\end{enumerate}
\end{center}
\end{mdframed}
    \caption{Realizable Online Setting}
    \label{fig:online-game}
\end{figure}
The learner makes a mistake in round $t$ whenever the guess $\wh{y}_t$ differs from the true label $y_t$. The goal of the learner is to minimize her loss and the adversary's intention is to provoke many errors to the learner. 

We say that the concept class $\calH$ is online learnable if there exists a strategy $\wh{y}_t = \wh{y}_t(x_1,y_1,...,x_{t-1}, y_{t-1},x_t)$ that makes a mistake only \emph{finitely} many times, regardless of what realizable sequence is presented by the adversary. 
Notice that compared to the classical online learning setting that asks for a bounded number of mistakes $d$, we settle for a more modest goal. The main result in this setting is the following. 

\begin{theorem}
[Strategies in the Adversarial Setting]
\label{thm:adversarial}
For any concept class $\calH \subseteq [k]^\calX$, the following dichotomy occurs.
\begin{enumerate}
    \item If $\calH$ does not have an infinite multiclass Littlestone tree, then there is a strategy for the learner that makes only finitely many mistakes against any adversary.
    \item If $\calH$ has an infinite multiclass Littlestone tree, then there is a strategy for the adversary that forces any learner to make a mistake in every round.
\end{enumerate}
\end{theorem}

Before the formal proof, we provide a proof sketch. The adversary's strategy is clear. Whenever the learner predicts a label, she must choose a different label that will cause the learner to make a mistake in that round; an infinite Littlestone tree is exactly the combinatorial structure the adversary is looking for. On the other hand, since there is no infinite multiclass Littlestone tree, the learner's predictions should  lead her to a leaf of the Littlestone tree that is defined by her interaction with the learner. This is established by a variant of the (multiclass) Standard Optimal Algorithm (SOA)~\cite{littlestone1988learning,multiclass-erm}, which works whenever $\calH$ has a finite multiclass Littlestone dimension $d$. To gain some intuition, it is instructive to consider the first step of her algorithm. The learner is presented with a point $x_1$ and she picks the label $y_1 = \argmax_{y\in[k]} \mathrm{Ldim}_k(\calH_{x_1,y})$. It is easy to see that there is at most one $y$ such that $\mathrm{Ldim}_k(\calH_{x_1,y}) = d$; if there were two of them then $\mathrm{Ldim}_k(\calH) = d+1$. Thus, by picking the one that induces the subset of the hypothesis class with the largest Littlestone dimension, the learner either does not make a mistake or gets closer to a leaf of the tree.

In our universal setting, this approach does not work immediately since it may be the case that the multiclass Littlestone tree is not infinite but the associated class has infinite multiclass Littlestone dimension. To this end, we have to introduce the \emph{ordinal multiclass Littlestone dimension}, which quantifies ``how infinite'' the multiclass Littlestone dimension is. Hence, the learner's strategy will be to play according to the \emph{ordinal Standard Optimal Algorithm}. The intuition is similar as in the classical setting, but the proof becomes more involved. To establish the result, we follow the approach of \cite{universal} and define a more general infinite game $\calG$ between the learner and the adversary and prove that it belongs to the family of the so-called Gale-Stewart games. Then, we leverage results from the theory of Gale-Stewart games and \cite{universal} in order to show that the ordinal SOA makes a finite number of mistakes whenever $\calH$ does not have an infinite multiclass Littlestone tree. Let us continue with the proof.
\begin{proof}[Proof of \Cref{thm:adversarial}]
We first introduce a two-player game $\calG$ that is played in discrete timesteps $t=1,2,\ldots$ between the adversary and the learner.
\begin{figure}[ht!]
    \centering
    \begin{mdframed}[style=MyFrame,nobreak=true]
\begin{center}
\begin{enumerate}
    \item The adversary picks a point $\kappa_t = \left(\xi_t, y^{(0)}_t, y^{(1)}_t \right)\in \calX \times [k] \times [k]$ % = (p_t, y^{(1)}_t, y^{(2)}_t) \in \calX \times [k] \times [k]$
    and reveals it to the learner.
    \item The learner chooses a point $\eta_t \in  \left\{0, 1\right\}$.
    %$\eta_t \in [k]$.% \{y^{(1)}_t, y^{(2)}_t\}$.
\end{enumerate}
\end{center}
\end{mdframed}
    \caption{Adversarial Setting - 2-Player Game}
    \label{fig:gs-game}
\end{figure}

The learning player wins in some finite round $t$ if $\calH_{\xi_1,y_1^{(\eta_1)},...,\xi_t,y_t^{(\eta_t)}} = \emptyset$. The adversary wins if the game continues indefinitely (i.e., the class of consistent hypotheses from $\calH$ never gets empty) . Clearly, the set of winning strategies for the learning player is
\[
\calW = \{ (\vec \kappa, \vec \eta) \in (\calX \times [k] \times [k] \times \{0,1\})^\infty : \exists ~1 \leq t^\star < \infty ~\textnormal{ such that }~ \calH_{\xi_1,y_1^{(\eta_1)},...,\xi_{t^\star},y_{t^\star}^{(\eta_{t^\star})}} = \emptyset \}\,.
\]

We now recall an important theorem (see \Cref{thm:gale-stewart winning strategies}) about Gale-Stewart games: In any Gale-Stewart game, exactly one of the adversary player and the learning player has a winning strategy. 

 Equipped with \Cref{thm:gale-stewart winning strategies}, we can show that the adversary has a winning strategy if and only if $\calH$ has an infinite multiclass Littlestone tree (provided that $\calG$ is a Gale-Stewart game). This is summarized in the next claim.
\begin{claim}
The game $\calG$ is a Gale-Stewart game and the adversary  has a winning strategy in $\calG$ if and only if the hypothesis class $\calH$ has an infinite multiclass Littlestone tree.
\end{claim}

\begin{proof}
It is clear from the definition of $\calW$ that every winning strategy of the learner is finitely decidable, hence $\calG$ is a Gale-Stewart game.
For the other part of the claim, notice that if $\calH$ has an infinite multiclass Littlestone tree, then the adversary's strategy is to present the learner at step $t$ the point of the tree at depth $t$ that is consistent with the execution of the game so far along with the labels of the edges that connect it with its children. By the definition of the tree, this strategy ensures that the game will keep going on forever. For the other direction, assume that the adversary has a winning strategy $\kappa_{\tau}(\eta_1,\ldots,\eta_{\tau-1}) \in \calX \times [k] \times [k]$. Then, define the multiclass Littlestone tree $\calT = \{x_{\vec{u}}: 0 \leq k < \infty, \vec{u} \in \{0,1\}^k\}$ where $x_{\eta_1,\ldots,\eta_{\tau-1}} = \xi_{\tau}(\eta_1,\ldots,\eta_{\tau-1})$ where the labels that connect $x_{\eta_1,\ldots,\eta_{\tau-1}}$ with its left, right children are $y^{(0)}_{\tau}(\eta_1,\ldots,\eta_{\tau-1}), y^{(1)}_{\tau}(\eta_1,\ldots,\eta_{\tau-1})$, respectively. We can see that $\calT$ is infinite since this is a winning strategy for the adversary.
\end{proof}

Having shown the above statement, we are ready to establish the desired dichotomy in the online game. Assume first that $\calH$ has an infinite multiclass Littlestone tree $\{ x_u \}$. The adversary's strategy is defined inductively based on the path followed so far in the game: In round $t$, set $\vec b_t = (b_1,...,b_{t-1}) \in \{0,1\}^{t-1}$ denote the path parsed so far in the tree by the two players. Then, the adversary picks $x_t = x_{\vec b_t}$. After the learner reveals her choice $\wh{y}_t$, the worst case adversary chooses as a response the branch of the Littlestone tree which does not correspond to the learner's choice (the adversary may even have two choices). By the definition of the tree, this chosen label is valid since there exists some $h \in \calH$ that realizes the path $(x_{b_1},...,x_{b_{t-1}}, x_t)$. Moreover, this choice provokes a mistake to the learning player and this is true for any round. Hence, there is a strategy for the adversary that forces
any learner to make a mistake in every round.

For the other direction, assume that the class $\calH$ does not have an infinite multiclass Littlestone tree. 
Before we describe the winning strategy of the learner, we need to introduce the notion of ordinal multiclass Littlestone dimension.
We will assign an ordinal to every finite multiclass Littlestone tree. For some preliminaries on ordinals and transfinite recursion, we refer to \cite{universal}. The rank is defined by a partial order $\prec$. We set $t' \prec t$ if $t'$ is a multiclass Littlestone tree that extends $t$ by one level, i.e., $t$ is obtained from $t'$ by removing its leaves. A multiclass Littlestone tree $t$ is minimal if it cannot be extended to a multiclass Littlestone tree of larger depth. For such a tree, we set $\mathrm{rank}(t) = 0$. If the tree $t$ is non-minimal, then it can be extended and this is quantified using transfinite recursion by
\[
\mathrm{rank}(t) = \sup\{ \mathrm{rank}(t') + 1 : t' \prec t \}\,.
\]
The rank is well-defined as long as $\calH$ has no infinite multiclass Littlestone tree (since $\prec$ is well-founded). In particular, we define
\[ 
\overline{\mathrm{Ldim}}_k(\calH) = \left\{
\begin{array}{ll}
      -1 & \textnormal{if $\calH$ is empty}\,,  \\
      \Omega & \textnormal{if $\calH$ has an infinite multiclass Littlestone tree}\,,\\
      \mathrm{rank}(\emptyset) & \textnormal{otherwise}\,. \\
\end{array} 
\right. 
\]
%When $\calH$ has no infinite multiclass Littlestone tree, we can construct a winning strategy for the learning player $P_L$.
The strategy is chosen so that $\overline{\mathrm{Ldim}}_k(\calH_{x_1,y_1,...,x_t,y_t})$ decreases in every round and the learner that follows this strategy will win the game, since the ordinals do not admit an infinite decreasing chain.
%Then, the learner has a winning strategy.
%\grigorisnote{It is not clear that the online game and the Littlestone game are exactly the same so maybe we can just describe the winning strategy of the learner and prove it using the GS game we have defined?} 
We note that this statement at first is purely existential via the theory of Gale-Stewart games. We next shortly provide a ``constructive'' way to compute the winning strategy of the learning player in the set of games we consider. Let us describe the winning strategy: The learner invokes the ordinal (multiclass) Standard Optimal Algorithm and chooses the label $y_t$ (given $x_t)$ that maximizes the ordinal multiclass Littlestone dimension, i.e., $y_t = \argmax_{y \in [k]} \overline{\mathrm{Ldim}}_k(V^y_t)$, where = $V^y_t = \{ h \in \calH_{x_1,y_1,...,x_{t-1},y_{t-1}} : h(x_t) = y\}$. The ordinal SOA at round $t = 1,2,...$ with initial set $V_0 = \calH$ works as follows:
\begin{enumerate}
       \item Receive $x_t$.
        \item For any $y \in [k]$, let $V_t^y = \{ h \in V_{t-1} : h(x_t) = y \}$.
        \item Predict $\wh{y}_t \in \argmax_{y \in [k]} \overline{\mathrm{Ldim}}_k(V_t^y)$, where $\overline{\mathrm{Ldim}}_k$ is the ordinal multiclass Littlestone dimension.
        \item Receive true answer $y_t$ and set $V_t = V_t^{y_t}$.
\end{enumerate}
This algorithm drives the game in a win-win phenomenon for the learner in every round: If the adversary forces the learner to a mistake, then she will ``prune'' the tree and set the learner closer to winning the game. Otherwise, the learner will be correct and will not incur any loss.
%Hence, the strategy $\eta_t(\xi_1,...,\xi_t)$ of the learning player is simply the guess of the ordinal SOA. 
In order to show that the ordinal SOA makes a finite number of mistakes, we couple the online game with a Gale-Stewart game. 
The idea is that every time the learner makes a mistake in the online game on point $x_t$, we advance the Gale-Stewart game by one round where we pretend that $\xi_{\tau} = x_t, y^{(0)}_{\tau} = \wh{y}_t,  y^{(1)}_{\tau} = y_t, \eta_{\tau} = y_t$. Notice that if the learner makes an infinite number of mistakes in the online game using the ordinal SOA, then the Gale-Stewart game can proceed infinitely. Hence, to conclude the proof, we need to show that in this coupled game, there is some finite point $\tau^*$ such that $\calH_{\xi_1, y_1^{(\eta_1)}, \ldots, \xi_{\tau^{\star}}, y_{\tau^{\star}}^{(\eta_{\tau^{\star}})}} = \emptyset$. The following result helps us establish that. In fact, the next lemma follows from \cite{universal}(Proposition B.8) by choosing the value of the game being the ordinal multiclass Littlestone dimension.

%\grigorisnote{in Bousquet et al. this is phrased for a general val. I think in our setting val = ordinal littlestone dimension but we can verify it.}
\begin{lemma}
[See Proposition B.8 of \cite{universal}]
\label{lem:decreasing value in GS games}
Assume that $\calH$ does not contain an infinite multiclass Littlestone tree. Then, for any choices of the adversary $\kappa_1,\ldots,\kappa_{t-1}$ up to round $t$ and for any choice $\kappa_t = (\xi_t, y^{(0)}_t, y^{(1)}_t)$ in round $t$ there is a choice $\eta_t$ of the learner such that 
\[
\overline{\mathrm{Ldim}}_k\left(\calH_{\xi_1, y_1^{(\eta_1)},\ldots,\xi_{t-1}, y_{t-1}^{(\eta_{t-1})},\xi_t,y_{t}^{(\eta_{t})}}\right)< \overline{\mathrm{Ldim}}_k\left(\calH_{\xi_1, y_1^{(\eta_1)},\ldots,\xi_{t-1}, y_{t-1}^{(\eta_{t-1})}}\right)\,.
\]
\end{lemma}

The previous result shows that for every $\xi_t$ there is at most one label $\ell_t \in [k]$ such that 
\[ 
\overline{\mathrm{Ldim}}_k\left(\calH_{\xi_1, y_1^{(\eta_1)},\ldots,\xi_{t-1}, y_{t-1}^{(\eta_{t-1})},\xi_t,\ell_t}\right) =  \overline{\mathrm{Ldim}}_k\left(\calH_{\xi_1, y_1^{(\eta_1)},\ldots,\xi_{t-1}, y_{t-1}^{(\eta_{t-1})}}\right)\,.
\]
Indeed, assume that there are two such labels $\ell_t, \ell_t'$ for some $\xi_t$. Then, if the adversary proposes the point $(\xi_t, \ell_t, \ell_t')$, there is no choice $\eta_t$ of the learner that decreases that ordinal Littlestone dimension in this round, which leads to a contradiction. Hence, the learner can pick any label as long as it is not the one that maximizes the ordinal Littlestone dimension. This is exactly how the coupled Gale-Stewart game proceeds, so we know that every time the learner makes a mistake in the online game the ordinal Littlestone dimension of the coupled game decreases. Since ordinals that are less than $\Omega$ do not admit infinitely decreasing chains, we get the desired result.

% \grigorisnote{the game is slightly different so I'm commenting the next part for now and we can see how we can to phrase it.}
% \grigorisnote{
% If we were to know a priori that the adversary always forces an error when possible, then the learner could use this strategy directly. To extend the conclusion for any arbitrary adversary, we make the game proceed to the next round only when the learner makes a mistake. Hence, the learner encapsulates her winning strategy $\eta_t$ in the following algorithm:
% \begin{enumerate}
%     \item Initially, set $t = 1$ and $f(x) = \eta_1(x)$.
%     \item In each round $T \geq 1:$ 
%     \begin{enumerate}
%         \item Predict $\wh{y}_T = f(x_T)$.
%         \item If $\wh{y}_T \neq y_T$, let $\xi_t = x_T, f(x) = \eta_{t+1}(\xi_1,...,\xi_{t},x), t = t+1$.
%     \end{enumerate}
% \end{enumerate}
% The finiteness of the tree implies that the above algorithm can only make a finite number of mistakes against any adversary. }
\end{proof}

\subsubsection{Moving from the Adversarial Setting to the Probabilistic Setting}
\label{section:to-prob}
The measurability of the winning strategies and of the learning algorithm developped in the previous section constitutes an important detail, extensively discussed in \cite{universal}, in order to move from the adversarial setting to the probabilistic one. We provide the next useful result. Its proof can be found at \Cref{proof:measurability}.
%\grigorisnote{I think this should be fine. I read the appendix in the paper (didn't fully understand everything) but I think the important part is that the set of choices of the learner is countable. As a side note, this breaks in the real-valued setting so we need to understand what can be done there.}
%\grigorisnote{I will add the lemma from~\cite{universal} which states that this is sufficient. Do you think we should add a box describing the algorithm?}
\begin{lemma}
\label{lemma:measurability}
Let $\calX$ be Polish, $k \in \nats$ be a finite constant and $\calH \subseteq [k]^\calX$ be measurable. Then, the Gale-Stewart game $\calG$ of \Cref{fig:gs-game}
has a universally measurable winning strategy. 
%As a result, the learning algorithm of \Cref{thm:adversarial} is universally measurable.
\end{lemma}

Crucially the above result states that the winning strategy $\eta_t$ of the learning player is measurable. However, the previous proof made use of the ordinal multiclass SOA algorithm, whose measurability is not directly implied. To this end, we modify
the adversarial algorithm to handle the measurability issue. The modification follows:

\begin{figure}[ht!]
    \centering
    \begin{mdframed}[style=MyFrame,nobreak=true]
\begin{center}
\begin{enumerate}
    \item Initialize $\tau \gets 1, G = \texttt{Clique}(V = [k]), f(\cdot, \cdot, \cdot) \gets \eta_1(\cdot, \cdot, \cdot)$ \Comment{$\tau$ is the mistake counter}
    \item For every round $t \geq 1:$
    \begin{enumerate}
        \item Observe $x_t$
        \item For any $y \neq y'$ with $y,y' \in [k]$, orient the edge $(y,y')$ of $G$ according to $f(x_t,y,y')$
        \item Let $G'$ the directed clique
        \item Predict $\wh{y}_t \gets \argmax_{y \in [k]} \mathrm{outdeg}(y ; G')$
        \item If $\wh{y}_t \neq y_t$, let $\xi_{\tau} \gets x_t, f(\cdot, \cdot, \cdot) \gets \eta_{\tau+1}(x_1,y_1,\ldots, x_{\tau}, y_{\tau}, \cdot, \cdot, \cdot), \tau \gets \tau + 1$
    \end{enumerate}
\end{enumerate}
\end{center}
\end{mdframed}
    \caption{Measurable Modification of Online Learning Algorithm for Exponential Rates}
    \label{fig:exp-algo}
\end{figure}

The above algorithm makes use of a tournament procedure.
The algorithm is a measurable function since (i) the winning strategy of the learner is measurable and (ii) the countable maximum of measurable functions is measurable.
This algorithm can be used in order to show that
if $\calH$ does not have an infinite multiclass Littlestone tree, then the above algorithm makes only a finite number of mistakes against any adversary. Essentially, this is due to the fact that when the winning strategy has converged to a zero-mistake prediction rule (which occurs after a finite number of mistakes), the tournament procedure will always output the correct label for the observed example. Hence, the algorithm will eventually make a finite number of mistakes in the adversarial setting.

The algorithm of \Cref{fig:exp-algo} works in the adversarial setting. We first show that it also applies to the probabilistic setting (and this is why we require the above measurability discussion). The proof is quite similar to Lemma 4.3 of~\cite{universal} and can be found at \Cref{proof:adversarial-to-probabilistic}.

\begin{lemma}
[From Adversarial to Probabilistic]
\label{lemma:adversarial-to-probabilistic}
For any distribution $P$ over $\calX \times [k]$ and for the learning algorithm $\wh{y}_t : \calX \to [k]$ of \Cref{fig:exp-algo}, we have
\[
\Pr_{S_t} \left[ \Pr_{(x,y) \sim P}[\wh{y}_t(x) \neq y] > 0 \right] \to 0 ~ as~~ t \to \infty\,,
\]
where $S_t$ is the training set $(x_1,y_1,...,x_{t-1},y_{t-1})$ of the algorithm.
\end{lemma}

The above result guarantees that the expected error of the learning algorithm tends to zero as $t$ goes to infinity, i.e., we have established that that $\E[\er(\wh{y}_t)] \rightarrow 0$ as $t \rightarrow \infty$. This means that the ordinal SOA is a consistent algorithm in the statistical setting. However, this fact is not enough to establish the exponential convergence rate. We follow the approach of \cite{universal} to come up with an algorithm that achieves this guarantee. 
The first step is to observe that there is some distribution-dependent $t^{\star}$ such that
$\Pr[\er(\wh{y}_{t^{\star}}) > 0] < 1/4$. If we were to know this $t^{\star}$ one way to get the exponential rates is the following: We divide the training set into $\Theta(n/t^{\star})$ batches and we get one classifier $\wh{y}^i_{t^{\star}}$ for every batch. Afterwards, we output as classifier the multiclass majority vote among the classifiers.
Nevertheless, we can obtain this exponential convergence even without knowing $t^\star$ by computing an estimate for this quantity using samples. Essentially, we provide this estimator in \Cref{lemma:general-estimate-time}. For the proof, we refer to the \Cref{proof:general-estimate-time}.
\begin{lemma}
\label{lemma:general-estimate-time}
For any $n \in \nats$, there exists a universally measurable $\wh{t}_n = \wh{t}_n(X_1,Y_1,...,X_n,Y_n)$ whose definition does not depend on $P$ so that the following holds. Set the critical time $t^\star \in \nats$ be such that 
\[
\Pr \left[ \Pr_{(x,y) \sim P} [\wh{y}_{t^{\star}}(x) \neq y]  > 0 \right] \leq 1/8\,,
\]
where the probability is over the training set of the algorithm $\wh{y}_t$. There exist $C,c > 0$ that depend on $P, t^\star$ but not $n$ so that
\[
\Pr[ \wh{t}_n \in T^\star ] \geq 1 - C e^{-cn}\,.
\]
where the probability is over the training of the estimator $\wh{t}_n$ and $T^\star$ is the set 
\[
T^\star = \left\{ 1 \leq t \leq t^\star : \Pr \left[ \Pr_{(x,y) \sim P} [\wh{y}_t(x) \neq y] > 0 \right] \leq 3/8 \right\}\,,
\]
where the probability is over the training of $\wh{y}_t$.
\end{lemma}

Our main result follows.

\begin{theorem}
[Exponential Rates]
\label{theorem:prob-algo-exp}
Assume that class $\calH \subseteq [k]^\calX$ does not have an infinite multiclass Littlestone tree. Then, $\calH$ admits a learning algorithm that achieves an exponentially fast rate. 
\end{theorem}
\begin{proof}
Consider the sequence $\wh{t}_n$ which satisfies the properties of \Cref{lemma:general-estimate-time}.
Consider the collection of the learners $\mathbb A_{i,n} := \wh{y}_{\wh{t}_n}^i : \calX \to [k]$ for any $1 \leq i \leq \lfloor \frac{n}{2 \wh{t}_n} \rfloor$ and $n \in \nats$. Let us fix a time $t \in T^\star$, where $T^\star$ is the set of good estimates of the critical time $t^\star$ (see \Cref{lemma:general-estimate-time}).
We have that
\[
\Pr \left[  \frac{1}{\lfloor \frac{n}{2t} \rfloor } \sum_{i=1}^{\lfloor \frac{n}{2t} \rfloor} \vec 1 \left\{ \Pr_{(x,y) \sim P}\left[  \wh{y}_t^i(x) \neq y \right] > 0 \right\} > \frac{7}{16} \right] \leq \exp \left( - \Big \lfloor n/(2t^\star) \Big \rfloor / 128 \right)\,,
\]
using Hoeffding's inequality. The above probability is over the sequence of all the training sets and essentially states that the ``bad'' event that the misclassification error is non-zero holds for the majority of the trained algorithms $\wh{y}_t^i$ with exponentially small probability. Conversely, except on an event of exponentially small probability, we have that $\Pr_{(x,y) \sim P}[\wh{y}_t^i(x) \neq y] = 0$ for the majority of $i$.
Recall that the above discussion holds for a chosen good $t$. We have to understand how well our learners $\mathbb A_{i,n}$ perform.
To this end, we have that
\begin{align*}
& \Pr \left[ \Pr_{(x,y) \sim P}[\mathbb A_{i,n}(x) \neq y] > 0  \text{ for the majority of $i \leq \lfloor n/(2 \wh{t}_n) \rfloor$} \right] \\
& \leq \Pr \left[ \wh{t}_n \notin T^\star \right] + \Pr \left[ \exists t \in T^\star : \Pr_{(x,y) \sim P}[\wh{y}_t^i(x) \neq y] > 0 \text{ for the majority of $i \leq \lfloor n/(2 \wh{t}_n) \rfloor$} \right]\\
& \leq C(P) \exp(-c(P) \cdot n) + t^\star \cdot \exp \left( - \Big \lfloor n/(2t^\star) \Big \rfloor / 128 \right)\,.
\end{align*}

This implies that the majority of our learners $\mathbb A_{i,n}$ will not incur  loss except on an event of exponentially small probability.
As a result, the majority vote of these classifiers is almost surely correct on a random sample from the distribution $P$ over $\calX \times [k]$. Hence, we have that 
\begin{align*}
\E \left[ \Pr_{(x,y) \sim P} \left[ \mathrm{Maj}((\mathbb A_{i,n}(x))_{i}) \neq y \right]  \right] 
& \leq \Pr \left[ \Pr_{(x,y) \sim P} \left[ \mathrm{Maj}((\mathbb A_{i,n}(x))_{i}) \neq y \right] > 0 \right]
\\
& \leq C(P) \exp(-c(P) \cdot n) + t^\star \cdot \exp \left( - \Big \lfloor n/(2t^\star) \Big \rfloor / 128 \right)\,.
\end{align*}
This concludes the proof.
\end{proof}

\subsection{Infinite Multiclass Littlestone Trees and Rates}
\label{section:linear-exp-gap}
We next show that if $\calH$ has an infinite mutliclass Littlestone tree, then there exists a significant drop
in the rate: any learning algorithm $\calH$ cannot be faster than linear. Our proof follows the approach in~\cite{universal}.
\begin{theorem}
\label{theorem:inf-littlestone-linear-lower-bound}
Assume that $\calH \subseteq \{0,1,...,k\}^{\calX}$ has an infinite multiclass Littlestone tree. Then, for any learning algorithm $\wh{h}_n$, there exists a realizable distribution $P$ over $\calX \times [k]$ such that $\E[\mathrm{er}(\wh{h}_n)] \geq \Omega(1/n)$ for infinitely many $n$. This means that $\calH$ is not learnable at rate faster than linear, i.e., $R(n) \geq 1/n$.
\end{theorem}

Let us provide some intuition. 
We will make use of the probabilistic method. We are going to define a distribution over probability distributions so that, with positive probability over this choice of the random object, any learning algorithm will have an expected error of order $\Omega(1/n)$. This positive probability implies that there exists such a distribution and hence the above result holds true. The key idea is that we are going to associate any random distribution $P_{\vec y}$ with a branch $\vec y$ of the infinite multiclass Littlestone tree. Given a finite number of samples from this distribution, only a finite part of this infinite path will be discovered; hence any algorithm after the ``revealed path'' must guess whether this random path goes left or right. This implies that the algorithm will err with probability $1/2$ when it observes a point that lies deeper in the branch than the training examples. The proof can be found at \Cref{proof:inf-littlestone-linear-lower-bound}.

The above result states that a class which does not have exponential rates cannot be learned faster than linearly. However, it is not clear if even linear rates are achievable. The next section deals with this case.

\subsection{Linear Learning Rates and Natarajan-Littlestone Trees}
\label{section:linear}
We introduce a novel combinatorial measure, the Natarajan-Littlestone (NL) tree, which essentially combines the structure of the Natarajan dimension and the Littlestone dimension.
\begin{definition}\label{definition:NL tree}
A Natarajan-Littlestone (NL) tree for $\calH \subseteq [k]^{\calX}$ of depth $d \leq \infty$ consists of a tree 
\[
\bigcup_{0 \leq \l < d} \{ x_u \in \calX^{\l+1}, u \in \{0,1\} \times \{0,1\}^2 \times ... \times \{0,1\}^{\l}\} 
\]
and two colorings $s^{(0)}, s^{(1)}$ mapping each position $u^i \in u$ for any node with pattern $u \in \{0,1\} \times ... \times \{0,1\}^{\l}$ for $i \in \{0,1,...,\l\}$ and $\l \in \{0,1,...,d-1\}$ of the tree to some color $\{0,1,...,k\}$
such that for every finite level $n < d$, the subtree $T_n = \cup_{0 \leq \l \leq n} \{ x_u = (x_u^0,...,x_u^\l) : u \in \{0,1\} \times \{0,1\}^2 \times ... \times \{0,1\}^{\l} \}$ satisfies the following: 
\begin{enumerate}
    \item At any point $x_u^i \in x_u \in T_n$, it holds $s^{(0)}(x_u^i) \neq s^{(1)}(x_u^i)$ and
    \item for any path $\vec y \in \{0,1\} \times ... \times \{0,1\}^{n+1}$, there exists a concept $h \in \calH$ so that
$h(x^i_{\vec y_{\leq \l}}) = s^{(0)}(x^i_{\vec y_{\leq \l}})$ if $y^i_{\l+1} = 1$ and $h(x^i_{\vec y_{\leq \l}}) = s^{(1)}(x^i_{\vec y_{\leq \l}})$ otherwise, for all $0 \leq i \leq \l$ and $0 \leq \l \leq n$, where
\[
\vec y_{\leq \l} = ( y_1^0, (y_2^0, y_2^1), ...,(y_\l^0,...,y_\l^{\l-1}) ), x_{\vec y_{\leq \l}} = (x^0_{\vec y_{\leq \l}},...,x^\l_{\vec y_{\leq \l}})\,.
\]
\end{enumerate}
We say that $\calH$ has an infinite NL tree if it has a NL tree of depth $d = \infty$. 
\end{definition}

We note that in the above definition we identify the color
$s^{(0)}(x^i_{\vec y \leq \l})$ with the (unique) position
of this point $x^i_{\vec y \leq \l}$ (since typically the coloring is over positions). For short, we will call the colorings $s^{(0)}, s^{(1)}$ used in the above definition \emph{everywhere different} and denote by $s^{(0)} \neq s^{(1)}$. As a sanity check, one can verify that if $\calH$ has an infinite NL tree, then it has an infinite multiclass Littlestone tree. In fact, one can construct the latter tree by choosing only one point from any node of the infinite NL tree. 
%The intuition behind the definition of this object is similar as in the exponential rates setting: even if there is no uniform bound on the Natarajan dimension of $\calH$ we want to measure ``how infinite'' this dimension can be.

The main result of this section is the next theorem.
\begin{theorem}
\label{theorem:linear-rates-algo}
Assume that $\calH \subseteq [k]^{\calX}$ does not have an infinite Natarajan-Littlestone tree.
Then, there exists an algorithm that learns $\calH$ at a linear rate.
\end{theorem}

\subsubsection{The Natarajan-Littlestone Game}
\label{section:nl-game}
In a similar manner as in the exponential rates setting, we need to come up with a function that is correct after a finite number of steps $n^*$ and this implies that some appropriate data-dependent complexity measure of $\calH$ is finite.
The first step towards establishing the desired result is to introduce the following two-player game between an adversary and a learning player. 
\begin{figure}[ht!]
    \centering
    \begin{mdframed}[style=MyFrame,nobreak=true]
\begin{center}
\begin{enumerate}
    \item The adversary picks $t$ points and two everywhere different colorings of these $t$ points $\xi_t = \left(\xi_t^{(0)}, ..., \xi_t^{(t-1)}, s_t^{(0)}, s_t^{(1)} \right)\in \calX^{t} \times [k]^t \times [k]^t$
    and reveals them to the learner.
    \item The learner chooses a pattern $\eta_t = (\eta_t^{(0)},...,\eta_t^{(t-1)}) \in  \left\{0, 1\right\}^t$.
    %$\eta_t \in [k]$.% \{y^{(1)}_t, y^{(2)}_t\}$.
\end{enumerate}
\end{center}
\end{mdframed}
    \caption{Adversarial Setting - 2-Player NL Game}
    \label{fig:gs-game2}
\end{figure}

Let $\Xi_t = \{ \xi_1^{(0)}, \xi_2^{(0)}, \xi_2^{(1)},...,\xi_t^{(0)},\ldots,\xi_t^{(t-1)}, s_1^{(0)}, s_1^{(1)}, \ldots, s_t^{(0)}, s_t^{(1)}\}$ be the sets of all points and colorings chosen by the adversary after $t$ rounds, for some finite integer $t$.
The learning player wins the game in round $t$ if the class $\calH_t = \calH_{\xi_1, \eta_1,...,\xi_t,\eta_t} = \emptyset$, where $\calH_0 = \calH$ and $\calH_t = \calH_{\xi_1, \eta_1,...,\xi_t,\eta_t}$ is the set
\[
\left \{ 
h \in \calH  
%\exists ~\textnormal{colorings}~ s_1 \neq s_2: \Xi_t \to [k] 
~\textnormal{so that}
\begin{array}{ll}
      h(\xi_z^{(i)}) = s_z^{(0)}(\xi_z^{(i)})~\textnormal{if}~\eta_z^{(i)} = 0 \\
      h(\xi_z^{(i)}) = s_z^{(1)}(\xi_z^{(i)})~\textnormal{if}~\eta_z^{(i)} = 1 \\
\end{array} 
~\textnormal{for}~0\leq i < z, z \in [1..t]
\right \}\,.
\]
%Note that if two everywhere different colorings $s_1,s_2$ are valid for the class $\calH_t$, in the sense that there exists some hypothesis that realizes all the patterns, then these colorings are also valid for any $\calH_{t'}$ with $t' < t$. Intuitively, the collection of valid everywhere different colorings is getting smaller as the length of the patterns increases. 
Note that the collection of valid functions decreases as the game proceeds, i.e., $\calH_t \subseteq \calH_{t-1}.$ The next lemma guarantees the existence of a winning strategy for the learner in the game. For the proof, see \Cref{proof:learner-winning}.

\begin{lemma}
\label{lemma:learner-winning}
If $\calH \subseteq [k]^\calX$ has no infinite NL tree, then there is a universally measurable winning strategy
for the learning player in the game $\calG.$
\end{lemma}

\subsubsection{Pattern Avoidance Algorithm}
\label{section:pattern-avoid}
Our learning algorithm for the linear rates case will be built over the notion of data patterns and pattern avoidance functions. In the exponential rates setting, we focused on the case where the learner is successful if she makes a finite number of mistakes. In the case of linear rates, this is replaced with controlling the \emph{model complexity}, i.e., we have to understand how expressive the class $\calH$ is.
We will say that a sequence $x_1,y_1,x_2,y_2,...$ in $(\calX \times [k])^{\infty}$ is \emph{consistent} with $\calH \subseteq [k]^\calX$ if, for every finite $t < \infty$, there exists $h \in \calH$ such that $h(x_i) = y_i$ for all $i \leq t$.

We next introduce the crucial notion of pattern. Intuitively, the expressivity of a hypothesis class $\calH$ is proportional to the length of the ``patterns'' realized by concepts in $\calH$. Hence, by controlling the length of the realizable patterns, we can obtain learners for $\calH$. 

%\grigorisnote{Maybe it's better to add above pink text at some point in the beginning of this section where we introduce the NL tree. I think it is important to have a small discussion about the multiclass PAC setting to make the paper self-contained.}

% \grigorisnote{!! I think we should not use the old pattern definition. The pattern avoidance function we create takes as input a tuple of x-points and two colorings and spits out a violated pattern. For every fixed x-tuple, if we change the coloring the pattern might be different. I think this might be fine though. Essentially, if we want the Natarajan dimension to be less than $d$ it suffices to show that for every coloring $s_1, s_2$ there is a pattern that is not observed for this specific coloring. My understanding is that we do not need to provide a pattern that does not work for any coloring. !!

% \begin{definition}
% \label{def:pattern}
% Given a sequence $S = (x_1,y_2,x_2,y_2,...) \in (\calX \times [k])^\infty$ that is consistent with $\calH$, a binary string $b_t \in \{0,1\}^t$ of length $t$ and two everywhere different colorings $s^{(0)}_t, s^{(1)}_t \subseteq [k]^t$, the tuple $(b_t, s^{(0)}_t, s^{(1)}_t)$ is a Natarajan-Littlestone pattern (or simply a pattern) if there exists a subsequence of consecutive terms (maybe we can even consider non-consecutive terms -- it should be the same) $S' = (x_{z+i},y_{z+i})_{i \in [t]}$ of length $t$ for some $z \in \nats$ so that $s_t^{(0)}(x_{z+i}) = y_{z+i}$ if $b_t = 0$ and $s_t^{(1)}(x_{z+i}) = y_{z+i}$ if $b_t = 1$. 
% \end{definition}
% }

\begin{definition}
\label{def:pattern}
Given a sequence $S = (x_1,y_2,x_2,y_2,...) \in (\calX \times [k])^\infty$ that is consistent with $\calH$, the binary string $b_t \in \{0,1\}^t$ of length $t$ is a Natarajan-Littlestone pattern (or simply a pattern) if there exists a subsequence of consecutive terms $S' = (x_{z+i},y_{z+i})_{i \in [t]}$ of length $t$ for some $z \in \nats$ so that there exist two $k$-colorings everywhere different $s^{(0)},s^{(1)}$ of the elements $(x_{z+i})_{i \in [t]}$ with $s^{(0)}(x_{z+i}) = y_{z+i}$ if $b_t = 0$ and $s^{(1)}(x_{z+i}) = y_{z+i}$ if $b_t = 1$. 
\end{definition}

We are going to use the winning strategy of the learning player of the game of \Cref{fig:gs-game2} in order to design an algorithm that \emph{avoids NL patterns in the data}. 
This is exactly the intuition behind the design of the Gale-Stewart game of \Cref{fig:gs-game2}. The learner predicts a binary string $\eta_t \in \{0,1\}^t$ with goal that this pattern is not realized by the class $\calH$, i.e., it is not a NL pattern with respect to a sequence consistent with $\calH$. Intuitively, the reason we are aiming for forbidden patterns is the following: before witnessing the realizable sequence we have no control over the complexity of the hypothesis class $\calH$, since its Natarajan dimension can be infinite. However, leveraging the fact that there is no infinite NL-tree we show that for every realizable sequence $S$ there is some (sequence-dependent) $n^*(S)$ such that the class $\calH_{x_1, y_1,\ldots,x_{n^*}, y_{n^*}}$ of the concepts that agree on the first $n^*$ terms has Natarajan dimension that is bounded by $O(n^{*})$. This is because if we can produce a forbidden pattern for every $x$-tuple of length $n^{*}$ we know that the Natarajan dimension of this class is bounded by $n^{*}$.

 The algorithm will get as input a data sequence consistent with $\calH$ and will identify patterns in the data. In particular, the algorithm works as follows: For any finite pattern length $t$, the algorithm traverses the data sequence in consecutive blocks of length $t$ and in each such block it tries all the everywhere different colorings for the points and uses the universally measurable winning strategy of the learning player to obtain a guess for a forbidden NL pattern for each of these colorings. The algorithm then checks if at least one of the pairs of colorings and forbidden patterns are actually realized by the data. If this is the case (i.e., the guess was false), the algorithm adds a new point $\xi_t$ in its list of ``bad points'' and continues searching for patterns of larger length. If the guess was correct and the pattern is not realized in this block, it continues with the same pattern size in the next block. 
% \grigorisnote{it might be a bit confusing to say that the guess is correct if the pattern is not realized -- if someone is looking at the algorithm we are saying that the guess is false if some equality holds}

The algorithm operates as follows. We remark that we can construct the strategy $\eta_t$ of the learner by invoking a notion of ordinal NL dimension in a similar manner as in the exponential rates case. 

\begin{figure}[ht!]
    \centering
    \begin{mdframed}[style=MyFrame,nobreak=true]
\begin{center}
\begin{enumerate}
    \item Initialize the pattern length $\l_0 = 1$.
    \item At every time step $t \geq 1:$
    \begin{enumerate}
        \item Set $L := \l_{t-1}$ and create a list $S_L$ of all the possible pairs of everywhere different colorings of $L$ elements.
        \item Set $j := 1$ be the index that traverses the list $S_L.$
        \item Set $s^{(0)} := S_L[j](0), s^{(1)} := S_L[j](1).$
        \item Compute pattern $b_{L} = \eta_{L}(\xi_1,...,\xi_{L-1}, x_{t-L+1},...,x_t, s^{(0)}, s^{(1)}) \in \{0,1\}^{L} \times [k]^L \times [k]^L.$
        \item If it holds that, $s^{(0)}_z(x_{t-L+1 +z}) = y_{t-L+1 +z},$ if $b_L(z) = 0$, and $s^{(1)}_z(x_{t-L+1 +z}) = y_{t-L+1 +z}$, otherwise, for $z \in \{0,1,...,L-1\}$, then
        \begin{enumerate}
            \item Set $\xi_{L} = (x_{t-L+1},...,x_t, s^{(0)}, s^{(1)})$  \texttt{//The pattern is realized (not forbidden).}
            \item Update $\l_t = L + 1$ \texttt{//Look for larger patterns.}
        \end{enumerate}
        \item Else:
        \begin{enumerate}
            \item If $j < L$ then update $j := j+1$ and go to step (c).
            \item Else shift the block, and move to the next timestep, i.e. $\ell_t = L, t := t+1$.
            
        \end{enumerate}
        
    \end{enumerate}
\end{enumerate}
\end{center}
\end{mdframed}
    \caption{Pattern Avoidance Algorithm}
    \label{fig:pattern-algo}
\end{figure}

Hence, the above procedure defines a pattern avoidance function 
\begin{equation}
\label{eq:pattern-strategy}
\wh{y}_{t-1}(x_{1},...,x_{\l_{t-1}}, s^{(0)}, s^{(1)}) := \eta_{\l_{t-1}}(\xi_1,...,\xi_{\l_{t-1}-1},x_{1},...,x_{\l_{t-1}}, s^{(0)}, s^{(1)}) \in \{0,1\}^{\l_{t-1}} \times [k]^{\ell_{t-1}}\times [k]^{\ell_{t-1}}\,.
\end{equation}
The intuition is that as long as the class $\calH$ does not have an infinite NL tree and a consistent sequence is provided to the algorithm, then whenever the ``if'' statement gets True, the game of \Cref{fig:gs-game2} progresses and the learner gets closer to winning the game. The finiteness of the NL tree gives the next key lemma.

\begin{lemma}
\label{lemma:pattern-avoidance}
For any sequence $x_1,y_1,x_2,y_2,...$ that is consistent with $\calH$, the pattern avoidance algorithm of \Cref{fig:pattern-algo}, in a finite number of steps, rules out NL patterns in the data sequence, in the sense that the ``if'' statement in Line (c) is false and $\l_t = \l_{t-1} < \infty, \wh{y}_t = \wh{y}_{t-1}$ for all sufficiently large $t$. Moreover, the mappings $\l_t$ and $\wh{y}_t$ are universally measurable.
\end{lemma}
\begin{proof}
Consider an infinite sequence of times $t = t_1,t_2,...$ and assume that the ``if'' condition is true for any time-step in this sequence. Since $\eta_t$ is a winning strategy for the learning player in the Gale-Stewart game $\calG$, we have that there exists some finite time $t^\star$ so that $\calH_{\xi_1,\eta_1,...,\xi_{t^\star},\eta_{t^\star}}= \emptyset.$ By the structure of the ``if'' condition, we have that $\xi_i = (x_{t_i - \l_{t_i}-1}+1, ..., x_{t_i})$ and $\eta_i = (y_{t_i - \l_{t_i}-1}+1, ..., y_{t_i})$. But since the sequence is consistent with $\calH$, the class $\calH_{\xi_1,\eta_1,...,\xi_{t^\star},\eta_{t^\star}}$ must contain the hypothesis that makes it consistent which is a contradiction since the class is empty. For the measurability of the strategies, we refer to Remark 5.4 of \cite{universal}.
\end{proof}

\subsubsection{Asymptotic Behavior of the Algorithm}
\label{section:pat-avoid-prob}
As in the case of exponential rates, our first goal towards the design of our algorithms is an asymptotic result. We first have to move from the adversarial to a probabilistic setting. Fix a realizable distribution $P$ over $\calX \times [k]$ and let $(X_1,Y_1),(X_2,Y_2),...$ be i.i.d. samples from $P$. As it is clear in this section for the upper bounds, we assume that $\calH \subseteq [k]^\calX$ does not have an infinite NL tree.

For any integer $\l \in \nats$ and any universally measurable pattern avoidance function $g : \calX^\l \times [k]^{\ell} \times [k]^{\ell} \to \{0,1\}^\l$, we consider the error
\[
\mathrm{per}(g) := \mathrm{per}_{\l}(g) = \Pr_{(x_1,y_1,...,x_\l,y_\l)}[g \textnormal{ fails to avoid some NL pattern realized by the data } x_1,y_1,...,x_\l,y_\l]\,,
\]
i.e., there exist some colorings $s^{(0)},s^{(1)}$ everywhere different that witness the pattern $g(x_1,...,x_\l,s^{(0)}, s^{(1)})$ with colors $(y_1,...,y_\l)$. The following lemma establishes this result. It follows the approach in~\cite{universal}.

\begin{lemma}
\label{lemma:asymptotic-linear}
For the algorithm of \Cref{eq:pattern-strategy}, it holds that
\[
\Pr[\mathrm{per}(\wh{y}_t) > 0] \to 0 \textnormal{ as } t \to \infty\,,
\]
where the probability is over the random data that are used to train the algorithm.
\end{lemma}
\begin{proof}
As in the proof of the asymptotic result in the exponential rates, since the distribution $P$ is realizable, we can get that the data sequence $X_1,Y_1,X_2,Y_2,...$ is consistent with $\calH$ almost surely. We can now employ \Cref{lemma:pattern-avoidance} and get that the random variable
\[
T = \sup \{ s \geq 1 : \textnormal{the Pattern Avoidance Algorithm of \Cref{fig:pattern-algo} gets in the ``if'' statement} \}
\]
is finite almost surely and $\wh{y}_s = \wh{y}_t$ and $\l_s = \l_t$ for all $s \geq t \geq T$. We can apply the law of large nymber for $m$-dependent sequences and get that the quantity $\Pr[ \mathrm{per}_{\l_t}(\wh{y}_t) = 0]
$ is equal to
\[
\Pr \left[ \lim_{S \to \infty} \frac{1}{S} \sum_{s = t+1}^{t+S} 1\{ \wh{y}_t \textnormal{ fails to avoid some NL pattern realized by } (X_s,Y_s,...,X_{s+\l_t-1},Y_{s+\l_t-1})\} = 0 \right]\,.
\]
We can decrease the probability of the right-hand side by taking the intersection of the right-hand side event with the event $\{T \leq t\}$.
Hence,  we have that
\[
\Pr[ \mathrm{per}_{\l_t}(\wh{y}_t) = 0]
\geq \Pr[T \leq t] \to 1, \textnormal{ as } t \to \infty\,,
\]
since $T$ is finite with probability one.
\end{proof}

\subsubsection{Linear Learning Rates}
\label{section:linear-rates}
Given some pattern avoidance function that is correct on \emph{any} tuple of size $t$, we can essentially use the 1-inclusion graph in a similar manner as if the Natarajan dimension of $\calH$ was bounded by $t$. This is established in the following lemma.

\begin{lemma}
[Learning NL pattern classes]
\label{lemma:learn-nl}
Fix $t \geq 1.$ Let $g : \calX^t \times [k]^t \times [k]^t \to \{0,1\}^t$ be the universally measurable NL pattern avoidance function of \Cref{fig:pattern-algo}. For any $n \geq 1,$ there exists a universally measurable classifier $\wh{Y}^g_n : (\calX \times [k])^{n-1} \times \calX \to [k]$ such that for every training set $(x_1,y_1,...,x_n,y_n) \in (\calX \times [k])^n$
where $g(x_{i_1},...,x_{i_t},s^{(0)},s^{(1)})$ is an avoiding NL pattern for all pairwise distinct indices $1 \leq i_1,...,i_t \leq n$ and everywhere different colorings $s^{(0)}, s^{(1)}$, the classifier achieves a linear permutation bound, i.e.,
\[
\Pr_{\sigma \sim \calU(\mathbb S_n)}
\left[
\wh{Y}^g_n(x_{\sigma(1)}, y_{\sigma(1)},...,x_{\sigma(n-1)},y_{\sigma(n-1)},x_{\sigma(n)})
\neq y_{\sigma(n)}
\right] \leq \frac{t \log(k)}{n}\,.
\]
\end{lemma}
Note that the above algorithm is transductive. This hints the use of the one-inclusion learning algortihm.

\begin{proof}
Fix $n \geq 1$ and set $X = [n]$. Also let $F$ be a set of colorings $f : X \to [k]$. We can apply the one-inclusion hypergraph algorithm (see \Cref{lemma:one-inc}) to get that there exists an algorithm $\mathbb A$ such that
\[
\Pr_{\sigma \sim \calU(\mathbb S_n)}[\mathbb A(F, \sigma(1), f(\sigma(1)),...,\sigma(n-1), f(\sigma(n-1)), \sigma(n)) \neq f(\sigma(n))] \leq \frac{\mathrm{Ndim}(F) \log(k)}{n}\,,
\]
for any coloring $f \in F$ and $F \in 2^{[k]^X}$. By construction of the algorithm $\mathbb A$, the output of the mapping is preserved under relabeling of $X$, i.e.,
\[
\mathbb A(F, \sigma(1),y_1,...,\sigma(n-1),y_{n-1}, \sigma(n))
=
\mathbb A(F \circ \sigma, 1,y_1,..., n-1,y_{n-1}, n)\,,
\]
where $F \circ \sigma = \{f \circ \sigma : f \in F\}$. Moreover, the mapping $\mathbb A$ is measurable since its domain is finite (since the number of colors $k$ is a constant).

We will make use of the next result.
\begin{claim}
\label{claim:colors-pattern}
Consider the pattern avoidance mapping $g : \calX^t \times [k]^t \times [k]^t \to \{0,1\}^t$ of \Cref{fig:pattern-algo}. Consider the class $F$ of $k$-colorings of $[n]$ and fix a sequence $\vec x = (x_1,...,x_n) \in \calX^n$. Fix $0 < t \leq n$.
Define the subset $F_{\vec x}$ that contains the colorings $f : [n] \to [k]$ that satisfy the following property: For all subsets $(i_1,...,i_t)$ of $[n]$ of size $t$ and for all colorings $s^{(0)},s^{(1)} : [n] \to [k]$ which are everywhere different, at least one of the next $t$ conditions is violated:
\[
f(i_j) = s^{(0)}(x_{i_j}) \textnormal{ if } g(x_{i_1},\ldots,x_{i_t}, s^{(0)}, s^{(1)})[x_{i_j}] = 0\,,
\]
\[
f(i_j) = s^{(1)}(x_{i_j}) \textnormal{ if } g(x_{i_1},\ldots,x_{i_t}, s^{(0)}, s^{(1)})[x_{i_j}] = 1\,.
\]
Then, we have that $\mathrm{Ndim}(F_{\vec x}) < t$.
\end{claim}
\begin{proof}
It must be the case that the Natarajan dimension cannot be more than $t-1$ since we cannot $N$-shatter any $t$-subset of $[n]$.
\end{proof}

Given any input sequence $(x_1,y_1,...,x_n,y_n)$, we introduce the concept class $F_x$ as in \Cref{claim:colors-pattern}. Moreover, consider the mapping $g : \calX^t \times [k]^t \times [k]^t\to \{0,1\}^t$ defined to be the pattern avoidance function generated by \Cref{fig:pattern-algo}. We can introduce a data-dependent classifier
\[
\wh{Y}_n(g; x_1,y_1,...,x_{n-1},y_{n-1},x_n) = \mathbb A(F_x,1,y_1,...,n-1,y_{n-1},n)\,.
\]
Note that the above mapping is universally measurable due to \Cref{lemma:pattern-avoidance}. By relabeling we have that
\[
\wh{Y}_n(g; x_{\sigma(1)},y_{\sigma(1)},...,x_{\sigma(n-1)},y_{\sigma(n-1)},x_{\sigma(n)}) = \mathbb A(F_x, \sigma(1), y_{\sigma(1)},...,\sigma(n-1),y_{\sigma(n-1)}, \sigma(n))\,.
\]
By the assumption of the lemma about the input $(x_1,y_1,...,x_n,y_n)$, we have that the coloring $y(i) = y_i \in [k]$ satisfies $y \in F_x$ (due to the construction of \Cref{claim:colors-pattern}). Hence, for such a sequence, we get that
\[
\Pr_{\sigma \sim \calU(\mathbb S_n)}
\left[
\wh{Y}^g_n(x_{\sigma(1)}, y_{\sigma(1)},...,x_{\sigma(n-1)},y_{\sigma(n-1)},x_{\sigma(n)})
\neq y_{\sigma(n)}
\right] \leq \frac{\mathrm{Ndim}(F_x) \log(k)}{n}\,.
\]
However, due to the conclusion of \Cref{claim:colors-pattern}, we have that the above rate is of order $\log(k) \cdot t/n$, since by construction $\mathrm{Ndim}(F_x) < t$.
\end{proof}

To make use of the results we just stated, we need to come up with a pattern avoidance function that is correct on \emph{any} tuple of size $t$. However, we can only establish that our pattern avoidance function is \emph{eventually} correct, so there is no bound on $t$ that is given to the learner. To deal with this problem, we keep track of some $\wh{t}_n$ such that, with high probability, the error of the pattern avoidance function is small. Then, we divide our sample into $n/\wh{t}_n$ batches of size $\wh{t}_n$. For each batch $i$ we generate some pattern avoidance function $g^i$, whose error probability is small. To achieve the linear rates we run the algorithm from \Cref{lemma:learn-nl} for every different batch $i$ and function $g^i$ to get a classifier $\wh{h}^i$. We output the majority vote of the classifiers $\wh{h}^i$.

\begin{lemma}
\label{lemma:time-linear}
For any $n \in \nats$,
Consider a training set $\{(X_i,Y_i)\}$ consisting of $n$ points i.i.d. drawn from $P$. Then there exists a universally measurable $\wh{t}_n = \wh{t}_n(X_1,Y_1,...,X_{\lfloor n/2\rfloor},Y_{\lfloor n/2 \rfloor})$ whose definition does not depend on $P$ so that the following holds. Set the critical time $t^\star \in \nats$ be such that 
\[
\Pr[ \mathrm{per}(\wh{y}_{t^{\star}}) > 0] \leq 1/8\,,
\]
where the probability is over the training set of the algorithm $\wh{y}_t$. Then, there exist $C,c > 0$ that depend on $P, t^\star$ but not $n$ so that
\[
\Pr[ \wh{t}_n \in T^\star ] \geq 1 - C e^{-cn}\,.
\]
where the probability is over the training of the estimator $\wh{t}_n$ and $T^\star$ is the set 
\[
T^\star = \left\{ 1 \leq t \leq t^\star : \Pr[ \mathrm{per}(\wh{y}_{t^{\star}}) > 0] \leq 3/8 \right\}\,,
\]
where the probability is over the training of $\wh{y}_t$.
\end{lemma}

For the proof of the above lemma, see \Cref{proof:time-linear}. We continue with the main result of this section.

\begin{theorem}
[Linear Rates]
Assume that class $\calH \subseteq [k]^\calX$ does not have an infinite Natarajan-Littlestone tree. Then, $\calH$ admits a learning algorithm that achieves an optimal linear rate. 
\end{theorem}

\begin{proof}
The learning algorithm works as follows: It first computes the estimate $\wh{t}_n$ introduced in \Cref{lemma:time-linear}. Then it splits the data into two halves: the first half is used to compute the pattern avoidance functions $g^i := \wh{y}^i_{\wh{t}_n}$ for $1 \leq i \leq \lfloor n/(4 \wh{t}_n ) \rfloor$ and the second half is used in order to apply \Cref{lemma:learn-nl} and get classifiers $\wh{y}^i$ with 
$\wh{y}^i(x) := \wh{Y}^{g^i}_{\lfloor n/2\rfloor + 2}(X_{\lceil n/2 \rceil}, Y_{\lceil n/2 \rceil}, ..., X_n, Y_n, x)$. Finally, the algorithm outputs
\[
\wh{h}_n = \mathrm{Maj} \left (\wh{y}^1, \wh{y}^2, ..., \wh{y}^{\lfloor n/(4\wh{t}_n) \rfloor } \right)\,.
\]
and aims to get $\E[\Pr_{(X,Y)}[\wh{h}_n(X) \neq Y]] \leq C/n$ for some constant $C$, where the expectation is over the training set used for the predescribed steps.
Set $\wh{N} = \lfloor n/(4\wh{t}_n) \rfloor$ and $N^\star = \lfloor n/(4 t^\star) \rfloor$.
With the notation of \Cref{lemma:time-linear}, we get
\[
\Pr \left[ \frac{1}{\wh{N}} \sum_{i \in [\wh{N}]} \vec 1\{ \mathrm{per}(\wh{y}^i_{\wh{t}_n}) > 0\} > \frac{1}{100k}, \wh{t}_n \in T^\star \right]
\leq t^\star \exp(-C_1 \cdot N^\star)\,,
\]
i.e., the strict majority of the pattern avoidance functions have zero error with high probability where $C_1$ is a fixed constant depending on the (uniformly bounded) number of labels $k \in \nats$. Using \Cref{lemma:time-linear}, we get
\begin{align*}
\E[ \Pr[\wh{h}_n(X) \neq Y] ]
 & \leq \Pr\left[ \mathrm{Maj} \left (\wh{y}^1(X), \wh{y}^2(X), ..., \wh{y}^{\lfloor n/(4\wh{t}_n)\rfloor } (X) \right) \neq Y\right] \\
 & \leq C\exp(-cn) + t^\star \exp(-C_1 \cdot n^\star) + p\,,
\end{align*}
where 
\[
p = \Pr\left[\wh{t}_n \in T^\star, \sum_{i \in [\wh{N}]} \vec 1\{ \mathrm{per}(\wh{y}^i_{\wh{t}_n}) = 0\} \geq (100k-1)\wh{N}/(100k), \mathrm{Maj} \left (\wh{y}^1(X), \wh{y}^2(X), ..., \wh{y}^{\lfloor n/(4\wh{t}_n)\rfloor } (X) \right) \neq Y\right]\,.
\]
We have that
\[
\Pr\left[\mathrm{Maj} \left (\wh{y}^1(X), \wh{y}^2(X), ..., \wh{y}^{\lfloor n/(4\wh{t}_n)\rfloor } (X) \right) \neq Y\right]
\leq 
\Pr\left[ \sum_{i \in [\wh{N}]} \vec 1\{ \wh{y}^i(X) \neq Y\} \geq \wh{N}/k\right]\,,
\]
since whenever the majority makes a mistake, the right-hand side event occurs. 
Since $k$ is a fixed constant, any two sets containing at least $1/k$ and $1 - 1/(100k)$ fractions of $\{1,2,...,\lceil n/\wh{t}_n \rceil \}$, must have at least a $\Theta(1)$ fraction in their intersection. 
%\grigorisnote{why is this true? for example if $k = 10$ for example this doesn't hold. Should we change the number 16 and make it a function of $k$ so that the intersection is $\Theta(1)$?}. 
Hence, we get that
\[
p \leq \Pr\left[\wh{t}_n \in T^\star, \frac{1}{\wh{N}}\sum_{i \in [\wh{N}]} \vec 1\{ \wh{y}^i(X) \neq Y \} \cdot \vec 1\{ \mathrm{per}(\wh{y}^i_{\wh{t}_n}) = 0\} \geq \Theta(1)\right] 
\]
and using Markov's inequality
\[
p \leq \Theta(1) \cdot \E\left[\frac{1}{\wh{N}}\sum_{i \in [\wh{N}]} \vec 1\{ \wh{t}_n \in T^\star\} \cdot \vec 1\{ \wh{y}^i(X) \neq Y \} \cdot \vec 1\{ \mathrm{per}(\wh{y}^i_{\wh{t}_n}) = 0\}\right] 
\]
We can now apply \Cref{lemma:learn-nl} conditionally on the first half of the data and get
\[
\E[\Pr[\wh{h}_n(X) \neq Y]]
\leq 
C\exp(-cn) + t^\star \exp(-n^\star /128)
+ \Theta(1) \cdot \log(k) \cdot  \E \left[ \vec 1\{\wh{t}_n \in T^\star\} \cdot \frac{1}{\wh{N}} \sum_{i \in [\wh{N}]} \frac{\l^i_{\wh{t}_n}}{\lfloor n/2 \rfloor + 2} \right]\,.
\]
Now we have that $\Theta(1) \log(k) = \Theta(1)$ and $\l^i_{\wh{t}_n} \leq \wh{t}_n+1 \leq t^\star + 1$, since $\wh{t}_n \in T^\star$. This gives
\[
\E[\Pr[\wh{h}_n(X) \neq Y]]
\leq 
C/n\,.
\]
\end{proof}

% In order to show the above result, we need the next tool.
% \begin{proposition}
% [Multiclass mistake bound \href{http://citeseerx.ist.psu.edu/viewdoc/download?doi=10.1.1.597.3425&rep=rep1&type=pdf}{Paper}]
% Consider any $k,n \in \nats$ and $\calF \subseteq [k]^{\calX}$ with $\min\{\Psi_P(\calF), \Psi_G(\calF) \} \leq \infty$. The multiclass one-inclusion prediction strategy satisfies
% \[
% \Pr_{\sigma \sim \calU}[\mathbb A(x_{\sigma(1),y_{\sigma(1)}},...,x_{\sigma(n-1)},y_{\sigma(n-1)}, x_{\sigma(n-1)}) \neq y_{\sigma(n)} ] \leq \min\{\Psi_P(\calF), \Psi_G(\calF) \}/n\,.
% \]
% \end{proposition}
% \begin{proof}
% Since the density of the one-inclusion hypergraph is at most $\min \{ \}$, we get that the permutation mistake bound is
% \[
% ...
% \]
% since ...
% \end{proof}

\subsection{Arbitrarily Slow Rates}
\label{section:slow}
The last step that is needed to establish the characterization of learnability in the multiclass setting is to show that if $\calH$ has an infinite Natarajan Littlestone tree it is learnable at an arbitrarily slow rate. The following theorem establishes that. The idea of the proof is similar as in the setting with the infinite multiclass Littlestone tree. Intuitively, the reason that in this regime we get arbitrarily slow rates whereas in the other one we get linear rates is that the branching factor now is \emph{exponential} in the depth of the tree, whereas before it remained constant.

\begin{theorem}
\label{theorem:slow-rates}
Assume that $\calH \subseteq [k]^{\calX}$ has an infinite Natarajan-Littlestone tree.
Then, any algorithm that learns $\calH$ requires arbitrarily slow rates. Moreover, there exists an algorithm that learns $\calH$ at an arbitrarily slow rate.
\end{theorem}
\begin{proof}
The desired algorithm that learns $\calH$ is provided by \cite{hanneke-kontorovich}.
For the proof of the lower bound, fix a vanishing rate $R(t) \to 0$, fix any learning algorithm $\wh{h}_n$ for $\calH$ and let $\{ x_u \}$ be an infinite NL tree for $\calH$. Consider a random branch of the tree $y = (y_1,y_2,...)$ where the pattern $y_\l = (y_\l^0,...,y_\l^{\l-1}) \in \{0,1\}^\l$ is chosen uniformly at random from the $\l$-dimensional Boolean hypercube for any $\l \in \nats$. Fix a finite level $n \in \nats$. By the structure of the NL tree, we know that there exist two different colorings $s_1, s_2$ and a hypothesis $h \in \calH$ so that $h$ agrees either
with $s_1$ if the pattern bit says 1 or with $s_2$ otherwise.
Our first goal in this universal lower bound is to construct a realizable distribution. We define the random distribution that assigns non-zero mass to the points of the branch $y$ with labels consistent with $h$, i.e.,
\[
P_y( (x^i_{y \leq \l-1}, h(x^i_{y \leq \l-1})) ) = \frac{p_\l}{\l} \textnormal{ for }~ 0 \leq i < \l, \l \in \nats\,,
\]
where $p_\l$ is a sequence of probabilities so that $\sum_{\l \in \nats} p_\l = 1$ that we will select later in the proof.
Intuitively, the distribution $P_y$ chooses the node of the infinite branch at level $\l$ with probability $p_\l$; this node contains $\l$ points of $\calX$ and one of them is chosen uniformly at random.

By the structure of the infinite NL tree, we get that such a labeling $h \in \calH$ exists for any level $n \in \nats$ and this labeling is consistent with all the previous levels $1 \leq n' < n$. Thus, we get that
\[
\Pr_{(x,z) \sim P_y}[h(x) \neq z] \leq \sum_{\l > n} p_\l\,.
\]
Hence, as $n \to \infty$, we get that $P_y$ is realizable for any realization of the random branch $y$. Moreover, the mapping $y \to P_y$ is measurable.

We now have to lower bound the loss of the potential learner $\wh{h}_n$ using the probability measure $P_y$. Let $(X,Z), (X_1,Z_1), (X_2, Z_2), ... \in (\calX \times [k])^\infty$ be a collection of i.i.d. samples from $P_y$. Equivalently, we can write
\begin{enumerate}
    \item $X = x^I_{y \leq T-1}$ and $Z = z^I_T$ for two random variables $(T,I)$ with joint distribution $\Pr[T = \l, I = i] = \frac{p_\l}{\l}$ for $0 \leq i < \l, \l \in \nats$.
    \item For $j \in \nats$, set $X_j = x^{I_j}_{y \leq T_j-1}$ and $Z_j = z^{I_j}_{T_j}$ for two random variables $(T_j,I_j)$ with joint distribution $\Pr[T_j = \l, I_j = i] = \frac{p_\l}{\l}$ for $0 \leq i < \l, \l \in \nats$.
\end{enumerate}
We underline that in the above the random variables $(T,I), (T_1,I_1), (T_2,I_2)$ are i.i.d. and independent of the random branch $y$. Our goal is to lower bound the error of any learning algorithm: For all $n$ and $\l$,
\[
\Pr[ \wh{h}_n(X) \neq Z, T=\l]
\geq \sum_{i=0}^{\l-1} \Pr[\wh{h}_n(x^i_{y\leq \l-1}) \neq z_\l^i, T=\l,I=i,T_1,...,T_n \leq \l, (T_1,I_1),..., (T_n,I_n) \neq (\l,i)]\,,
\]
where in the right hand side the probability is decreased by additionally requiring that the whole training set is concentrated before the level $\l+1$ and the testing example is not contained in the training set. Consider this event $E_{n,\l,i} = \{ T=\l, I=i, T_1, ..., T_n \leq \l, (T_1,I_1), ..., (T_n,I_n) \leq (\l,i) \}$. If we condition on $E_{n,\l,i}$, we have that the prediction of $\wh{h}_n(X) = \wh{h}_n(x^i_{y\leq \l-1})$ is independent of the label $z_\l^i$. Hence, we have that
\[
\Pr[ \wh{h}_n(X) \neq Z, T=\l]
\geq \sum_{i=0}^{\l-1} \Pr[\wh{h}_n(x^i_{y\leq \l-1}) \neq z_\l^i | E_{n,\l,i}] \Pr[E_{n,\l,i}] \geq \frac{1}{2}\sum_{i = 0}^{\l-1} \Pr[E_{n,\l,i}]\,.
\]
By the choice of the randomness over $T,I,T_1,I_1,...$, we get that
\[
\Pr[ \wh{h}_n(X) \neq Z, T=\l]
\geq \frac{p_\l}{2} \left(1 - \sum_{m > \l} p_m - \frac{p_\l}{\l} \right)^n \,.
\]
To conclude the proof we have to choose the sequence of probabilities $(p_\l)$ and relate it to the vanishing rate $R$. By combining Lemma 5.12 of \cite{universal} and by applying the reverse Fatou's lemma (see e.g., the end of the proof of Theorem 5.11 of \cite{universal}), the proof is concluded.
\end{proof}

\subsection{A Sufficient Condition for Linear Rates for Multiclass Learning}
\label{section:gl-trees}
Another approach to come up with an algorithm that works in the multiclass setting is to use the algorithm that was developed in~\cite{universal} for the binary setting. Towards this end, we define a slightly different combinatorial measure for a class $\calH$.

\begin{definition}
\label{definition:gl-tree}
A Graph-Littlestone (GL) tree for $\calH \subseteq [k]^{\calX}$ of depth $d \leq \infty$ consists of a tree 
\[
\bigcup_{0 \leq \l < d} \{ x_u \in \calX^{\l+1}, u \in \{0,1\} \times \{0,1\}^2 \times ... \times \{0,1\}^{\l}\} 
\]
and a coloring $s$ 
mapping each position $u^i \in u$ for any node with pattern $u \in \{0,1\} \times ... \times \{0,1\}^{\l}$ for $i \in \{0,1,...,\l\}$ and $\l \in \{0,1,...,d-1\}$ of the tree to some color $\{0,1,...,k\}$
such that for every finite level $n < d$, the subtree $T_n = \cup_{0 \leq \l \leq n} \{ x_u : u \in \{0,1\} \times \{0,1\}^2 \times ... \times \{0,1\}^{\l} \}$ satisfies the following: 
\begin{enumerate}
    \item For any path $\vec y \in \{0,1\} \times ... \times \{0,1\}^{n+1}$, there exists a concept $h \in \calH$ so that
$h(x^i_{\vec y_{\leq \l}}) = s(x^i_{\vec y_{\leq \l}})$ if $y^i_{\l+1} = 1$ and $h(x^i_{\vec y_{\leq \l}}) \neq s(x^i_{\vec y_{\leq \l}})$ otherwise, for all $0 \leq i \leq \l$ and $0 \leq \l \leq n$, where
\[
\vec y_{\leq \l} = ( y_1^0, (y_2^0, y_2^1), ...,(y_\l^0,...,y_\l^{\l-1}) ), x_{\vec y_{\leq \l}} = (x^0_{\vec y_{\leq \l}},...,x^\l_{\vec y_{\leq \l}})\,.
\]
\end{enumerate}
We say that $\calH$ has an infinite GL tree if it has a GL tree of depth $d = \infty$. 
\end{definition}
We note that in the above definition we identify the color
$s(x^i_{\vec y \leq \l})$ with the (unique) position
of this point $x^i_{\vec y \leq \l}$ (since typically the coloring is over positions). One can again verify that if $\calH$ has an infinite GL tree, then it has an infinite multiclass Littlestone tree. We show that if a class $\calH$ does not have an infinite GL tree, then it is possible to learn $\calH$ at a linear rate. The proof of this statement is via a reduction to the binary setting. In fact, we invoke the well-studied ``one versus all'' approach, where we are trying to learn $k$ different classes $\calH_i$ that distinguish between points that belong to the $i$-th class and points that belong to some class $j \neq i$. 

\begin{theorem}
Assume that $\calH \subseteq [k]^\calX$ does not have an infinite Graph-Littlestone tree. Then, there exists an algorithm that learns $\calH$ at a linear rate.
\end{theorem}
\begin{proof}
Consider $k$ binary classes $\calH_i$ induced by the hypothesis class $\calH$ where $\calH_i = \{ x \mapsto 1\{h(x) = i\}  : h \in \calH \}$, i.e., for $h \in \calH, \wt{h} \in \calH_i$ we have that $\wt{h}(x) = 1$ if and only if $h(x) = i$.
\begin{claim}
Assume that $\calH$ does not have an infinite GL tree. Then for any $i \in [k]$ the class $\calH_i$ is learnable at a linear rate.
\end{claim}
\begin{proof}
Fix $i \in [k]$. It suffices to show that the class $\calH_i$ does not have an infinite VCL tree (VCL trees are introduced in \cite{universal}  and, intuitively, a VCL tree is a GL tree with $k=1$ and $s = 1$ everywhere). Towards contradiction, assume that $\calH_i$ admits an infinite VCL tree $\calT = \{x_u\}$. We can construct the following Graph-Littlestone tree: we use the same nodes as in $\calT$ and for any node in the tree we use the coloring that colors each point with the color $i$. Fix an arbitrary level $n$ and a path $y$. We know that there exists a binary hypothesis $\wt{h} \in \calH_i$ that realizes this path in the VCL tree. Due to the construction of the class $\calH_i$, we have that there exists a mapping $h \in \calH$ (where $\wt{h} = 1\{h(\cdot) = i\})$ that realizes the path in the constructed tree. This property holds for any path and any level. Hence, we have constructed an infinite GL tree for $\calH$ which yields a contradiction.
\end{proof}
Assume that we get $k$ binary classifiers $\wh{h}_n^{(i)}$, one for each class $\calH_i$. Finally, we set $
\wh{h}_n(x) = \argmax_i \wh{h}_n^i(x)$ for any $x \in \calX$. The above claim concludes the proof since
\[
\E \Pr_{(x,y) \sim P}[\wh{h}_n(x) \neq y]
=
\E \Pr_{(x,y) \sim P}[\exists i \in [k] : \wh{h}_n^{(i)}(x) \neq \vec 1\{y=i\}]
\leq \sum_{i \in [k]} \E[\mathrm{err}(\wh{h}_n^{(i)})] \leq 1/n\,,
\]
since $k$ is a fixed constant.
\end{proof}

We close this section with an open question. Is the GL tree roughly speaking equivalent to the NL tree? This should remind the reader the connection between the Graph and the Natarajan dimension in the uniform multiclass PAC learning. In fact, it holds that $\mathrm{Ndim}(\calH) \leq \mathrm{Gdim}(\calH) \leq \mathrm{Ndim}(\calH) \cdot O(\log(k))$.

\begin{openquestion}
\label{open-question}
Let $\calH \subseteq [k]^\calX$ for some fixed constant $k \in \nats$. 
\begin{enumerate}
    \item Is it true that $\calH$ has an infinite NL tree if and only if it has an infinite GL tree?
    \item Is it possible to obtain an analogue of the inequality between the Graph and the Natarajan dimensions for the ordinal GL and NL dimensions?
\end{enumerate}
\end{openquestion}

One approach to tackle this problem is to show that if  $\calH$ has an infinite GL tree, then it is learnable at an arbitrarily slow rate. It is not clear to us that the current proof can be modified to work in this setting. If we try to follow the same steps we cannot guarantee the realizability of the sequence. If we pick an infinite random path and let $p_i$ be the path up to depth $i$, we know that for each such depth there exists some $h_i$ that realizes this path, i.e., if $b_i = 1 \implies h(x_i) = y^{h_i} = s(x_i)$ and if $b_i = 0 \implies h(x_i) = y^{h_i} \neq s(x_i)$. Now if consider some depth $j > i$ then the hypothesis $h_j$ will agree with $h_i$ on every $x_i$ with $b_i = 1$. However, we cannot guarantee that $h_j(x_i) = h_i(x_i)$ if $b_i = 0$. Hence, the straightforward way to modify the proof to account for that would be to consider a \emph{family} of different distributions. Note that this is not allowed in the universal learning setting, since we fix the distribution.
% We first prove the one direction.
% \begin{claim}
% If $\calH$ has an infinite NL tree, then it also has an infinite GL tree.
% \end{claim}
% \begin{proof}
% By construction.
% \end{proof}
% Next, we show the other direction.
% \begin{claim}
% If $\calH$ has an infinite GL tree, then it also has an infinite NL tree.
% \end{claim}
% \begin{proof}
% \alkissnote{This is my intuition. This is not formal.}
% Let us assume that the class $\calH$ does not have an infinite NL tree. We are going to prove that it also does not have an infinite GL tree. Since we do not have an infinite NL tree, for any sequence $\vec x \in \calX^\infty$, the tree constructed using BFS for the given sequence will be a NL tree of finite depth.
% Since the tree is of finite depth, there exists a path of finite length. Focusing on this path, there exists a sequence of sets of increasing size which are N-shattered by $\calH$. Since we have that $d_N \leq d_G \leq d_N \log(k)$, we know that the length of this path could increase only by a multiplicative constant if we tried to G-shatter it using $\calH$.
% \end{proof}
% \end{proof}

\section{Multiclass Learning for Partial Concepts}
In this section, we provide our results on multiclass learnability in the context of partial concept classes.
The missing proofs are presented in \Cref{appendix:missing proofs for partial concepts}.

\subsection{PAC Multiclass Learnability for Partial Concepts: The Proof of \Cref{theorem:partial-pac-learning-sc}}
We show the following theorem, which implies \Cref{theorem:partial-pac-learning}.
For reader's convenience, we restate \Cref{theorem:partial-pac-learning-sc}.
\begin{theorem*}
For any partial concept class $\calH \subseteq \{0,1,...,k,\star\}^\calX$ with $\mathrm{Ndim}(\calH) \leq \infty$, the sample complexity of PAC learning the class $\calH$ satisfies
\[
\calM(\eps, \delta) = O \left ( \frac{\mathrm{Ndim}(\calH) \log(k)}{\eps} \log(1/\delta) \right) \text{ and } \calM(\eps, \delta) = \Omega \left(\frac{\mathrm{Ndim}(\calH) + \log(1/\delta)}{\eps} \right)\,.
\]
In particular, if $\mathrm{Ndim}(\calH) = \infty,$ then $\calH$ is not PAC learnable.
\end{theorem*}

\begin{proof}
Our algorithm will make use of the one-inclusion hypergraph algorithm whose utility is provided by \Cref{lemma:one-inc} for total concepts. We first show the next lemma for the one-inclusion hypergraph predictor for partial concepts.
\begin{lemma}
\label{lemma:one-inc-partial}
Fix a positive constant $k$.
For any partial concept class $\calH \subseteq \{0,1,...,k,\star\}^\calX$ with $\mathrm{Ndim}(\calH) < \infty$, there exists an algorithm $\mathbb A : (\calX \times [k])^* \times \calX \to [k]$ such that, for any $n\in \nats$ and any sequence $\{(x_1,y_1),....,(x_n,y_n)\} \in (\calX \times [k])^n$ that is realizable with respect to $\calH$,
\[
\Pr_{\sigma \sim \calU(\mathbb S_n)}[\mathbb A( x_\sigma(1),  y_\sigma(1),...,x_\sigma(n-1), y_\sigma(n-1), x_\sigma(n)) \neq y_\sigma(n)] \leq \frac{\mathrm{Ndim}(\calH)  \log(k)}{n}\,.
\]
\end{lemma}
The proof is deferred to \Cref{appendix:missing proofs for partial concepts}.

Let us now focus on the upper bound given that $\mathrm{Ndim}(\calH) < \infty$. For any distribution $P$ realizable with respect to $\calH$ and for a sequence of $n$ labeled i.i.d. examples from $P$, we define the strategy
$\wh{h}_n(\cdot) = \mathbb A(X_1,Y_1,...,X_n, Y_n, \cdot)$ and so
\[
\E[\mathrm{er}_P(\wh{h}_n)]
=
\E_{(X_i, Y_i)_{i \leq n}}\left[\Pr_{(X_{n+1},Y_{n+1})}[\mathbb A(X_1,Y_1,...,X_n, Y_n, X_{n+1}) \neq Y_{n+1}] \right] \leq \frac{\mathrm{Ndim}(\calH) \Theta(\log(k))}{n+1}\,.
\]
We next have to convert this algorithm which guarantees an expected error bounded by $\mathrm{Ndim}(\calH) \Theta(\log(k))/(n+1)$ into an algorithm
that guarantees a bound on the error with probability at least $1-\delta$. In order to boost the algorithm, we use a standard boosting algorithm by decomposing the dataset into $\log(1/\delta)$ parts and using Chernoff bounds. For the details we refer to the boosting trick of \cite{haussler1994predicting} and the proof of Theorem 34.(i) of \cite{alon2022theory}.

Let $\mathrm{Ndim}(\calH) = \infty$. We will show that $\calH$ is not PAC learnable. For any $\l \leq \mathrm{Ndim}(\calH),$ let $\calX_\l = \{x_1,...,x_\l\}$ be a set $N$-shattered by $\calH$ using the function $f$. Let $\calH_\l$ be the class of all total functions $\calX_\l \to \{0,1,...,\l\}$, any distribution $P$ on $\calX_\l \times \{0,1,...,\l\}$ realizable with respect to $\calH_\l$ can be extended to a distribution on $\calX \times \{0,1,...,\l\}$ realizable with respect to $\calH$ with $P( (\calX \setminus \calX_k) \times \{0,1,...,\l\} ) = 0$. Thus, any lower bound on the sample complexity of PAC learning the total concept class $\calH_\l$ is also a lower bound on the sample complexity of learning the partial class $\calH$. This gives the desired lower bound. Also, the partial concept classes with infinite Natarajan dimension are not PAC learnable.

\end{proof}

% \begin{theorem}
% The following statements are equivalent for any partial concept class $\calH \subseteq \{0,1,\ldots,k,\star\}^{\calX}$.
% \begin{itemize}
%     \item $\mathrm{Ldim}_k(\calH) < \infty$.
%     \item $\calH$ is online learnable in the realizable setting.
%     %\item $\calH$ is online learnable in the agnostic setting.
% \end{itemize}
% \end{theorem}

\subsection{The ERM Principle and Multiclass Partial Concepts}

Our goal is to understand \emph{when the ERM principle succeeds in the partial setting}. To address this natural and important task, we pose the following question: What is the difference between learning partial concepts with output in $\{0,1,...,k,\star\}$ and learning total concepts with labels in $\{0,1,...,k,k+1\}$, where $k$ is a positive integer? Conceptually, the key difference between partial concepts with $k+1$ labels and ($k+2$)-label multiclass classification has to do with the support of the distribution: In the latter, the learning problem is a distribution over $\calX \times \{0,1,...,k,k+1\},$ while in the former we have a distribution only over $\calX \times \{0,1,...,k\}$ (recall \Cref{definition:partial-pac}). We address this question in \Cref{theorem:equivalence-total-partial}. We show that any partial concept class $\calH$ is learnable in the $(k+2)$-label setting if and only if $\calH$ is learnable in the partial setting and the VC dimension of the set family $\{ \mathrm{supp}(h) : h \in \calH \}$ is finite. This result is helpful since it guarantees that when the VC dimension of the above family is bounded, the ERM principle provably holds and can be applied in the partial setting.

%The  key difference between $\{0,1,...,k,\star\}$ (partial concepts with $k$ labels) and $\{0,1,...,k,k+1\}$ (($k+2$)-label multiclass classification) is that, in the latter, the learning problem is a distribution over $\calX \times \{0,1,...,k,k+1\}$ while in the former we have a distribution only over $\calX \times \{0,1,...,k\}$.
The next result gives a formal connection between the two settings.

\begin{proposition}
\label{theorem:equivalence-total-partial}
Any partial class $\calH \subseteq \{0,1,\ldots,k,\star\}^{\calX}$ is PAC learnable in the $(k+2)$-label multiclass setting if and only if
$\calH$ is PAC learnable in the partial concepts setting \textbf{and} $\mathrm{VCdim}(\{\mathrm{supp}(h) : h \in \calH\}) < \infty$.
\end{proposition}

%\begin{proposition*}The following statements are equivalent for any partial concept class $\calH \subseteq \{0,1,\ldots,k,\star\}^{\calX}$.\begin{itemize}
    %\item $\calH$ is PAC learnable in the $(k+2)$-label multiclass setting.
    %\item $\calH$ is PAC learnable in the partial concepts setting \textbf{and} $\mathrm{VCdim}(\{\mathrm{supp}(h) : h \in \calH\}) < \infty$.
%\end{itemize}
%\end{proposition*}
The following proof is an adaptation of Proposition 23 of \cite{alon2022theory} to the multiclass setting.
\begin{proof}
Let us first assume that $\calH$ is PAC learnable in the $(k+2)$-label multiclass setting. This implies that $\mathrm{Ndim}_{k+2}(\calH) \leq \mathrm{Gdim}_{k+2}(\calH) \leq O(\log(k+2)) \cdot \mathrm{Ndim}_{k+2}(\calH) < \infty$. This implies that $\calH$ is also learnable in the partial concepts setting by \Cref{theorem:partial-pac-learning} (by the definitions of the extended Graph and Natarajan dimensions; intuitively in the partial setting, we can use one less color). Now consider the collection of sets $\mathbb S = \{ \mathrm{supp}(h) : h \in \calH \} = \{ \{x \in \calX : h(x) \neq \star \} : h \in \calH \}.$ For any sequence of $d$ points $x_1,\ldots, x_d$
shattered by $\mathbb S$, let us take the hypothesis $f : \calX \to \{0,1,...,k,\star\}$ so that $f(x_i) = \star$ for any $i \in [d]$.
As an implication, this sequence $x_1,...,x_d$ is also shattered by
$\{ x \mapsto \Vec{1}\{h(x) = f(x)\} : h \in \calH \}$. By the definition of the graph dimension, we get that $\mathrm{Gdim}_{k+2}(\calH) := \sup_{f : \calX \to \{0,1,...,k,\star\}} \mathrm{VCdim}( x \mapsto 1\{h(x) = f(x)\} : h \in \calH ) \geq \mathrm{VCdim}(\mathbb S)$. Since $\calH$ is PAC learnable in the $(k+2)$-label multiclass setting, we get that
$\mathrm{VCdim}(\{\mathrm{supp}(h) : h \in \calH\}) < \infty$.

Let us now assume that $\calH$ is PAC learnable in the partial concepts setting and $\mathrm{VCdim}(\{\mathrm{supp}(h) : h \in \calH\}) < \infty$. We are going to show that the Natarajan dimension of $\calH$ is not infinite in the $(k+2)$-label multiclass setting. Fix a sequence $(x_i, y_i^{(0)}, y_i^{(1)})_{i \in [d]} \in (\calX \times \{0,1,...,k,\star\} \times \{0,1,...,k,\star\})^d$ for some $d \in \nats$, as in the definition of the $(k+2)$-label Natarajan dimension. For this sequence, consider the set $\{ x_i : \star \notin \{y_i^{(0)}, y_i^{(1)}\} \}$ with $y_i^{(0)} \neq y_i^{(1)}$. This set is shattered by the partial concept class $\calH$ by the extension of the Natarajan dimension to the partial concepts setting. Moreover, the set $\{ x_i : \star \in \{ y_i^{(0)}, y_i^{(1)} \} \}$
is shattered by the set $\mathbb S = \{ \mathrm{supp}(h) : h \in \calH \}$. This implies that $ \mathrm{Ndim}_{k+2}(\calH) \leq \mathrm{Ndim}(\calH) + \mathrm{VCdim}(\mathbb S) < \infty$.
\end{proof}

\subsection{Multiclass Disambiguations}
We extend the definition of \cite{alon2022theory} to multiclass partial concepts classes.
\begin{definition}
[e.g., \cite{alon2022theory}]
\label{definition:disambiguation}
A total concept class $\overline{\calH} \subseteq [k]^\calX$ is a special type of partial concept class
such that every $h \in \overline{\calH}$ has range $\{0,1,...,k\}$, i.e., is a total concept. A total concept class $\overline{\calH}$
is said to \textbf{disambiguate} a partial concept class $\calH \subseteq \{0,1,...,k,\star\}^\calX$ if every finite data sequence $S \in (\calX \times [k])^*$
realizable with respect to $\calH$ is also realizable with respect to $\overline{\calH}$. In this case, $\overline{\calH}$ is called a disambiguation of $\calH$.
\end{definition}

We will make use of the following key result in graph theory and communication complexity. We let $\chi(G)$ be the chromatic number of the simple graph $G$. We also set $\mathrm{bp}(G)$ be the biclique partition number of $G$, i.e., the minimum number of complete bipartite graphs needed to partition the edge set of $G$. The following result lies in the intersection of complexity theory and graph theory and is a result of numerous works \cite{huang2012counterexample, amano2014some,alon2022theory, shigeta2015ordered, goos2015lower, goos2018deterministic, ben2017low, balodis2022unambiguous}. Motivation was the Alon–Saks–Seymour problem in graph theory, which asks: How large
a gap can there be between the chromatic number of a graph and its biclique partition
number?
\begin{proposition}
[Biclique Partition and Chromatic Number \cite{balodis2022unambiguous}]
\label{proposition:alon-saks-seymour}
For any $n \in \nats$, there exists a simple graph $G$ with
$\mathrm{bp}(G) = n$ such that
\[
\chi(G) \geq n^{(\log(n))^{1-\eps(n)}}\,,
\]
where $\eps(n)$ is a sequence that tends to $0$ as $n \to \infty$\,.
\end{proposition}

Let us consider the binary classification setting. In the work of \cite{alon2022theory}, it was shown that the (combinatorial version) of the Sauer-Shelah-Perles (SSP) lemma fails in the partial concepts setting. In the total concepts setting, this lemma controls the size of a concept class $\calH$ in terms of its VC dimension. Another variant of this lemma controls the growth function of the class. Given a set $C = \{x_1,...,x_m\} \subseteq \calX$, the growth function of $\calH \subseteq \{0,1\}^\calX$ with respect to $C$ is the cardinality of the set $\Pi_{\calH}(C)$ of binary patterns realized by hypotheses in $\calH$ when projected to $C$, i.e.,
\[
\Pi_\calH(C) = \{ (h(x_1),...,h(x_m)) : h \in \calH \} \subseteq \{0,1\}^m\,.
\]
We define the growth function of $\calH$ at $m \in \nats$ as
\[
\Pi_\calH(m) = \sup_{C \subseteq \calX : |C| = m} |\Pi_\calH(C)|\,.
\]
This definition naturally extends to partial concepts where we still look only for binary patterns. Interestingly, while the combinatorial \cite{sauer} and the growth function \cite{understanding-ml} versions of the SSP lemma are both true in the total concepts setting, this is not true in the partial case. We show that the growth function variant still holds, while the combinatorial one fails \cite{alon2022theory}. 

\begin{lemma}
[Growth Function - SSP Lemma for Partial Concepts]
\label{lemma:ssp}
Let $\calH \subseteq \{0,1,\star\}^\calX$ be a partial concept class with finite VC dimension. For any $m \in \nats,$ it holds that
\[
\Pi_\calH(m) \leq \sum_{i=0}^{\mathrm{VCdim}(\calH)} \binom{m}{i}\,.
\]
\end{lemma}
The proof is postponed to \Cref{appendix:missing proofs for partial concepts}. We invoke the above SSP variant to prove a bound for the Natarajan dimension of a partial concept class with multiple labels.
\begin{lemma}
Let $\calH_{\mathrm{bin}} \subseteq \{0,1,\star\}^\calX$ be a partial concept class with finite VC dimension and $\calH = \{ h = r(h_1,...,h_\l) : h_i \in \calH_{\mathrm{bin}} \}$ for some $r : \{0,1\}^\l \to [k]$. It holds that
\[
\mathrm{Ndim}(\calH) \leq \wt{O}(k \cdot \mathrm{VCdim}(\calH_{\mathrm{bin}}))\,.
\]
\end{lemma}
\begin{proof}
Let the VC dimension of the binary concept class be $d$.
Let $S \subseteq \calX$ be a shattered set by the partial concept class $\calH$. Hence we have that
\[
|\Pi_\calH(S)| \geq 2^{|S|}\,.
\]
Now any hypothesis $h \in \calH$ is identified by $k$ binary partial concepts from $\calH_{\mathrm{bin}}$. We have that
\[
|\Pi_\calH(S)| \leq |\Pi_{\calH_{\mathrm{bin}}}(S)|^k\,,
\]
by the structure of $\calH$. This implies that
\[
|\Pi_\calH(S)| \leq O(|S|^d)\,,
\]
by the properties of the growth function. This gives that $|S| \leq \wt{O}(dk)$ and concludes the proof.
\end{proof}

The next theorem is one of the main results of this section; it essentially states that there exists some simple (with small Natarajan dimension) partial concept class $\calH^\star$ in the multiclass classification setting which cannot be disambiguated, in the sense that any extension of $\calH^\star$ to a total concept class has unbounded Natarajan dimension. Hence, for this class, there is no way to assign labels to the undefined points and preserve the expressivity of the induced collection of total classifiers.  
\begin{theorem}
[Disambiguation]
\label{theorem:disambiguation}
Fix $k \in \nats$.
For any $n \in \nats$, there exists a partial concept class $\calH_n \subseteq \{0,1,...,k,\star\}^{[n]}$ with $\mathrm{Ndim}(\calH_n) = O_k(1)$ such that any disambiguation $\overline{\calH}$ of $\calH_n$ has size at least $n^{\log(n)^{1-o(1)}}$, where the $o(1)$ term tends to 0 as $n \to \infty$. This implies that there exists $\calH_\infty \subseteq \{0,1,...,k,\star\}^\nats$ with $\mathrm{Ndim}(\calH_{\infty}) = O_k(1)$ but $\mathrm{Ndim}(\overline{\calH}) = \infty$ for any disambiguation $\overline{\calH}$ of $\calH_\infty$.
\end{theorem}
The proof is deferred to \Cref{appendix:missing proofs for partial concepts}.

\bibliography{main.bib}

\appendix

% \section{Omitted Details from the Main Body}
% \label{appendix:omitted details from main body}

% \subsection{Omitted Figures}
% \label{appendix:figures}

% \subsection{Previous Work}
% \label{appendix:multiclass}

\section{Assumption on Cardinality of $\calH$}
\label{appendix:cardinality of H}
We briefly discuss why the assumption that $|\calH| > k+2$ comes without loss of generality. 
If $\calH$ contains either a single hypothesis or contains some hypotheses which have no conflict, i.e., $h(x) \neq h'(x)$ everywhere, then $\Pr_{(x,y) \sim P}[\wh{h}_n(x) \neq y] = 0$ is trivially achievable for all $n$. Now, if $|\calH| \leq k+2$ but does not fall in the above two cases, then $\calH$ is learnable at an optimal exponential rate. To see this, let $\eps$ be the minimal error among all hypotheses $h \in \calH$ with non-zero error. The probability that there exists a hypothesis with error $\eps$ that makes no mistakes in the $n$ training data is at most $|\calH|(1-\eps)^n$ using the union bound. Thus a learning algorithm that outputs any classifier $\wh{h}_n \in \calH$ that correctly classifies the training set satisfies $\E[\Pr_{(x,y) \sim P}[\wh{h}_n(x) \neq y]] \leq C \exp(-c n)$ with $C =C(\calH, P)$ and $c = c(\calH, P)$. Moreover, this is optimal due to \Cref{proposition:exp-lower-bound}.

\section{Preliminaries}\label{appendix:preliminaries}
In this section we discuss more extensively the Preliminaries from \Cref{section:notation and preliminaries}.

\subsection{Complexity Measures}
\label{appendix:definitions}
We first state the classical definition of the Littlestone-dimension~\cite{littlestone1988learning}
that characterizes learnability in the online setting.
\begin{definition}[Littlestone dimension]
\label{definition:littlestone-dimension}
Consider a complete binary tree $T$ of depth $d+1$ whose internal
nodes are labeled by points in $\calX$ and edges by $\{0,1\}$, when they connect the parent to the right, left child, respectively. 
We say that $\calH \subseteq \{0,1\}^\calX$ Littlestone-shatters $T$ if for every root-to-leaf path $x_1,y_1,x_2,y_2,\ldots,x_d,y_d,x_{d+1}$ there exists some $h\in\calH$
such that $h(x_i) = y_i, 1 \leq i \leq d$. The Littlestone dimension is
denoted by $\mathrm{Ldim}(\calH)$ is defined to be the largest $d$ such
that $\calH$ Littlestone-shatters such a binary tree of depth $d+1$. If this happens for
every $d \in \nats$ we say that $\mathrm{Ldim}(\calH) = \infty$.
\end{definition}
The above definition naturally extends to multiple labels \cite{multiclass-erm}; we denote the multiclass Littlestone dimension by $\mathrm{Ldim}_k(\cdot)$. We next recall the well-known notion of VC dimension that characterizes PAC-learnability of binary
concept classes~\cite{vapnik2015uniform}.

\begin{definition}[VC dimension]\label{definition:vc-dimension}
We say that $\calH \subseteq \{0,1\}^{\calX}$ VC-shatters a sequence $\{x_1,\ldots,x_n\} \in \calX^n$ if $\{h(x_1), \ldots,h(x_n):h \in \calH\} = \{0,1\}^n$. The VC-dimension of $S$ is denoted by $\mathrm{VCdim}(\calH)$ and is defined to be the largest $d$ such that $\calH$ VC-shatters some sequence of length $d$. If this happens for every $d \in \nats$ we say that
$\mathrm{VCdim}(\calH) = \infty$.
\end{definition}

When moving to $\calH \subseteq \{0,\ldots,k\}^\calX, k > 1,$ there are many different extensions of the VC-dimension that have been considered in the literature. We recall the definition of two of them that are important in our work, the Natarajan dimension~\cite{natarajan1989learning} and the Graph dimension~\cite{natarajan1988two,natarajan1989learning}.

\begin{definition}[Natarajan dimension]\label{definition:natarajan dimension}
We say that $\calH \subseteq \{0,\ldots,k\}^{\calX}$ N-shatters a sequence $\{x_1,\ldots,x_n\} \in \calX^n$ if there exist two colorings $s^{(0)}, s^{(1)}$ such that $s^{(0)}(x_i) \neq s^{(1)}(x_i), \forall i \in [n],$ and $\forall I \subseteq [n]$ there exists some $h\in\calH$ with $h(x_i) = s^{(0)}(x_i), i \in I$ and $h(x_i) = s^{(1)}(x_i), i \in [n]\setminus I$. The Natarajan dimension is denoted by $\mathrm{Ndim}(\calH)$ and is defined to be the largest $d$ for which $\calH$ N-shatters a sequence of length $d$. If this happens for every $d\in \nats$ we say that $\mathrm{Ndim}(\calH) = \infty$.
\end{definition}

\begin{definition}[Graph dimension]\label{definition:graph dimension}
We say that $\calH \subseteq \{0,\ldots,k\}^{\calX}$ G-shatters a sequence $\{x_1,\ldots,x_n\} \in \calX^n$ if there exists a coloring $s^{(0)}$ such that $\forall I \subseteq [n]$ there exists some $h\in\calH$ with $h(x_i) = s^{(0)}(x_i), i \in I$ and $h(x_i) \neq s^{(0)}(x_i), i \in [n]\setminus I$. The Graph dimension is denoted by $\mathrm{Gdim}(\calH)$ and is defined to be the largest $d$ for which $\calH$ G-shatters a sequence of length $d$. If this happens for every $d\in \nats$ we say that $\mathrm{Gdim}(\calH) = \infty$.
\end{definition}

Notice that if $k=2$ the definitions of the Natarajan dimension and the Graph dimension are equivalent to the definition of the VC-dimension.

\begin{remark}
\label{rem:vc}
We underline that the above definitions can be naturally extended to the partial concepts setting. For instance, we say that $\calH \subseteq \{0,1,\star\}^\calX$
VC-shatters a set $S$ if every \textbf{binary} pattern is realized by some $h \in \calH$. For further details, see \cite{alon2022theory}.
\end{remark}

\subsection{One-Inclusion Hypergraph Algorithm}
\label{appendix:one-inclusion}
We next review a fundamental result which is a crucial ingredient in the design of our algorithms, namely the one-inclusion hypergraph algorithm $\mathbb A_{\calH}$ for the class $\calH \subseteq [k]^\calX$ \cite{haussler1994predicting,rubinstein2009shifting,daniely2014optimal,moran-dinur}. This algorithm gets as input a training set $(x_1,y_1),...,(x_n,y_n)$ realizable by $\calH$ and an additional example $x$. The goal is to predict the label of $x$. In this sense, the one-inclusion graph constitutes a transductive model in machine learning. The idea is to construct the one-inclusion (hyper)graph of $\calH|_{x_1,...,x_n,x} \subseteq [k]^{n+1}$. The nodes of this graph are the equivalence classes of $\calH$ induced by the examples $x_1,...,x_n,x$. For the binary classification case, two equivalence classes are connected with an edge if the nodes differ by exactly one element $x$ of the $n+1$ points and $x$ appears only once in $(x_1,...,x_n,x)$. For the case $k > 1$, the hyperedge set is is generalized accordingly.  Having created the one-inclusion graph, the goal is to orient the edges; the crucial property is that good orientations of this graph yield low error learning algorithms. Here, an orientation is good if the maximum out-degree of the graph is small. Intuitively, if the maximum out-degree of any node is $M$, then this can yield a prediction for the label $x$ with $n+1$ training samples with permutation mistake bound at most $M/(n+1)$.

\begin{lemma}
[One-Inclusion Hypergraph Algorithm (See Lemma 17 of \cite{moran-dinur})]
\label{lemma:one-inc}
Let $\calX = [n]$ with $n \in \nats$, $k$ be positive constant and $\calH \subseteq [k]^n$ be a class with Natarajan dimension $\mathrm{Ndim}(\calH) < \infty$. There exists an algorithm $\mathbb A : 2^{[k]^X} \times (X \times [k])^{n-1} \times X$ such that
\[
\frac{1}{n!}\sum_{\sigma \sim \calU(\mathbb S_n)}[\mathbb A(\calH, \sigma(1), h(\sigma(1)),...,\sigma(n-1), h(\sigma(n-1)), \sigma(n)) \neq h(\sigma(n))] \leq \frac{\mathrm{Ndim}(\calH)  \log(k)}{n}\,,
\]
for any $h \in \calH$.
\end{lemma}

\subsection{Ordinals}
\label{appendix:ordinals}
The discussion of ordinals is borrowed from~\cite{universal}. For a thorough treatment of the 
subject, the interested reader is referred to~\cite{hrbacek1999introduction, sierpinski1958cardinal}.

We consider some set $S$. A well ordering of $S$
is defined to be any linear ordering $<$ so that
every non-empty subset of $S$ contains a least element. For
example, the set of natural numbers $\nats$ along with the
usual ordering is well-ordered, whereas $[0,1]$ is not (take, e.g., $S = (0,1)$, then it is clear that there is not a least
element in $S$.)

We say that two-well ordered sets are \emph{isomorphic} if there
exists a bijection between the two which preserves the ordering.
There is a canonical way to construct ``equivalence classes'' of
well-ordered sets, called \emph{ordinals}, so that every 
well-ordered set is isomorphic to exactly one such ordinal.
In that sense, ordinals uniquely encode well-ordered sets
in the same way as cardinals uniquely encode sets. We denote the
class of all ordinals by \emph{$\mathrm{ORD}$}.

An important property of ordinals is that every pair of
well-ordered sets is either isomorphic, or one of them is
isomorphic to an \emph{initial segment} of the other. This fact
induces an \emph{ordering} over the ordinals. To be more precise,
we say that for two ordinals $\alpha, \beta \in \mathrm{ORD}$
we have that $\alpha < \beta$ if $\alpha$ is isomorphic to
an initial segment of $\beta$. The defining property of ordinals
is that every ordinal $\beta$ is isomorphic to the set of ordinals that
precede it, i.e., $\{\alpha: \alpha < \beta \}$. Moreover, the ordering
$<$ is itself a well-ordering over the ordinals. This is because every non-empty
set of ordinals contains a least element and it has a least upper bound.

Given the above discussion, we can see that ordinals provide a set-theoretic
extension of natural numbers. This is because every ordinal $\beta$ has 
a successor $\beta+1$, which is the smallest ordinal that is larger that $\beta$. Thus, we can create a list of all ordinals: the first elements
in the least are $0,1,2,\ldots$, which are all the finite ordinals where
we identify each number $k$ with the well-ordered set $\{0,\ldots,k-1\}$.
We let the ``smallest'' infinite ordinal be $\omega$. This can be identified
with the family of all natural numbers along with their natural ordering. 
The way to count past infinity is the following: we write $0,1,2\ldots,\omega,\omega+1,\omega+2,\ldots,\omega+\omega,\ldots,$ and we
denote the smallest uncountable ordinal with $\omega_1$.

A concept that is defined by ordinals and is useful for proving
the guarantees of our algorithms in the Gale-Stewart games is that of
\emph{transfinite recursion.} Roughly speaking, this principle states that
if we have a ``recipe'', which is given sets of ``objects'' $O_\alpha$, indexed
by all ordinals $\alpha < \beta$, defines a new set of ``objects'' $O_\beta$ and has access to a ``base set'' $\{O_\alpha: \alpha < \alpha_0\}$, then
$O_\beta$ is uniquely defined for all ordinals $\beta$. To give an example,
this concept helps us define addition $\gamma + \beta$ between two ordinals.
We set $\gamma + 0 = \gamma$ and $\gamma + 1$ to be the successor of $\gamma$. We continue inductively and define for any $\beta$ the addition
$\gamma + \beta = \sup\{(\gamma + \alpha) + 1: \alpha < \beta\}$. Following
this principle we can define an arithmetic in the ordinals.

\subsection{Well-founded Relations and Ranks}
We continue the above discussion by extending the notion of a 
well-ordering to a more general type of orders, and we introduce
the notion of \emph{rank,} which is important for the derivation
of the winning strategies in the Gale-Stewart games. We follow
the presentation from~\cite{universal}. The classical reference
for this topic is~\cite{kechris2012classical}.

We define a \emph{relation} $\prec$ on a set $S$ by an arbitrary subset
$R_\prec \subseteq S \times S$ and we let $x \prec y$ if and only if
$(x,y) \in R_\prec$. An element $x$ of $(S, \prec)$ is called \emph{minimal}
if there is no $x' \in S$ with $x' \prec x$. The relation is called 
\emph{well-founded} if every non-empty subset of $S$ has a minimal
element.

We associate every well-founded relation on $S$ with a rank function $\rho_\prec: S \rightarrow \mathrm{ORD}$, which is defined by the following
transfinite recursion: we let $\rho_\prec(x) = 0$ if $x$ is the minimal 
element of $S$ and we let $\rho_\prec(x) = \sup\{\rho_\prec(y) + 1: y \prec x\}$. Intuitively, the rank of some element $x$ quantifies how far away
it is from the minimal element. The following property of 
rank justifies the name ``transfinite recursion''.
\begin{remark}
\label{remark:well-defined rank}
Any element $x \in S$ has a well defined rank. This is because the transfinite recursion 
defines $\rho_\prec(x)$ as soon as all $\rho_\prec(y)$ is defined, for
all $y \prec x$. Indeed, if $\rho(x)$ is undefined it means that
$\rho_\prec(x')$ is also undefined, for some $x' \prec x.$ Repeating
this for $x'$ constructs an infinitely decreasing chain of elements
in $S$, which contradicts the fact that $\prec$ is well-founded, since
an infinitely decreasing chain does not contain a minimal element.
\end{remark}

We give some examples that help develop some intuition about the notion of the
rank of a well-founded
relation. Even though a well-founded relation does not admit an infinitely
decreasing chain, it can contain finite chains of arbitrary length. 
Essentially, the rank of some element $\rho_\prec(x)$ measures
the length of a decreasing chain starting from $x$.

\begin{example}
As a warmup, consider the case where $\rho_\prec(x) = k$, where $k$ is a finite ordinal. Then, there exists some $x_1 \prec x$ such that $\rho_\prec(x_1) = k-1$. Continuing in the same manner, we can see that we
can create a decreasing chain of length $k+1$.
\end{example}

\begin{example}
Moving on, assume that $\rho_\prec(x) = \omega.$ Recall that $\omega$ is 
the smallest infinite ordinal. This means that any $y \prec x$ has rank
$\rho_\prec(y) = k$, for some finite number $k$. Hence, starting from
$x$ we can create a decreasing chain of arbitrary length, but this length
is determined when we fix the first element of the chain.
\end{example}

\begin{example}
Now assume that $\rho_\prec(x) = \omega + k$, for some finite number $k$. Then, we can create $x \succ x_1 \succ \ldots \succ x_k$, with 
$\rho_\prec(x_k) = \omega.$ Thus, the length of the chain is determined after
we pick $x_{k+1}.$
\end{example}

\begin{example}
A slightly more involved example is the case where $\rho_\prec(x) = \omega + \omega.$ As a first element in the chain, we pick some $x_1 \prec x$ with
$\rho_\prec(x_1) = \omega + k,$ for some finite $k$. Then, we continue
as in the previous example. So in this case, the length of the chain
is determined by two choices.
\end{example}

\subsection{Measurability}\label{appendix:measurability}
This short exposition is mainly from \cite{universal}. A \textbf{Polish space} is a separable topological space that can be metrized by a complete metric. For example, the $n$-dimensional Euclidean space, any compact metric space, any separable Banach space are Polish spaces.

\begin{definition}
[See \cite{universal}]
\label{definition:measurability}
A concept class $\calH \subseteq [k]^\calX$ on a Polish space $\calX$ is
said to be \textbf{measurable} if there is a Polish space $\Theta$ and Borel-measurable map $h : \Theta \times \calX \to [k]$
so that $\calH = \{ h(\theta, \cdot) : \theta \in \Theta \}$.
\end{definition}

Roughly, a subset $B$ of a Polish space $\calX$ is universally measurable if it is measurable
with respect to every complete probability measure on $\calX$.
\begin{definition}
[Universally Measurable]
Let $\calF$ be the Borel $\sigma$-field on a Polish space 
$\calX$. For any probability measure $\mu$, denote by 
$\calF_\mu$ the completion of $\calF$ with respect to $\mu$, 
that is, the collection of all subsets of $\calX$ that differ from a 
Borel set at most on a set of zero probability. A set 
$B \subseteq \calX$ is called \textbf{universally measurable} if 
$B\in\calF_\mu$ for every probability measure $\mu$. Similarly, a 
function $f:\mathcal{X}\to\mathcal{Y}$ is called \textbf{universally measurable} if 
$f^{-1}(B)$ is universally measurable for any universally measurable set 
$B$.
\end{definition}

Let $\mathcal{X},\mathcal{Y}$ be Polish spaces, and let 
$f:\mathcal{X}\to\mathcal{Y}$ be a continuous function. It holds that $f$ is Borel measurable, that is, 
$f^{-1}(B)$ is a Borel subset of $\mathcal{X}$ for any Borel subset $B$ of 
$\mathcal{Y}$. A subset $B\subseteq\mathcal{X}$ of a Polish space is called \textbf{analytic} if it is the image of some Polish space under a continuous map. The complement of an analytic set is called \textbf{coanalytic}. A set is Borel if and only if it is both analytic and coanalytic. The following is a consequence of Choquet’s Capacitability Theorem.

\begin{fact}
Every analytic (or coanalytic) set is universally measurable.
\end{fact}

An important property of analytic sets is that
they help us establish measurability of uncountable unions
of measurable sets. 

We now state a very important result regarding well-founded 
relations on Polish spaces. We let 
$\calX$ be a Polish space and $\prec$ a well-found relation on $\calX$. We say that $\prec$ is analytic if $R_{\prec} \subseteq \calX \times \calX$ is an analytic set. The following theorem, known as Kunen-Martin~\cite{kechris2012classical, dellacherie1977derivations}, bounds the rank function of such a relation.

\begin{theorem}
\label{theorem:kunen-martin}
Let $\prec$ be an analytic well-founded relation on a Polish space $\calX$. Then, its rank function
satisfies $\sup_{x \in \calX} \rho_{\prec}(x) \leq \omega_1$.
\end{theorem}

\subsection{Gale-Stewart Games}
\label{appendix:gs}
We add a short discussion about GS games from \cite{universal}. We refer to their work for further details and pointers.
Fix sequences of sets $\mathcal{X}_t,\mathcal{Y}_t$ for $t\ge 1$. We consider infinite 
games between two players: in each round $t\ge 1$, the first player $\texttt{P}_A$ 
selects an element $x_t\in\mathcal{X}_t$, and then player $\texttt{P}_L$ selects an 
element $y_t\in\mathcal{Y}_t$. The rules of the game are determined by 
specifying a set $\mathcal{W} \subseteq \prod_{t\ge 
1}(\mathcal{X}_t\times\mathcal{Y}_t)$ of winning sequences for~$\texttt{P}_L$. That 
is, after an infinite sequence of consecutive plays 
$x_1,y_1,x_2,y_2,\ldots$, we say that $\texttt{P}_L$ wins if 
$(x_1,y_1,x_2,y_2,\ldots)\in\mathcal{W}$; otherwise, $\texttt{P}_A$ is declared the 
winner of the game. 

A \textbf{strategy} is a rule used by a given player to determine the next 
move given the current position of the game. A strategy for $\texttt{P}_A$ is a 
sequence of functions 
$f_t:\prod_{s<t}(\mathcal{X}_s\times\mathcal{Y}_s)\to\mathcal{X}_t$ for 
$t\ge 1$, so that $\texttt{P}_A$ plays $x_t=f_t(x_1,y_1,\ldots,x_{t-1},y_{t-1})$ in 
round $t$. Similarly, a strategy for $\texttt{P}_L$ is a sequence of 
$g_t:\prod_{s<t}(\mathcal{X}_s\times\mathcal{Y}_s) 
\times\mathcal{X}_t\to\mathcal{Y}_t$ for $t\ge 1$, so that $\texttt{P}_L$ plays 
$y_t=g_t(x_1,y_1,\ldots,x_{t-1},y_{t-1},x_t)$ in round $t$. A strategy for 
$\texttt{P}_A$ is called \textbf{winning} if playing that strategy always makes 
$\texttt{P}_A$ win the game regardless of what $\texttt{P}_L$ plays; a winning strategy for 
$\texttt{P}_L$ is defined analogously. The crucial question follows:
\begin{center}
    \emph{When do winning strategies exist in infinite two-player games?} 
\end{center}
The simplest assumption was introduced by \cite{gale1953infinite}: we call
$\mathcal{W}$ \textbf{finitely decidable} if for every sequence of plays 
$(x_1,y_1,x_2,y_2,\ldots)\in\mathcal{W}$, there exists $n<\infty$ so that
$$
	(x_1,y_1,\ldots,x_n,y_n,x'_{n+1},y'_{n+1},x'_{n+2},y'_{n+2},\ldots)
	\in\mathcal{W}
$$
for all choices of 
$x'_{n+1},y'_{n+1},x'_{n+2},y'_{n+2},\ldots$
In words, that ``$\mathcal{W}$ is finitely decidable'' means that if $\texttt{P}_L$ 
wins, then she knows that she won after playing a finite number of rounds. 
Conversely, in this case $\texttt{P}_A$ wins the game 
precisely when $\texttt{P}_L$ does not win after any finite number of rounds.

An infinite game whose set $\mathcal{W}$ 
is finitely decidable is called a \textbf{Gale-Stewart game}. The main result behind GS games follows.
\begin{proposition*}
\label{thm:gale-stewart winning strategies}
In any Gale-Stewart game, either $\texttt{P}_A$ or $\texttt{P}_L$ has a winning strategy.
\end{proposition*}
The above existential result provides no information, however, about the complexity of the winning strategies. In particular, it is completely unclear whether winning strategies can be chosen to be measurable. The next lemma addresses this concern.
\begin{lemma}
[Theorem B.1 of \cite{universal}]
\label{lemma:measurability-strat}
Let $\{X_t\}_{t \geq 1}$ be Polish spaces and $\{Y_t\}_{t \geq 1}$ be countable sets. Consider a Gale-Stewart game whose set $\calW \subseteq \prod_{t \geq 1} (X_t \times Y_t)$ of winning strategies for $\texttt{P}_L$ is finitely decidable and coanalytic. Then there is a universally measurable winning strategy.
\end{lemma}

For an extensive exposition on \emph{game values} which connect ordinals with the positions of the game, i.e., the sequences of choices between the two players, we refer to \cite{universal}.

\section{Deferred Proofs for Universal Rates}
\label{appendix:proofs for universal rates}
\subsection{The Proof of \Cref{proposition:exp-lower-bound}}
\label{proof:exp-lower-bound}
\begin{proof}
Recall that $|\calH| > k+2$. Hence, there are $h_0, h_0 \in \calH$ and $x, x' \in \calX$ such that $h_0(x) = h_1(x) = y$ and $h_0(x') = y_0 \neq h_1(x) = y_1$. We fix some learning algorithm $\wh{h}_n$ and two distributions $P_0, P_1$ where $P_i \left\{ (x,y)\right\} = \frac{1}{2}, P_i \left\{ (x,y_i)\right\} = \frac{1}{2}, i \in \{0,1\}$. We let $I \sim \mathrm{Bernoulli}(1/2)$ and given I, we let
$(X_1, Y_1), (X_2, Y_2),\ldots$ be i.i.d. samples from $P_I$ and $(X_1, Y_1), \ldots, (X_n, Y_n)$ are the training samples for $\wh{h}_n$ and $(X_{n+1}, Y_{n+1})$ is the test point. Then, we have that
\begin{align*}
    \E[\Pr(\wh{h}_n(X_{n+1}) \neq Y_{n+1}|\{X_t, Y_y\}_{t=1}^n, I)] \geq \frac{1}{2} \Pr(X_1 = \ldots = X_n, X_{n+1}=x') = 2^{-n-2}\,.
\end{align*}
We also have that
\begin{align*}
    &\E[\Pr(\wh{h}_n(X_{n+1}) \neq Y_{n+1}|\{X_t, Y_y\}_{t=1}^n, I)] \\&= \frac{1}{2 } \sum_{i\in\{0,1\}} \E[\Pr(\wh{h}_n(X_{n+1}) \neq Y_{n+1}|\{X_t, Y_y\}_{t=1}^n, I=i)|I=i]\,.
\end{align*}
Thus, for every $n$, there exists some $i_n \{0,1\}$ such that for
$(X_1, Y_1),\ldots,(X_n,Y_n)$ i.i.d. from $P_{i_n}$ it holds that
\begin{align*}
    \E[\er_{P_{i_n}}(\wh{h}_n)] \geq 2^{-n-2}\,.
\end{align*}
Hence, there exists some fixed $i\in\{0,1\}$ such that $\E[\er_{P_{i}}(\wh{h}_n)] \geq 2^{-n-2}$ for infinitely many $n$.
\end{proof}

\subsection{The Proof of \Cref{lemma:measurability}}
\label{proof:measurability}
\begin{proof}
It suffices to prove that the set $\calW$ of the winning strategies for the learning player in the Gale-Stewart game is coanalytic (see \Cref{lemma:measurability-strat}). Equivalently, we will prove that the set of winning strategies of the adversary $\calW^c$ is analytic, where
\[
\calW^c = \{ (\vec \kappa, \vec \eta) \in (\calX \times [k] \times [k] \times \{0,1\})^\infty : \calH_{\xi_1,y_1^{(\eta_1)},...,\xi_t,y_t^{(\eta_t)}} \neq \emptyset \textnormal{ for all } t < \infty \}\,.
\]
This set is equal to
\[
\calW^c = \bigcap_{1 \leq \tau < \infty} \bigcup_{\theta \in \Theta} \bigcap_{1 \leq t \leq \tau} \{ (\vec \kappa, \vec \eta) \in (\calX \times [k] \times [k] \times \{0,1\})^\infty : h(\theta, \xi_t) = y_t^{(\eta_t)}\}\,.
\]
The set $\{ (\theta, \vec \kappa, \vec \eta) : h(\theta, \xi_i) = y_i^{(\eta_j)}\} = \bigcup_{(y_i^{(0)}, y_i^{(1)}) \in [k] \times [k], y_i^{(0)} \neq y_i^{(1)}} \{ (\theta, \vec \xi, \vec \eta) : h(\theta, \xi_i) = y_i^{(\eta_j)}\}$ is a Borel set using the standard measurability assumption of \Cref{definition:measurability}. The set $\calW^c$ is analytic since the two intersections are over countable sets and the union is a projection of a Borel set.
\end{proof}

\subsection{The Proof of \Cref{lemma:adversarial-to-probabilistic}}
\label{proof:adversarial-to-probabilistic}
\begin{proof}
Since the distribution $P$ is realizable, there exists a sequence of  functions $h_k \in \calH$ so that
\[
\Pr_{(x,y) \sim P}[ h_k(x) \neq y] < \frac{1}{2^k}\,.
\]
Let us fix $t \geq 1$. We have that
\[
\sum_{k = 1}^{\infty} \Pr[ \exists s \leq t :  h_k(X_s) \neq Y_s ]
\leq t \sum_{k = 1}^{\infty} \Pr_{(X,Y) \sim P}[  h_k(X) \neq Y]
< \infty\,,
\]
where the first inequality is due to union bound. By Borel-Cantelli, with probability one, there exists for every $t \geq 1$ a hypothesis $h \in \calH$ so that $h(X_s) = Y_s$ for all $s \leq t$. Hence, the sequence $X_1,Y_1,X_2,Y_2,...$ is a valid input for the online learning game 
%\footnote{Not correct yet. I think we can say that for any such sequence there exists a tree of finite depth that has a prefix of this sequence as path??? \alkissnote{Check that this proof is correct}} 
with probability one.
In particular, we make use of the following statement: If $\calH$ does not have an infinite multiclass Littlestone tree, then there is a strategy for the learner that makes only finitely many mistakes against any adversary. This is proved in \Cref{thm:adversarial}.
The existence of a winning strategy $\wh{y}_t$ for the learning player implies that the time $T$ where the player makes a mistake is
\[
T = \sup \{ s \in \nats :  \wh{y}_{s-1}(X_s) \neq Y_s\}
\]
is a random variable that is finite with probability one. 
Moreover, the online learner is selected so that it is changed only when a loss is observed. This means that $\wh{y}_s = \wh{y}_t$ for all rounds $s \geq t \geq T$. 

We now employ the law of large numbers in order to understand the asymptotic behavior of the online learner:
\[
\Pr \left[ \Pr_{(x,y) \sim P}[\wh{y}_t(x) \neq y] = 0 \right]
= \Pr \left[ \lim_{S \to \infty} \frac{1}{S} \sum_{s = t+1}^{t+S} 1\{\wh{y}_t(X_s) \neq Y_s\} = 0 \right]
\]
and this probability is at least the probability of this event and of the event that $T \leq t$, i.e.,
\[
\Pr \left[ \Pr_{(x,y) \sim P}[\wh{y}_t(x) \neq y ] = 0 \right]
\geq \Pr \left[ \lim_{S \to \infty} \frac{1}{S} \sum_{s = t+1}^{t+S} 1\{ \wh{y}_t(X_s) \neq Y_s\} = 0, T \leq t \right] = \Pr[T \leq t]\,,
\]
where the last inequality follows from the observation that since $s \geq t$ and $t$ is greater than the critical time $T$ then the first event occurs with probability one. This implies that
\[
\Pr \left[ \Pr_{(x,y) \sim P}[\wh{y}_t(x) \neq y] = 0 \right] \geq \Pr[T \leq t]
\]
and so
\[
\Pr \left[ \Pr_{(x,y) \sim P}[\wh{y}_t(x) \neq y] >0 \right] \leq \lim_{t \to \infty} 1 - \Pr[T \leq t] = 0\,.
\]
\end{proof}

\subsection{The Proof of \Cref{lemma:general-estimate-time}}
\label{proof:general-estimate-time}
\begin{proof}
We split the training set into two sets. The idea is to use the first one to train the learning algorithm and the other set to estimate the generalization error. For each $1 \leq t \leq \lfloor \frac{n}{2} \rfloor$ and $1 \leq i \leq \lfloor \frac{n}{2t} \rfloor$, we let
\[
\wh{y}_t^i(x) = \wh{Y}_{t+1}(X_{(i-1)t+1}, Y_{(i-1)t+1}, X_{it}, Y_{it},x)
\]
be the output of the learning algorithm that is trained on batch $i$ of the data. For every fixed $t$, the data that the classifiers $\left\{\wh{y}_t^i\right\}_{i \leq \lfloor n/2t \rfloor}$ are trained on are independent of each other and of the second half of the training set. This means that we can view every $\left\{\wh{y}_t^i\right\}_{i \leq \lfloor n/2t \rfloor}$ as an independent draw of the distribution of $\wh{y}_t$. To estimate the generalization error of the algorithm we use the second half of the training set. We let
\[
\hat{e}_t = \frac{1}{\lfloor n/2t \rfloor} \sum_{i=1}^{\lfloor n/2t \rfloor} \vec{1} \{ \wh{y}_t^i(X_s) \neq Y_s \text{ for some } n/2 \leq s \leq n \}\,.
\]
Now observe that, with probability one,
\[
\hat{e}_t \leq e_t = \frac{1}{\lfloor n/2t \rfloor} \sum_{i=1}^{\lfloor n/2t \rfloor} \vec{1} \left\{ \Pr_{(X,Y) \sim P}[\wh{y}_t^i(X) \neq Y] > 0 \right\}\,.
\]
We define $\wh{t}_n = \inf\{t \leq \lfloor n/2 \rfloor: \wh{e}_t < 1/4\}$, where we assume that $\inf \emptyset = \infty$.

We now want to bound the probability that $\wh{t}_n > t^{\star}$. Using Hoeffding's inequality we get that
\begin{align*}
\Pr\left[\wh{t}_n > t^{\star}\right]
\leq \Pr\left[\wh{e}_{t^{\star}} \geq \frac{1}{4}\right] \leq &\Pr\left[e_{t^{\star}} \geq \frac{1}{4}\right] =\\   
& = \Pr\left[e_{t^{\star}} - \frac{1}{8} \geq \frac{1}{8}\right] = \Pr\left[e_{t^{\star}} - \E[e_{t^{\star}}] \geq \frac{1}{8}\right] \leq e^{-\lfloor n/2t^{\star} \rfloor/32}\,.
\end{align*}

This implies that $\wh{t}_n \leq t^{\star}$ except for an event with exponentially small probability.

Moreover, for all $1\leq t \leq t^{\star}$ that $\Pr\left[\Pr_{(x,y)\sim P} [\wh{y}_t(x) \neq y] > 0 \right] > \frac{3}{8}$, there is some $\eps > 0$ such that $\Pr\left[\Pr_{(x,y)\sim P} [\wh{y}_t(x) \neq y] > \eps \right] > \frac{1}{4} + \frac{1}{16}$ (this holds by continuity).
Now fix some $1\leq t \leq t^{\star}$ such that $\Pr\left[\Pr_{(x,y)\sim P} [\wh{y}_t(x) \neq y] > 0\right] > \frac{3}{8}$ (if it exists).
Then, using Hoeffding's inequality again we get that
\[
\Pr\left[\frac{1}{\lfloor n/2t \rfloor} \sum_{i=1}^{\lfloor n/2t \rfloor} \vec{1} \left\{ \Pr_{(x,y)\sim P} [\wh{y}_t^i(x) \neq y] > \eps \right\} < \frac{1}{4}\right] \leq e^{\lfloor n/2t^{\star} \rfloor/128}\,.
\]
Whenever $f$ is a function such that $\Pr_{(x,y) \sim P}\left[f(x) \neq y\right] > \eps$, then

\[
\Pr\left[ f(X_s) \neq Y_s \text{ for some } n/2 \leq s \leq n\right] \geq 1 - \left(1-\eps\right)^{n/2}\,.
\]
As we mentioned before, $\{\wh{y}_t^i\}_{i \leq \lfloor n/2t \rfloor}$ are independent of $(X_s,Y_s)_{s > n/2}$. Thus, applying a union bound we get that the probability that all $\wh{y}_t^i$ that have $\Pr_{(x,y)\sim P}[\wh{y}_t^i(x) \neq y] > \eps$ make at least one error on the second half of the training set is
\[
    \Pr\left[ \sum_{i=1}^{\lfloor n/2t \rfloor}\vec{1} \left\{\Pr_{(x,y)\sim P} [\wh{y}_t^i(x) \neq y] > \eps \right\} \leq \sum_{i=1}^{\lfloor n/2t \rfloor}\vec{1} \{ \wh{y}_t^i(X_s) \neq Y_s \text{ for some } n/2 < s \leq n\}\right] \geq 1 - \left\lfloor\frac{n}{2t}\right\rfloor\left(1-\eps\right)^{n/2}\,.
\]
Thus, we get that
\[
    \Pr[\wh{t}_n = t] \leq \Pr\left[\wh{e}_t < \frac{1}{4}\right] \leq \left\lfloor \frac{n}{2}\right\rfloor\left(1-\eps \right)^{n/2} + e^{-\lfloor\frac{n}{2t^{\star}}\rfloor/128}.
\]
%\grigorisnote{understand why they drop $1/t$}. \alkissnote{I do not think it is crucial. It only makes the RHS larger and worst case it cancels out with $t^\star$.}
Using the previous estimates and applying a union bound, we get that
\[
    \Pr[\wh{t}_n \notin T^\star] \leq e^{-\lfloor n/2t^{\star} \rfloor/32} + t^{\star}\left\lfloor\frac{n}{2}\right\rfloor \left(1-\eps \right)^{n/2} + t^{\star}e^{-\lfloor n/2t^{\star}\rfloor/128} \leq Ce^{-cn},
\]
for some constants $C,c > 0$. Note that $C = C(P, t^\star)$ and $c = c(P,t^\star)$.
\end{proof}

\subsection{The Proof of \Cref{theorem:inf-littlestone-linear-lower-bound}}
\label{proof:inf-littlestone-linear-lower-bound}
\begin{proof}
Fix any learner $\wh{h}_n$ and an infinite multiclass Littlestone tree for $\calH$. Fix also a random branch $\vec y = (y_1,y_2,...)$ of this tree, where the sequence is an i.i.d. sequence of fair Bernoulli coins. We introduce the random distribution over $\calX \times \{0,1,...,k\}$ as
\[
P_{\vec y}((x_{\vec y_{\leq \l}}, z_{\l+1})) = \frac{1}{2^{\l+1}}, \l \geq 0\,,
\]
where $z_{\l+1} \in \{0,1,...,k\}$ is the label of the edge connecting $x_{\vec y_{\leq \l}}$ to its child according to the chosen path $\vec y$. For any $n < \infty$, there exists a hypothesis $h \in \calH$ so that
\[
h(x_{\vec y_{\leq \l}}) = z_{\l+1} 
\]
for $0 \leq \l \leq n$. This is due to the construction of a multiclass Littlestone tree. We have that
\[
\mathrm{er}_{\vec y}(h) 
= \Pr_{(x,z) \sim P_{\vec y}}[h(x) \neq z] \leq \sum_{\l > n} 2^{-\l-1}\,, 
\]
which goes to $0$ as $n \to \infty$. This implies that $P_{\vec y}$ is realizable for every infinite branch $\vec y \in \{0,1\}^{\infty}$.
Let us draw $(X,Z),(X_1,Z_1),(X_2,Z_2),...$ i.i.d. samples from $P_{\vec y}$. Moreover, the mapping $y \to P_y$ is measurable.
The first sample corresponds to the test sample and the other samples deal with the training phase. Moreover, let $T,T_1,T_2,...$ be i.i.d. Geometric random variables with success probability $1/2$ starting at $0$. We can set
\begin{enumerate}
    \item $X = x_{\vec y \leq T}, Z = z_{T+1}$ and
    \item $X_i = x_{\vec y \leq T_i}, Z_i = z_{T_i + 1}$.
\end{enumerate}
Crucially, on the event that $\{T=\l, \max\{T_1,...,T_n\} < \l\}$, the value of $\wh{h}_n(X)$ is conditionally independent of $z_{\l+1}$ given $X,(X_1,Z_1),...,(X_n,Z_n)$. We next have that
\[
\Pr[ \wh{h}_n(X) \neq Z, T = \l, \max\{T_1,...,T_n\} < \l]
=
\Pr[ \wh{h}_n(X) \neq Z_{\l+1}, T = \l, \max\{T_1,...,T_n\} < \l]\,.
\]
This is equal to
\[
\E[ \Pr_{Z}[ \wh{h}_n(X) \neq Z | X, (X_1, Z_1),...,(X_n,Z_n)] \vec 1\{T = \l, \max\{T_1,...,T_n\} < \l\} ]
\]
Now conditional on this event, any algorithm will err with probability $1/2$ (since it will guess the value at random). Hence, this quantity is lower bounded by
\[
\frac{1}{2} \Pr[T= \l , \max\{T_1,...,T_n\} < \l] = 2^{-\l-2}(1-2^{-\l})^n\,.
\]
We are free now to pick $\l$. Choosing $\l = \l_n := \lceil 1 + \log(n) \rceil$, we have that
$1/2^\l > 1/(4n)$ and $(1-2^{-\l})^n \geq 1/2$. 
Our goal is to apply the reverse Fatou lemma. This can be done since almost surely, we have that
\[ 
n \Pr[ \wh{h}_n(X) \neq Z, T = \l_n | \vec y]
\leq n \Pr[T = \l_n] = n 2^{-\l_n-1} \leq 1/4\,.
\]
Hence, we can apply the reverse Fatou lemma and get
\[
\E \left[ \limsup_{n \to \infty} n \Pr[ \wh{h}_n(X) \neq Z, T = \l_n | \vec y] \right] \geq \limsup_{n \to \infty} n \Pr [ \wh{h}_n(X) \neq Z, T = \l_n] > 1/32\,.
\]
But, almost surely, it holds that
\[
\E[ \mathrm{er}_{\vec y}(\wh{h}_n) | \vec y]
= \Pr[\wh{h}_n(X) \neq Z | \vec y]
\geq \Pr[ \wh{h}_n(X) \neq Z, T = \l_n | \vec y ]\,.
\]
So, combining the above inequalities
\[
\E \left[ \limsup_{n \to \infty} n \E[\mathrm{er}_{\vec y}(\wh{h}_n)]  \right] > \Omega(1)\,.
\]
Hence, there must exist a realization of $\vec y$ so that
$\E[\mathrm{er}_{\vec y}(\wh{h}_n)] = \Omega(1/n)$ infinitely often. Choosing $P = P_{\vec y}$ completes the proof.
\end{proof}

\subsection{The Proof of \Cref{lemma:learner-winning}}
\label{proof:learner-winning}
\begin{proof}
The first point is that the 2-player game of \Cref{fig:gs-game2} is Gale-Stewart. This follows by the observation that the set of winning sequences of the learning player is finitely decidable, since the membership of $(\vec \xi, \vec \eta)$ in the set is witnessed by a finite subsequence. Hence, exactly one of the two players has a winning strategy in the game (see \Cref{thm:gale-stewart winning strategies}).

The second point is that the class $\calH$ has an infinite NL tree if and only if the adversary player has a winning strategy in the game. Suppose that $\calH$ has an infinite NL tree. The adversary can adopt the strategy iteratively by setting $\xi_t(\eta_1,...,\eta_{t-1}) = (x_{\eta_1,...,\eta_{t-1}}, s^{(0)}_{\eta_1,...,\eta_{t-1}}, s^{(1)}_{\eta_1,...,\eta_{t-1}}) \in \calX^t \times [k]^t \times [k]^t$. Hence, the adversary traverses the infinite tree and, by construction of the NL tree, the set $\calH_{\xi_1,\eta_1,...,\xi_t,\eta_t}$ never gets empty for any sequence of patterns $\vec \eta$ and $t < \infty$. Thus, this is a winning strategy for the adversary player. In the opposite direction, assume that the adversary has a winning strategy. Then, she can construct an infinite NL tree by setting $(x_{\eta_1,...,\eta_{t-1}}, s^{(0)}_{\eta_1,...,\eta_{t-1}}, s^{(1)}_{\eta_1,...,\eta_{t-1}}) = \xi_t(\eta_1,...,\eta_{t-1})$ for any possible binary pattern. Note that since the strategy is winning, the class $\calH_{\xi_1,\eta_1,...,\xi_t,\eta_t} \neq \emptyset$ for any $t \in \nats$ and this means that for any level of the tree, there exists two everywhere different colorings that witness the patterns.

Using the second point, we have that the learning player has a winning strategy in the game if the class $\calH$ has no infinite NL tree. To show that this strategy is also universally measurable, it suffices to show that the set of winning sequences for the learning player is coanalytic. The set of winning strategies for the adversary is
\[
\calW^c = \left \{ ( (\vec \xi, \vec s^{(0)}, \vec s^{(1)}), \vec \eta) \in \bigcup_{t=1}^{\infty} ( (\calX^t \times [k]^t \times [k]^t) \times \{0,1\}^t) : \calH_{\xi_1, \eta_1,...,\xi_t,\eta_t} \neq \emptyset \textnormal{ for all } t < \infty \right\}\,.
\]
The set $\calW^c$  is equal to
\[
\bigcap_{1 \leq t < \infty} \bigcup_{\theta \in \Theta}
%\bigcup_{k-\textnormal{colorings}~s^{(0)} \neq s^{(1)} \textnormal{over } t \textnormal{ points}} 
\bigcap_{1 \leq \l \leq t} 
 \left\{ ( (\vec \xi, \vec s^{(0)}, \vec s^{(1)}), \vec \eta) : 
\begin{array}{ll}
      h(\theta,\xi_z^{(i)}) = s_z^{(0)}(\xi_z^{(i)}) & \textnormal{if}~\eta_z^{(i)} = 0 \\
      h(\theta, \xi_z^{(i)}) = s_z^{(1)}(\xi_z^{(i)}) & \textnormal{if}~\eta_z^{(i)} = 1 \\
\end{array} 
\textnormal{for}~0\leq i < z, z \in [\l]
\right\}\,.
\]
In words, this set contains all the states of the game, for any timestep $t$, so that there exists a parameter in $\Theta$ and some $k$-colorings $s^{(0)},s^{(1)}$ which are everywhere different, that are $N$-consistent with some hypothesis in the class. Also, note that the set
\begin{align*}
& \left \{ ((\vec \xi, \vec s^{(0)}, \vec s^{(1)}), \vec \eta) : 
\begin{array}{ll}
      h(\theta, \xi_z^{(i)}) = s_z^{(0)}(\xi_z^{(i)}) & \textnormal{if}~\eta_z^{(i)} = 0 \\
      h(\theta, \xi_z^{(i)}) = s_z^{(1)}(\xi_z^{(i)}) & \textnormal{if}~\eta_z^{(i)} = 1 \\
\end{array} 
\textnormal{for}~0\leq i < z, z \in [\l]
\right \}
 = \\
 & \bigcap_{1 \leq z \leq \l} \bigcap_{0 \leq i < z} \left\{ ((\vec \xi, \vec s^{(0)}, \vec s^{(1)}), \vec \eta) :  h(\theta, \xi_z^{(i)}) = s_z^{(0)}(\xi_z^{(i)}) ~ \textnormal{if}~\eta_z^{(i)} = 0 \right\} \cap \\
&\left\{ ((\vec \xi, \vec s^{(0)}, \vec s^{(1)}), \vec \eta) :  h(\theta, \xi_z^{(i)}) = s_z^{(1)}(\xi_z^{(i)}) ~ \textnormal{if}~\eta_z^{(i)} = 1 \right\} =\\
& \bigcap_{1 \leq z \leq \l} \bigcap_{0 \leq i < z} \left\{ ((\vec \xi, \vec s^{(0)}, \vec s^{(1)}), \vec \eta) :  h(\theta, \xi_z^{(i)}) = s_z^{(\eta_z^{(i)})}(\xi_z^{(i)})\right\}\,.
\end{align*}
Recall that a Borel set is any set in a topological space that can be formed from open sets through the operations of countable union, countable intersection, and relative complement. By the measurability assumption of \Cref{definition:measurability}, the set $\left\{ (\theta, (\vec \xi, \vec s^{(0)}, \vec s^{(1)}), \vec \eta) :  h(\theta, \xi_z^{(i)}) = s_z^{(\eta_z^{(i)})}(\xi_z^{(i)}) \right\} 
%\left\{ (\theta, \vec \xi, \vec \eta) :  h(\theta, \xi_z^{(i)}) = s_z^{(1)}(\xi_z^{(i)}) ~ \textnormal{if}~\eta_z^{(i)} = 1 \right\}
$ 
which corresponds to the subspace of the product space of $\Theta$ and the space of all infinite sequences $((\vec \xi, \vec s^{(0)}, \vec s^{(1)}), \vec \eta )$
whose $(z,i)$-th term is consistent with some hypothesis $h$
is Borel measurable (note that the sequence $( \vec s^{(0)}, \vec s^{(1)}), \vec \eta)$ uniquely induces a sequence of labels over $[k]$ on which we apply the measurability definition). Moreover, the following hold: (i) the two intersections over $t$ and $\l$ are countable, (ii) the union over $\theta$ is a projection of a Borel set and (iii) the union over the colorings is countable (it is even finite given $t$). This implies that $\calW$ is coanalytic.
\end{proof}

\subsection{The Proof of \Cref{lemma:time-linear}}
\label{proof:time-linear}
\begin{proof}
We consider the strategies
\[
\l_t = L_t(X_1,Y_1,...,X_t,Y_t)
\]
and
\[
\wh{y}_t(z_1,...,z_{\l_t},s^{(0)},s^{(1)}) = \wh{Y}_t(X_1,Y_1,...,X_t,Y_t,z_1,...,z_{\l_t},s^{(0)},s^{(1)})\,,\footnote{Notice that $L_t,\wh{Y}_t$ depend also on the colorings that were used throughout the execution of the algorithm. We try the colorings in some fixed order and we slightly abuse the notation to drop that dependence.}
\]
which can be obtained by the pattern avoidance strategies. We remark that the algorithms behind $L_t$ and $\wh{Y}_t$ try out all the possible colorings as described in the pattern avoidance algorithm.
%\grigorisnote{for every different coloring $\wh{y}_t$ will return a different pattern for the same $z$-tuple so maybe we should somehow show this dependence on the coloring by modifying the arguments of $\wh{y}_t$? Similarly, $\wh{Y}_t$ depends on the colors for which the pattern avoidance function made a mistake but if we add that it might make the notation illegible -- we could say that we abuse the notation and drop the color of the history of the execution? If we say that we consider the colorings in some fixed order this should be ok I think.}
As a first step we decompose the training set into four parts. We will use the first quarter of the training examples as follows: For any $1 \leq t \leq \lfloor n/4 \rfloor$ and $1 \leq i \leq \lfloor n/(4t) \rfloor =: \wh{N}$, we invoke the above two strategies to get
\[
\l_t^i = L_t(X_{(i-1)t+1}, Y_{(i-1)t+1},...,X_{it}, Y_{it})
\]
and
\[
\wh{y}_t^i(z_1,...,z_{\l_t},s^{(0)},s^{(1)}) = \wh{Y}_t(X_{(i-1)t+1}, Y_{(i-1)t+1},...,X_{it}, Y_{it},z_1,...,z_{\l_t},s^{(0)},s^{(1)})\,.
\]
For each $t$ in the decomposition, we estimate the value $\Pr[\mathrm{per}(\wt{y}_t) > 0]$ using our estimates as
\[
\wh{e}_t = \frac{1}{\wh{N}} \sum_{i \in [\wh{N}]} \vec 1\left\{ \wh{y}_t^i \textnormal{ fails to avoid some NL pattern realized by $ (X_{s+1},Y_{s+1},...,X_{s+\l_t^i},Y_{s+\l_t^i})$ for some $n/4 \leq s \leq n/2 - \l_t^i$} \right\}\,.
\]
We note that almost surely $\wh{e}_t \leq e_t = \sum_{i \in [\wh{N}]} \vec 1\{ \mathrm{per}(\wh{y}_t^i) > 0 \} /\wh{N}$. Moreover, we set
\[
\wh{t}_n = \inf \{ t \leq \lfloor n/4 \rfloor : \wh{e}_t < 1/4 \}\,,
\]
and we set $\inf \emptyset = \infty$. Set the critical time $t^\star \in \nats$ be such that 
\[
\Pr[ \mathrm{per}(\wh{y}_{t^{\star}}) > 0] \leq 1/8\,,
\]
where the probability is over the training set of the algorithm $\wh{y}_t$. It holds that
\[
\Pr[\wh{t}_n > t^\star] \leq \Pr[\wh{e}_{t^\star} \geq 1/4] \leq \Pr[e_{t^\star} - \E[e_{t^\star}] \geq 1/8] \leq \exp(-N^\star/32)\,,
\]
where $N^\star = \lfloor n/(4t^\star) \rfloor$. By continuity, there exists $\eps > 0$ such that for all $1 \leq t \leq t^\star$ such that $\Pr[\mathrm{per}(\wh{y}_t)] >  3/8$, we have that $\Pr[\mathrm{per}(\wh{y}_t) >  \eps] > 1/4 + 1/16$. Fix $1 \leq t \leq t^\star$ with $\Pr[\mathrm{per}(\wh{y}_t) > 0] > 3/8$ (if such an index exists). Standard concentration inequalities yield
\[
\Pr \left[ \sum_{i \in [\wh{N}]} \vec 1\{ \mathrm{per}(\wh{y}_t^i) > \eps \} < \frac{\wh{N}}{4} \right] \leq \exp(-N^\star/128)\,.
\]
We remark that any NL pattern avoidance function $g$ which satisfies $\mathrm{per}(g) > \eps$, we have that
\[
\Pr[ g \textnormal{ fails to avoid some NL pattern realized by } (X_{s+1},Y_{s+1},...,X_{s+\l},Y_{s+\l}) \textnormal{ for some } n/4 \leq s \leq n/2-\l]
\geq 1 - p\,,
\]
where $p = p(\eps, n,\l) = (1-\eps)^{\lfloor (n-4)/(4\l) \rfloor}$ since there are
$\lfloor (n-4)/(4\l) \rfloor$ disjoint intervals of length $\l$ in $[n/4+1, n/2] \cap \nats$.

We now pass to the second quarter of the dataset to test our guesses using the above results. First, the estimates $(\l_t^i, \wh{y}_t^i)$ for $1 \leq i \leq \wh{N}$ are independent of $(X_s, Y_s)$ for $s > n/4$. Second, using a union bound conditionally on the first quarter of the dataset gives that the probability that every guess function $\wh{y}_t^i$
with $\mathrm{per}^{\l_t^i}(\wh{y}_t^i) > \eps$ (we let $\mathrm{per}^\l$ denote that the probability is over a sample from the distribution $P^{\otimes \l})$ makes an error on the second quarter of the dataset is
\[
\Pr[~(\forall i)~ \vec 1\{ \mathrm{per}^{\l_t^i}(\wh{y}_t^i)> \eps \} \leq \vec 1\{ E_{i,t} \}\}]
\geq 1 - \wh{N} (1-\eps)^{\lfloor (n-4)/ (4t^\star) \rfloor}\,,
\]
since $\l_t^i \leq t^\star$, where $E_{i,t}$ is the event that the $(i,t)$-pattern avoidance estimate
$\wh{y}_t^i$ fails to avoid some NL pattern realized by the data sequence $(X_{s+1},Y_{s+1},...,X_{s+\l_t^i},Y_{s+\l_t^i})$ for some $n/4 \leq s \leq n/2-\l_t^i$. This yields that
\[
\Pr[\wh{t}_n = t] \leq \Pr[\wh{e}_t < 1/4] \leq \lfloor n/4 \rfloor (1-\eps)^{\lfloor (n-4)/ (4t^\star) \rfloor} + \exp(-N^\star/32)\,.
\]
Taking a union bound over the elements of $T^\star$, where
\[
T^\star = \left\{ 1 \leq t \leq t^\star : \Pr[ \mathrm{per}(\wh{y}_{t^{\star}}) > 0] \leq 3/8 \right\}\,,
\]
we obtain that
\[
\Pr[\wh{t}_n \notin T^\star] \leq \exp(-N^\star/32) + t^\star \cdot (\lfloor n/4 \rfloor (1-\eps)^{\lfloor (n-4)/ (4t^\star) \rfloor} + \exp(-N^\star/32) )\,.
\]
This concludes the proof since there exist $C,c > 0$ so that
\[
\Pr[\wh{t}_n \in T^\star] \geq 1 - C \exp(-c n)\,.
\]
\end{proof}

\section{Deferred Proofs for Partial Concept Classes}
\label{appendix:missing proofs for partial concepts}
\subsection{The Proof of \Cref{lemma:one-inc-partial}}
\begin{proof}
Fix $n \in \nats$. Consider a set of points $S = \{x_1,...,x_n\}$ and let $S_d$ be the set of distinct elements of the sequence $S$. Define the hypothesis class $\calH_{S_d}$ that contains all the total functions $h : S_d \to [k]$ such that the sequence $\{(x,h(x)) : x \in S_d\}$ is realizable with respect to $\calH$. 

\noindent\textsc{Case A:} Assume that $\calH_{S_d} \neq \emptyset$. This is a total concept class and so let $\mathbb A_{S_d}$ be the algorithm guaranteed to exist
by \Cref{lemma:one-inc} with $X = S_d$ and $\calH = \calH_{S_d}.$ For any $y_1,...,y_n \in [k]$ so that the training sequence $(x_1,y_1),...,(x_n,y_n)$ is realizable with respect to $\calH$ (and so realizable with respect to $\calH_{S_d}$), define
\[
\mathbb A(x_1,y_1,...,x_{n-1},y_{n-1},x_n)
= \mathbb A_{S_d}(\calH_{S_d}, x_1,y_1,...,x_{n-1},y_{n-1},x_n)\,.
\]
Moreover, we can consider any permutation of the sequence $x_1,...,x_n$ and let the feature space $S_d$ and the hypothesis class $\calH_{S_d}$ the same. Finally, we have that $\mathrm{Ndim}(\calH_{S_d}) \leq \mathrm{Ndim}(\calH)$. This gives the desired bound.\\
\noindent\textsc{Case B:} Assume that $\calH_{S_d}$ is empty. In this case, set $\calA(x_1,y_1,...,x_{n-1},y_{n-1},x_n) = 0$ for all sequences $(x_1,...,x_n) \in \calX^n$ and $(y_1,...,y_{n-1}) \in [k]^{n-1}$ so that
$\{h \in \calH : h(x_{i}) = y_i \text{ with } i < n \text{ and } h(x_n) \in [k] \} = \emptyset$.
\end{proof}

\subsection{The Proof of \Cref{lemma:ssp}}
\begin{proof}
Set $d = \mathrm{VCdim}(\calH)$.
Consider a set $C = \{x_1,...,x_m\} \subseteq \calX$ of $m$ points and define $\Pi_\calH(C) = \{ (h(x_1),...,h(x_m)) : h \in \calH \} \subseteq \{0,1\}^{m}$ (where we ignore any vector that contains the $\star$ symbol). Note that if $m \leq d$, then it holds $|\Pi_\calH(C)| = 2^m$, by the definition of the VC dimension in the partial setting. 
%We also set $h_C = (h(x_1),...,h(x_m)).$
% We will prove the desired inequality by induction on $m+d$.
% If $m=d=0$, then the desired inequality holds trivially.
% Now let us assume that $m > 0, d > 0$ and fix a point $x_t \in S$. We define the set
% \[
% \calH' = \{ h_S \in \Pi_\calH(S) : x_t  \notin h_S, h_S \cup \{x_t\} \in \Pi_\calH(S) \}\,,
% \]
% where we identify $h_S$ with a subset of $S$. Due to the above definition, it holds that
% \[
% |\Pi_\calH(S)| = |\Pi_\calH(S \setminus \{x_t\})| + |\calH'| = |\Pi_\calH(S \setminus \{x_t\})| + |\Pi_{\calH^{''}}(S)|\,.
% \]
% The class $\calH^{''}$ can be identified as the set $\{ h \in \calH : \exists h' \in \calH \textnormal{ s.t. } (h(x_t) = 0, h(S \setminus x_t)) = (h'(x_t) = 1, h'(S \setminus x_t)) \}$ and has $\mathrm{VCdim}(\calH'') \leq d-1$.  Hence, we can apply the inductive hypothesis and get that
% \[
% |\Pi_\calH(S)| \leq \Phi_d(m-1) + \Phi_{d-1}(m) = \Phi_d(m)\,.
% \]
We are going to prove the next claim.
\begin{claim}
For any $C = \{x_1,...,x_m\}$ and any binary partial concept class $\calH$, we have that
\[
|\Pi_\calH(C)| \leq |\{ B \subseteq C : \calH \textnormal{ shatters } B\}|\,.
\]
\end{claim}
The above claim suffices since the RHS is at most $\sum_{i = 0}^{\mathrm{VCdim}(\calH)} \binom{m}{i}$. Now we prove the above claim.
For $m=1$, the result holds. Assume that the claim is true for sets of size $\l < m$ and let us prove it for sets of size $m$. Fix some partial binary concept class $\calH \subseteq \{0,1,\star\}^\calX$ and set $C = \{x_1,...,x_m\}$. Let $C' = C \setminus \{x_1\}$ and define
\[
Y_0 = \{ (y_2,...,y_m) \in \{0,1\}^{m-1} : 
(0,y_2,...,y_m) \in \Pi_\calH(C) \lor
(1,y_2,...,y_m) \in \Pi_\calH(C)
\}\,,
\]
and
\[
Y_1 = \{ (y_2,...,y_m) \in \{0,1\}^{m-1} : 
(0,y_2,...,y_m) \in \Pi_\calH(C) \land
(1,y_2,...,y_m) \in \Pi_\calH(C)
\}\,,
\]
Note that $|\Pi_\calH(C)| =
%\grigorisnote{\leq?} 
|Y_0| + |Y_1|$ (due to double counting).
Our first observation is that $|Y_0| \leq  |\Pi_{\calH}(C')|$, since there may exist some $h \in \calH$ which is undefined at $x_1$ and that generates a pattern that is not contained in $Y_0$. Using the inductive hypothesis on $\calH$ and $C'$ for the second inequality and the definition of $C'$ for the third equality, we get
\[
|Y_0| \leq |\Pi_{\calH}(C')| \leq |\{B \subseteq C' : \calH \textnormal{ shatters } B\}| = |\{B \subseteq C : x_1 \notin B \land \calH \textnormal{ shatters } B\}|\,.
\]
Let us now set
\[
\calH' = \{ h \in \calH : \exists h' \in \calH \textnormal{ such that } (h(x_1),h(C')) = (1-h'(x_1),h'(C')) \land \star \notin h(C)  \}\,,
\]
i.e., $\calH'$ contains all the pairs of hypotheses that (i) are well defined over $C$, (ii) agree on $C'$ and (iii) differ on $x_1$.
Note that if $\calH'$ shatters a set $B \subseteq C'$ then it also shatters $B \cup \{x_1\}$ and vice versa. Moreover we have  that $Y_1 = \Pi_{\calH'}(C')$. By the inductive hypothesis on $\calH'$ and $C'$, we have that
\[
|Y_1| = |\Pi_{\calH'}(C')|
\leq |\{ B \subseteq C' : \calH' \textnormal{ shatters } B\}|\,.
\]
This gives that
\[
|Y_1| \leq |\{ B \subseteq C' : \calH' \textnormal{ shatters } B \cup \{x_1\}\}|
= |\{ B \subseteq C : x_1 \in B \land \calH' \textnormal{ shatters } B\}|\,.
\]
Finally since $\calH'$ lies inside $\calH$, we get
\[
|Y_1| \leq |\{ B \subseteq C : x_1 \in B \land \calH \textnormal{ shatters } B\}|
\]
Combining our observations for $Y_0$ and $Y_1$, we get that
\[
|\Pi_\calH(C)| \leq |\{B \subseteq C : \calH \textnormal{ shatters } B\}|\,.
\]
\end{proof}

\subsection{The Proof of \Cref{theorem:disambiguation}}
\begin{proof}
The proof is essentially a tensorization of the construction of \cite{alon2022theory}. Fix $k, n\in \nats$. Set $L = \log_2(k+1)$ and assume that $L \in \nats$ without loss of generality. Let us consider the partial concept class
\[
\calH_n = \calH_n^{(1)} \times \ldots \times \calH_n^{(L)}\,,
\]
where $\calH_n^{(i)} \subseteq \{0,1,\star\}^{[n]}$. Let us now explain the partial concepts that lie in $\calH_n^{(i)}$ and subsequently in $\calH_n$. To this end, we invoke \Cref{proposition:alon-saks-seymour} which lies in the intersection of combinatorics and complexity theory. Consider $L$ independent and disjoint copies of the graph promised by \Cref{proposition:alon-saks-seymour} and let $G$ be the union of these $L$ graphs. Fix $i \in [L]$. Define the partial class $\calH_n^{(i)} \subseteq \{0,1,\star\}^{[n]}$ using the graph $G^{(i)}$ and its $n$ bipartite complete graphs $B^{(i)}_j = (L_j^{(i)}, R_j^{(i)}, E_j^{(i)})$ with $j \in [n]$. The class contains $|V(G^{(i)})|$ concepts, each one identified by a vertex $v \in V(G^{(i)})$ with
\[
c_v^{(i)}(j)  =
\left\{
\begin{array}{ll}
      0  & \textnormal{ if } v \in L_j^{(i)}, \\
      1 & \textnormal{ if } v \in R_j^{(i)}, \\
      \star & \textnormal{otherwise.} \\
\end{array} 
\right\}.
\]
Using Lemma 31 from \cite{alon2022theory}, we have that $\mathrm{VCdim}(\calH_n^{(i)}) = 1$.

We overload the ``+'' notation by setting $g + \star = \star$ for any $g \in \nats$.
The partial concept class $\calH_n$ contains all the partial concepts $h_{v_1,...,v_L}(i) = \sum_{j \in [L]} 2^L \cdot c_{v_j}^{(j)}(i) \in \{0,1,...,k,\star\}$. Hence we have that $\calH_n \subseteq \{0,1,...,k,\star\}^{[n]}.$ We can use the growth function of the partial concepts setting and get that $\mathrm{Ndim}(\calH_n) = O(\log(k+1) \cdot \mathrm{VCdim}(\calH_n^{(1)})) = O_k(1).$ 

Consider some disambiguation $\overline{\calH} \subseteq [k]^{[n]}$ of $\calH_n$.  The class $\overline{\calH}$ induces $L$ disambiguations $\overline{\calH^{(i)}}$ for the binary partial classes $\calH_n^{(i)}$.  
Then $\overline{\calH}$ defines a coloring of $G$ using $\min_{i} |\overline{\calH^{(i)}}|$ colors. \Cref{proposition:alon-saks-seymour} implies that
\[
\min_i |\overline{\calH^{(i)}}| \geq n^{\log(n)^{1-o(1)}}\,.
\]
Finally, we can consider the class $\calH_\infty$ as the disjoint union of $\calH_n$.
Each $\calH_n$ has domain $\calX_n$, where the domains $\calX_n$ are mutually disjoint and $\calH_\infty$ is the union $\bigcup_{n} \wt{\calH}_n$, where $\wt{\calH}_n$ is obtained from $\calH_n$ by adding $\star$ outside of its domain.
Then $\mathrm{Ndim}(\calH_\infty) = O_k(1)$. Since the size of its disambiguations is unbounded, then the multiclass Sauer-Shelah-Perles Lemma \cite{characterization} implies that the Natarajan dimension of any disambiguation of $\calH_\infty$ is infinite.
\end{proof}
\end{document}